\documentclass{article}

\usepackage[main,final]{neurips_2026}
\usepackage[utf8]{inputenc} % allow utf-8 input
\usepackage[T1]{fontenc}    % use 8-bit T1 fonts
\usepackage{hyperref}       % hyperlinks
\usepackage{url}            % simple URL typesetting
\usepackage{booktabs}       % professional-quality tables
\usepackage{amsfonts}       % blackboard math symbols
\usepackage{nicefrac}       % compact symbols for 1/2, etc.
\usepackage{microtype}      % microtypography
\usepackage{xcolor}         % colors
\usepackage{amsmath}
\usepackage{amssymb}
\usepackage{graphicx}
\usepackage{enumitem}
\usepackage{mathtools}
\usepackage{cleveref}
\usepackage{algorithm}
\usepackage{algorithmic}
\usepackage{caption}
\usepackage{aliascnt}
\usepackage{tikz}
\usetikzlibrary{calc}
\usetikzlibrary{arrows.meta}
\usepackage{amsthm}

\title{Is $\sqrt{d}$ Separation Necessary for Gradient EM to Learn Gaussian Mixtures in High Dimensions?}

\author{%
  Yiran Zhang \\
  University of California, Berkeley\\
  \texttt{zhangyir@berkeley.edu} \\
  \And
  Mo Zhou\\
  University of Washington\\
  \texttt{mozhou17@cs.washington.edu} \\
  \AND
  Weihang Xu\\
  University of Washington\\
  \texttt{xuwh@cs.washington.edu} \\
  \And
  Maryam Fazel\\
  University of Washington\\
  Amazon, Inc.\\
  \texttt{mfazel@uw.edu} \\
  \And
  Simon S.~Du\\
  University of Washington\\
  \texttt{ssdu@cs.washington.edu} \\
}

\theoremstyle{plain}
\newtheorem{theorem}{Theorem}

\newaliascnt{lemma}{theorem}
\newtheorem{lemma}[lemma]{Lemma}
\aliascntresetthe{lemma}

\theoremstyle{remark}
\newtheorem{remark}[theorem]{Remark}

\crefname{theorem}{theorem}{theorems}
\Crefname{theorem}{Theorem}{Theorems}

\crefname{proposition}{proposition}{propositions}
\Crefname{proposition}{Proposition}{Propositions}

\crefname{corollary}{corollary}{corollaries}
\Crefname{corollary}{Corollary}{Corollaries}

\crefname{assumption}{assumption}{assumptions}
\Crefname{assumption}{Assumption}{Assumptions}

\crefname{lemma}{lemma}{lemmas}
\Crefname{lemma}{Lemma}{Lemmas}

\crefname{remark}{remark}{remarks}
\Crefname{remark}{Remark}{Remarks}
\let\cref\Cref

\newcommand{\E}{\mathbb{E}}
\newcommand{\R}{\mathbb{R}}
\newcommand{\Rpos}{\mathbb{R}_+}

\newcommand{\bx}{\bar{x}}

\newcommand{\calL}{\mathcal{L}}
\newcommand{\calP}{\mathcal{P}}
\newcommand{\calN}{\mathcal{N}}

\newcommand{\botmu}{\mu^\bot}
\newcommand{\pimin}{\pi_{\min}}
\newcommand{\pigap}{\pi_{\textnormal{gap}}^*}
\newcommand{\parmu}{\mu^\parallel}

\begin{document}

\maketitle

\begin{abstract}
  Learning Gaussian mixture models (GMMs) using the Expectation-Maximization (EM) algorithm and its gradient-based variants is a fundamental problem in machine learning. It is known that randomly initialized (gradient) EM fails to learn multi-component GMMs in the exact-parameterized setting, where the number of components matches that of the ground-truth GMM. Recently, global convergence of gradient EM has been established in the over-parameterized setting, where more components are used, provided that the ground-truth components are well separated. In particular, the minimum separation between ground-truth components is required to scale as $\Omega(\sqrt{d})$, where $d$ is the dimension.
In this paper, we show that this dimensional dependence is unavoidable in high-dimensional settings. Specifically, we consider a hybrid EM algorithm that uses standard EM updates for the mixing weights and gradient EM updates for the component means. For any $\epsilon > 0$, we prove that when the dimension is sufficiently large, in the worst case a separation of order $\Omega(d^{0.5-\epsilon})$ is insufficient to guarantee global convergence of population gradient EM in sub-exponential time under random initialization, even in the over-parameterized regime. Our result establishes an almost optimal worst-case lower bound on the ground-truth separation required for learning Gaussian mixtures via gradient EM in high dimensions.
\end{abstract}

\section{Introduction}
\label{sec:intro}
%In this section, introduce your research area and explain why this problem is important (motivation). Provide enough background for the reader to understand your work. Finally, clearly state the research question or core problem you aim to solve.
%\subsection{Background}

The \emph{Expectation-Maximization (EM)} algorithm \citep{dempster1977maximum} is one of the most popular methods for learning latent-variable models. Despite its empirical success and significant theoretical progress \citep{wu1983convergence,xu1996convergence,balakrishnan2017statistical,dwivedi2020singularity,wu2021randomly}, the global behavior of EM, as well as its practical variants such as \emph{gradient EM}, remains not fully understood. In this paper, we study the classical problem of learning \emph{Gaussian Mixture Models (GMMs)} as a canonical testbed for analyzing EM dynamics. Our primary focus is on how global convergence depends on the \emph{separation} between mixture component means in high-dimensional settings.

Specifically, we consider the problem of learning Gaussian mixtures with identity covariance.
The ground truth distribution is a $d$-dimensional $m$-component Gaussian mixture model ($m$-GMM) with the probability density function $p^*$ given by
\[
p^*(x)=\sum_{i=1}^m \pi_i^*\phi(\mu_i^*;x),
\]
where $\pi_i^*\in\Rpos$ are the mixing weights with $\sum_{i}\pi_i^*=1$, $\mu_i^*\in \R^d$ are the true mean vectors. The function $\phi(\mu;x)$ is the p.d.f. of a $d$ dimensional Gaussian with mean $\mu$ and identity covariance $I_d$
\[
\phi(\mu;x):=(2\pi)^{-d/2}\exp\left(-\|x-\mu\|^2/2\right),
\]
where $\|\cdot\|$ denotes 2-norm throughout the paper. 
The goal is to recover the true distribution $p^*$ with observed data sampled from it.

\paragraph{The EM and Gradient EM Algorithms. }
The EM algorithm fits an $n$-GMM 
\begin{equation}
    p(x) = \sum_{i=1}^{n} \pi_i \, \phi(\mu_i; x),
\end{equation}
to recover the ground-truth GMM $p^*$, where the number of Gaussian components in the model ($n$) and in the ground truth ($m$) may differ.
The algorithm iteratively updates the mixing weights $\pi := (\pi_1, \ldots, \pi_n)^\top$ and model means $\mu := (\mu_1^\top, \ldots, \mu_n^\top)^\top$.
In practice, a commonly used variant is \emph{gradient EM}, which replaces the exact maximization step when standard EM updates are computationally expensive.

To understand optimization behavior in its cleanest form, we focus on the population dynamics of EM and gradient EM to separate statistical issues from optimization dynamics. This population-level analysis is standard in studies of EM algorithm \citep{xu2016global, xu2024toward,zhou2025global}.

At iteration $\ell$, both population EM and gradient EM consist of two steps:\\
\textbf{Expectation (E) step.} 
For $i = 1, \ldots, n$, compute the membership weight function $\psi_i : \mathbb{R}^d \to \mathbb{R}_+$, the posterior probability of $x \in \mathbb{R}^d$ being sampled from the $i$-th Gaussian:
\begin{equation}
    \psi_i(x) = \Pr(x|i) = 
    \frac{\pi_i^{(\ell)} \, \phi(\mu_i^{(\ell)}; x)}
         {\sum_{k=1}^{n} \pi_k^{(\ell)} \, \phi(\mu_k^{(\ell)}; x)}.
\end{equation}
\textbf{Maximization (M) step.}
Define the $Q$ function as
\begin{equation}
    Q(\pi, \mu \mid \pi^{(\ell)}, \mu^{(\ell)}) = 
    \mathbb{E}_{x \sim p^*} \left[
        \sum_{i=1}^{n} \psi_i(x; \pi^{(\ell)}, \mu^{(\ell)}) 
        \left( \log \pi_i - \frac{1}{2} \|x - \mu_i\|_2^2 \right)
    \right].
\end{equation}
Standard EM updates $(\pi, \mu)$ to be the maximizer of $Q$ function:
\[
(\pi^{(\ell+1)}, \mu^{(\ell+1)}) = \arg\max_{\pi, \mu} Q(\pi, \mu \mid \pi^{(\ell)}, \mu^{(\ell)}).
\]
Gradient EM replaces this maximization by a gradient step on $Q$ function. Let $\eta \in \mathbb{R}_+$ be the learning rate. The update rule is:
\begin{equation}\label{eq: gradient em}
    (\pi^{(\ell+1)}, \mu^{(\ell+1)}) 
    \leftarrow 
    (\pi^{(\ell)}, \mu^{(\ell)}) 
    + \eta \, \nabla_{\pi, \mu} Q(\pi, \mu \mid \pi^{(\ell)}, \mu^{(\ell)}).
\end{equation}
%In mixture models, the mixing weights lie on the simplex $\sum_i \pi_i = 1$, so a naive unconstrained gradient step on $\pi$ does not preserve feasibility. Accordingly, throughout this paper we use standard EM updates for $\pi$ and gradient steps only for $\mu$. The precise update rules are described in \Cref{sec:prel}.

A central quantity governing both the computational complexity of learning GMMs and the dynamics of EM is the minimum separation between the true component means:
\[
\Delta := \min_{i \neq j} \| \mu_i^* - \mu_j^* \|.
\]
Information-theoretic bounds suggest that $\Delta=\Omega(\sqrt{\log m})$ is necessary for efficient learning \citep{regev2017learning}, and specialized algorithms have been developed to match this bound \citep{liu2022clustering,diakonikolas2024implicit}. However, these algorithms often differ significantly from the EM algorithm used in practice. Despite extensive study \citep{balakrishnan2017statistical, dwivedi2020singularity, wu2021randomly}, the global convergence of EM for general $m$-GMMs under such separation remains an open problem. Moreover, \cite{jin2016local} showed that when $n=m$ the exact-parameterized (gradient) EM fails to converge to global minima starting from random initialization, even for well-separated GMMs with $\Delta=\Omega(\sqrt{d})$. 

Recent work by \cite{zhou2025global} established a global convergence guarantee for gradient EM in the over-parameterized regime where $n = O(m \log m)$, under certain technical assumptions. One of their key assumptions (see Assumption~3 in \cite{zhou2025global}) requires the separation $\Delta = \Omega(\sqrt{d})$. 
Under this condition, when the initializations for $\mu_i$ are randomly sampled from the ground-truth mixture, each true mean $\mu_i^*$ has, with high probability, a nonempty neighborhood consisting of some $\mu_j$'s. 
Roughly speaking, $\mu_j$ is a \emph{neighbor} of $\mu_i^*$ if its distance to $\mu_i^*$ is substantially smaller than its distance to any other $\mu_k^*$. 
Consequently, the training dynamics in the vicinity of each $\mu_i^*$ behave similarly to the single-component setting analyzed by \cite{xu2024toward}.
When the ground-truth separation falls below $\sqrt{d}$, however, this intuition no longer applies, and existing analyses cannot establish global convergence.
This motivates the central question of this paper:

\begin{center}
\emph{Can global convergence of gradient EM be guaranteed for smaller separation?}
\end{center}

In this paper, we show that the $\Omega(\sqrt{d})$ dependence on the dimension $d$ is intrinsic to global convergence in high dimensions. 
We prove that for any $\epsilon > 0$, in the worst case a separation of order $\Omega(d^{0.5-\epsilon})$ is insufficient to guarantee global convergence in sub-exponential time for sufficiently large $d$. 
In contrast, a separation of order $\Omega(\sqrt{d})$ suffices to ensure global convergence in polynomial time, and therefore constitutes a nearly optimal separation.

\subsection{Related Work}
\paragraph{Analysis of EM and gradient EM algorithm. }
The theoretical analysis of (gradient) EM dates back to the seminal works of \cite{wu1983convergence,xu1996convergence}. Here we focus on EM for learning GMMs. For exactly parameterized 2-GMMs, a line of work established local and, in some cases, global convergence guarantees \citep{balakrishnan2017statistical,daskalakis2017ten,xu2016global,wu2021randomly,weinberger2022algorithm}. Beyond two components, existing results for exactly parameterized $m$-GMMs guarantee only local convergence from good initialization \citep{yan2017convergence,zhao2020statistical,kwon2020algorithm,segol2021improved}, while global convergence is known to be intractable due to spurious local optima even under strong separation \citep{jin2016local}. In the over-parameterized regime, prior work primarily studied using multi-component mixtures to learn one or two Gaussians \citep{dwivedi2020singularity,dwivedi2020sharp,dwivedi2018theoretical,xu2024toward} and the optimization landscape of general $m$-GMMs \citep{chen2024local}. Recently, \cite{zhou2025global} provided the first global convergence guarantee for (gradient) EM beyond the special case of $m=1,2$.

\paragraph{Learning Gaussian Mixtures. } 
Following \cite{dasgupta1999learning}, which introduced efficient GMM learning into theoretical computer science and learning theory, many algorithmic alternatives to EM have been developed. Early works typically clustered the data before parameter estimation \citep{sanjeev2001learning,vempala2004spectral,achlioptas2005spectral}. More recently, theoretical algorithms based on sum-of-squares proofs \citep{hopkins2018mixture,kothari2018robust,liu2022clustering} and method-of-moments techniques \citep{feldman2006pac,anandkumar2014tensor,ge2015learning,bhaskara2014smoothed} have achieved strong guarantees. In practice, however, despite these theoretical advances, EM remains the most widely used method due to its simplicity and efficiency. It is still not well understood under what conditions EM succeeds or whether it can achieve optimal guarantees. This work addresses these questions in the high-dimensional setting.

\paragraph{Separation requirement in learning GMMs.}
The difficulty of learning GMMs is largely governed by the minimum separation $\Delta$ between ground-truth components. It is known that $\Delta = \Omega(\sqrt{\log m})$ is necessary for efficient learning \citep{regev2017learning}. Early algorithms required much stronger separation, such as $\Delta = \Omega(\sqrt{d})$ or $\Delta = \Omega(d^{1/4})$ \citep{dasgupta1999learning,dasgupta2007probabilistic,sanjeev2001learning}, while later spectral methods reduced the dependence on $d$ to $m$ by projecting onto the span of the true means \citep{vempala2004spectral, achlioptas2005spectral}. Method-of-moments approaches avoid explicit separation assumptions but rely on algebraic non-degeneracy \citep{feldman2006pac, anandkumar2014tensor, ge2015learning, bhaskara2014smoothed}. More recent sum-of-squares and higher-order moment methods achieve learning under $\Delta = \Omega(m^{1/\gamma})$ with polynomial complexity in $d^{\mathrm{poly}(1/\gamma)}$ \citep{hopkins2017mixturemodelsrobustnesssum, kothari2018robust, diakonikolas2018list}, and were further improved to $\Delta = \Omega(\sqrt{\log m})$, matching the information-theoretic limit \citep{liu2022clustering, diakonikolas2024implicit}. In contrast, EM is only guaranteed to converge locally under this optimal separation with good initialization \citep{kwon2020algorithm}. We prove that in high dimensions EM fails from random initialization even with large separation, showing that the main challenge for EM lies in the early, pre-local phase.

\section{Preliminaries}
\label{sec:prel}
In this section, we specify the update scheme used in our analysis and present some useful facts.
%In \cite{zhou2025global}, the authors apply the standard gradient EM update for the component means $\mu_i$, while the mixing weights $\pi_i$ are updated using a first-order nearly optimal solution.
%The reason for using a different update algorithm for $\pi_i$ is that gradient EM on $\pi_i$'s cannot guarantee the constraint $\sum_i\pi_i=1$.
%As discussed earlier, gradient EM update on $\pi_i$ cannot guarantee to preserve the constraint $\sum_i\pi_i=1$.
%Our work follows \cite{zhou2025global} in using gradient EM for updating the means.

In mixture models, the mixing weights lie on the simplex $\sum_i \pi_i = 1$, so a naive unconstrained gradient step~\eqref{eq: gradient em} on $\pi$ does not preserve feasibility. Accordingly, throughout this paper we use standard EM updates for $\pi$ and gradient steps only for $\mu$. %The precise update rules are described in \Cref{sec:prel}.
%We therefore adopt an EM-style update for the mixing weights, which we believe is a more natural choice and automatically preserves feasibility.
We believe this is a more natural choice and automatically preserves feasibility.
Formally, at iteration $t$, we consider the following update:
\begin{align*}
\pi_i^{(t+1)}
    &=\E_{x\sim p^*}[\psi_i(x)],\\
\mu_i^{(t+1)}
    &= \mu_i^{(t)} + \eta \, \nabla_{\mu_i} Q(\pi, \mu \mid \pi^{(t)}, \mu^{(t)}).
\end{align*} 
We initialize all $\mu_i$'s independently according to $p^*$, and we let $\pi_i=\frac{1}{n}$ in the initialization. The full procedure is summarized in Algorithm~\ref{alg:em_update}.

\begin{algorithm}[th]
\caption{Population Hybrid (Gradient) EM}
\label{alg:em_update}
\begin{algorithmic}[1]
\REQUIRE Step size $\eta > 0$, number of components $n$
\STATE \textbf{Initialization:} for any $i\in [n]$, set $\mu_i^{(0)} \sim p^*$ and $\pi_i^{(0)} = \frac{1}{n}$
\FOR{$t = 0,1,2,\dots$}
    \FOR{$i = 1$ to $n$}
        \STATE Update $\pi$:
        $\pi_i^{(t+1)} \gets \mathbb{E}_{x \sim p^*}[\psi_i(x)]$
        \STATE Update $\mu$:
        $\mu_i^{(t+1)} \gets \mu_i^{(t)} + \eta \, \nabla_{\mu_i} Q(\mu \mid \pi^{(t)}, \mu^{(t)})=\mu_i^{(t)} - \eta \, \nabla_{\mu_i}\calL(\pi^{(t)}, \mu^{(t)})$
    \ENDFOR
\ENDFOR
\end{algorithmic}
\end{algorithm}

One interpretation of the above update rule is that (gradient) EM is equivalent to minimizing the following KL divergence loss \citep{jin2016local, xu2024toward}:
\[
\calL(\pi,\mu):=D_{KL}(p^*\Vert p)=-\E_{x\sim  p^*}[\log p(x)]+\E_{x\sim  p^*}[\log p^*(x)].
\]

The lemma below formalizes the equivalence between a gradient step on the $Q$-function in gradient EM and a gradient step on the KL loss $\mathcal{L}$. It also provides an explicit form of the gradient update that we rely on in our analysis. Its proof can be found in \cite{zhou2025global}.
\begin{lemma}
\label{lm:gradient}
For each $1\leq i\leq n$, the gradient of the loss is equivalent to the following:
\[
\nabla_{\mu_i}\calL=-\nabla_{\mu_i} Q(\mu \mid \pi^{(t)}, \mu^{(t)})=\sum_j \pi_j^*\E_{x\sim j}\left[\psi_i(x)\sum_k \psi_k(x)(\mu_k-\mu_j^*)\right].
\]
\end{lemma}

In this paper, for each $1\leq j\leq m$, we use $x\sim j$ as a simplified notation for $x\sim \calN(\mu_j^*,I_d)$.
\section{Main results}
\label{sec:mainresults}
In this section, we present our main negative and positive results.
We assume that $m\geq 3$ and that the ground-truth mixture has a unique component of largest
weight. Relabeling the components if necessary, we write
\[
\pi_1^*>\max_{2\leq i\leq m}\pi_i^*.
\]
We use the notation
\[
\pimin^*:=\min_{1\leq i\leq m}\pi_i^*,\qquad
\pigap:=\pi_1^*-\max_{2\leq i\leq m}\pi_i^*,\qquad
\rho:=\min\{\pimin^*,\pigap\}.
\]
For $0<\epsilon\leq0.01$, define
\begin{equation}
\label{eq:dimension-threshold}
d_0(\epsilon):=\exp\!\left(\frac{C_0}{\epsilon}\log\frac{C_0}{\epsilon}\right).
\end{equation}
Here and below, $C_0$ denotes a sufficiently large universal constant.

%We consider $d$ to be large enough, and each pair of ground truth centers is of distance $d^{0.5-\epsilon}$, where $0<\epsilon\leq0.01$. There is a $m-1$ dimensional hyperplane $\calP$ containing all ground truth centers. 
%Also, we assume that $(0,0,\cdots,0)\in \calP$. We assume that $n\geq \frac{3}{\pi_{\min}^*}\log m$, $m\geq 3$. Therefore, $n\geq 10$. Assume $d>\Omega\left((100n/\pimin^*)^{15/\epsilon}\right)$.
%We also let $d>\Omega\left((100n/\pigap)^{15/\epsilon}\right)$.

\subsection{A Separation Lower Bound at \texorpdfstring{$d^{0.5-\epsilon}$}{d to the power 0.5-epsilon}}
Our main result in this paper is the following:
\begin{theorem}
\label{thm:main}
Let $n\geq \frac{3}{\pimin^*}\log m$ and $0<\epsilon\leq0.01$. Suppose that
\begin{equation}
\label{eq:standing-dimension}
d\geq \max\left\{C_0\left(\frac{100n}{\rho}\right)^{15/\epsilon},d_0(\epsilon)\right\}.
\end{equation}
Then there exist
$\mu_1^*,\mu_2^*,\ldots,\mu_m^*$ such that:
\begin{itemize}
\item For each $1\leq i<j\leq m$, $\|\mu_i^*-\mu_j^*\|=d^{0.5-\epsilon}$. 
\item Consider using \Cref{alg:em_update} to learn the ground truth with learning rate $0<\eta\leq 1/(n^{10} d^{0.6})$. With probability at least $1-\frac{1}{m}$, for each $t\leq \exp\left(d/100\right)/\eta$, the loss
\[
\calL\left(\pi^{(t)},\mu^{(t)}\right)\geq \frac{(1-\pi_1^*)^2}{3}.
\]
\end{itemize}
\end{theorem}

\begin{remark}
This theorem shows that for fixed $m$, $n$, and the ground-truth mixing weights, when the dimension $d$ is sufficiently large, gradient EM fails to learn an $m$-component Gaussian mixture in sub-exponential time if the component separation is only guaranteed to be at least $d^{0.5-\epsilon}$. 
Thus no fixed separation exponent below $1/2$ can guarantee efficient learning by this hybrid
update in the stated regime. This is an exponent-level lower bound. It does not rule out thresholds
such as $\sqrt d/\log d$.

The dimension condition also quantifies the range of over-parameterization covered by the
counterexample. For fixed true weights and fixed $\epsilon$, it applies to
$n\leq c\rho d^{\epsilon/15}$, where $c>0$ is a sufficiently small universal constant,
subject to the initialization lower bound on $n$.
Avoiding this particular obstruction therefore requires leaving that range; the theorem does not
assert that increasing $n$ beyond it suffices.
The positive result below concerns the same hybrid update.

\end{remark}

We briefly sketch the proof of this theorem in Section~\ref{sec: proof sketch}. The full proof is deferred to \Cref{sec:thm1proof}.

\subsection{Exponentially Small Loss under \texorpdfstring{$\sqrt{d}$}{Square-Root-D} Separation}
To complement the negative result, we show that $\Delta\geq10\sqrt d$ suffices for the same
hybrid population update to reach exponentially small KL loss from random initialization
within the stated polynomial time bound.

%As we use a slightly different update for the parameter $\pi_i$'s compared to \cite{zhou2025global} (they use a first-order optimal for each $\pi$ update, while we use EM update as discussed in \Cref{sec:prel}), for completeness, we prove a global convergence guarantee for our update rule.
%Actually, 

\begin{theorem}
\label{thm:positive}
Let $n\geq \frac{3}{\pimin^*}\log m$,
\[
d\geq \max\left\{C_0\left(\frac{100n}{\rho}\right)^{1500},d_0(1/100)\right\},
\qquad 0<\eta\leq 1/(n^{10} d^{0.6}).
\]
Assume that for each $1\leq i< j\leq m$, $\|\mu_i^*-\mu_j^*\|\geq 10\sqrt{d}$.
Then with probability at least $\left(1-\frac{1}{m}\right)$, for each
$10d/(\eta\pimin^*)\leq t\leq \exp(d/50)/\eta$, Algorithm~\ref{alg:em_update} returns the loss
\[
\calL\leq \exp(-d/200).
\]
\end{theorem}

The proof of this theorem is deferred to Appendix~\ref{sec:positive:proof}.

\begin{remark}
The positive guarantee is a finite-accuracy statement on the displayed time interval for
sufficiently large $d$. It does not assert that the loss tends to zero as $t\to\infty$ at fixed
$d$. Together with \cref{thm:main}, it identifies the separation exponent $1/2$ up to an
arbitrarily small fixed polynomial factor for Algorithm~\ref{alg:em_update}.
\end{remark}

\section{Proof sketch of \texorpdfstring{\Cref{thm:main}}{the negative theorem}}\label{sec: proof sketch}
In this section, we provide a proof sketch of \Cref{thm:main}. We show that random initialization satisfies the desired property in Section~\ref{sec:init}. We then analyze the dynamics of gradient EM (Algorithm~\ref{alg:em_update}) in three stages, presented in Section~\ref{sec:stage1}, Section~\ref{sec:stage2}, and Section~\ref{sec:stage3}, respectively.
%The ground truth in our proof is a set of points satisfying $\|\mu_i^*-\mu_j^*\|=d^{0.5-\epsilon}$ for any $i\neq j$.
The lower bound is established over the following class of ground-truth instances: the true means $\{\mu_i^*\}_{i=1}^m$ are separated satisfying
 $$\|\mu_i^*-\mu_j^*\|=d^{0.5-\epsilon}$$ for any $i\ne j$.
%Our proof consists of an initialization analysis and three stages of convergence dynamics analysis.
% In the following, we always assume that for each $i\neq j$, $\|\mu_i^*-\mu_j^*\|=d^{0.5-\epsilon}$.
Translate the affine span of the true means to contain the origin, and let $\calP$ be an
$(m-1)$-dimensional linear subspace containing it. For the centered equilateral construction,
$\calP$ is exactly that span. The initialization and Stage~1 lemmas below also hold for arbitrary
true means with minimum separation at least $d^{0.5-\epsilon}$ and $0<\epsilon\leq0.01$;
their proofs do not require an upper bound on separation. This generality will be used in
Appendix~\ref{sec:positive:proof}.

\subsection{Initialization}\label{sec:init}
For initialization, we prove that with high probability, we have the following properties.

\begin{lemma}
\label{thm:init}
Under initialization from $p^*$, let $S_i$ be the set of indices $j$ for which $\mu_j$ was sampled from true component $i$. For each $\mu_i$, we consider its projection on $\calP$ to be $\parmu_i$. With probability at least $\left(1-\frac{1}{m}\right)$, by permuting indices of $\mu_i$'s (and corresponding $\parmu_i$'s) we have:
\begin{itemize}
\item For each $1\leq i\leq m$, $|S_i|\geq 1$.
\item For each $1\leq i\leq m$, we have $i=\arg\min_{j\in S_i}\|\mu_j-\parmu_j\|$. Also, for any $k\in S_i,k\neq i$, $\|\mu_k-\parmu_k\|\geq \|\mu_i-\parmu_i\|+\frac{1}{n^3}$.
\item For each $1\leq i\leq m,j\in S_i$, $\|\parmu_j-\mu_i^*\|\leq n$.
\item For each $1\leq j\leq n$, $\sqrt{d}-n\leq \|\mu_j-\parmu_j\|\leq \sqrt{d}+n$.
\item For each $1\leq i\leq n$, there exists $\botmu_i\bot\calP$ such that 
\begin{enumerate}
\item $(\mu_i-\botmu_i)\bot\botmu_i$.
\item $\|\mu_i-\botmu_i-\parmu_i\|\leq n$.
\item $\botmu_i,1\leq i\leq n$ are pairwise orthogonal.
\end{enumerate}
\end{itemize}
\end{lemma}

The proof of this lemma is deferred to \Cref{subsec:init}.
Here we briefly outline the main ideas behind the initialization analysis. The first item is a direct result from $n\geq \frac{3}{\pimin^*}\log m$. Now we consider the other items.

Notice that each $\mu_j$ is sampled from one of the ground-truth components $\calN(\mu_i^*, I_d)$. 
We decompose $\mu_j$ into its projection $\parmu_j$ onto the subspace $\calP$ and its orthogonal component $\mu_j - \parmu_j$. 
From this viewpoint, $\parmu_j$ and $\mu_j - \parmu_j$ can be treated as independent Gaussian vectors drawn from an $(m-1)$-dimensional and a $(d-m+1)$-dimensional Gaussian distribution, respectively.
Items (2)–(4) follow from a careful analysis of the concentration properties of these two components. 
In particular, the norms of the orthogonal components $\|\mu_j - \parmu_j\|$ concentrate sharply around $\sqrt{d}$, while the projected components $\parmu_j$ remain within a bounded distance of the corresponding ground-truth means $\mu_i^*$.

Finally, since the vectors $\mu_j - \parmu_j$ lie in a $(d-m+1)$-dimensional space and are independently sampled, they are nearly orthogonal to one another with high probability. 
Item (5) formalizes this intuition by introducing vectors $\botmu_j$ that approximate the directions of $\mu_j - \parmu_j$ and are pairwise orthogonal.

In the following, we analyze the training dynamics.

\subsection{Stage 1}\label{sec:stage1}
In this stage, given an initialization satisfying all items in \cref{thm:init}, we prove the following lemma:
\begin{lemma}
\label{thm:stage1}
Assume minimum true-center separation at least $d^{0.5-\epsilon}$ and the standing dimension and step-size conditions. Given the initialization in \cref{thm:init}, when $t=\big\lceil1000n^{10}d^{\epsilon-0.5}/\eta\big\rceil$, the following are true:
\begin{itemize}
\item For $1\leq i\leq m$, $\pi_i>\pi_i^*/n$.
\item For any $i\in [m],j\in S_i$ such that $i\neq j$,
$\|\mu_i-\parmu_i\|\leq \|\mu_j-\parmu_j\|-d^{\epsilon}$;
\item For each $1\leq i\leq m,j\in S_i$, $\|\parmu_j-\mu_i^*\|\leq 2n$;
\item For each $1\leq j\leq n$, $\sqrt{d}-2d^{2\epsilon}\leq \|\mu_j-\parmu_j\|\leq \sqrt{d}+2n$.
\end{itemize}
\end{lemma}

The proof of this lemma is in \Cref{subsec:stage1}.
Here we briefly outline the main idea of this lemma and how to prove it.
This stage shows that, given the initialization in \cref{thm:init}, in some steps, for any $i$, there is a $\mu_j$ that is far closer to $\mu_i^*$ than any other $\mu_k$. Moreover, the projection of each $\mu_j$ on $\calP$ doesn't move much. An illustration is in Figure~\ref{fig:stage1}.

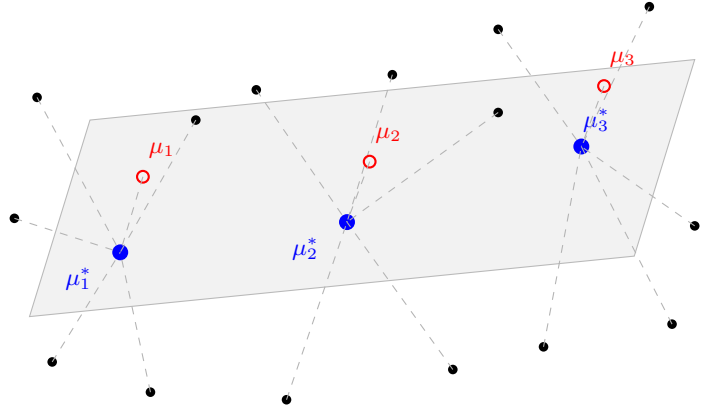
\begin{figure}[t]
\centering
\begin{tikzpicture}[scale=1.0]

% --- Force a symmetric-ish bounding box so the picture is truly centered ---
% (adjust these corners if you later move points)
\path[use as bounding box] (-4.8,-2.5) rectangle (4.8,3.4);

% --- Plane as a parallelogram (centered) ---
\fill[gray!10] (-4.2,-1.0) -- (3.8,-0.2) -- (4.6,2.4) -- (-3.4,1.6) -- cycle;
\draw[gray!55] (-4.2,-1.0) -- (3.8,-0.2) -- (4.6,2.4) -- (-3.4,1.6) -- cycle;

% --- Ground truths (separated + non-collinear) ---
\coordinate (g1) at (-3.0,-0.15);
\coordinate (g2) at (0.0,0.25);
\coordinate (g3) at (3.1,1.25);

\fill[blue] (g1) circle (3.0pt);
\fill[blue] (g2) circle (3.0pt);
\fill[blue] (g3) circle (3.0pt);

% Labels for ground-truth means
\node[blue] at ($(g1)+(-0.55,-0.35)$) {\small $\mu_1^*$};
\node[blue] at ($(g2)+(-0.55,-0.35)$) {\small $\mu_2^*$};
\node[blue] at ($(g3)+(0.20,0.35)$) {\small $\mu_3^*$};

% =======================
% Points around g1 (off-plane) + projection lines to g1
\fill[black] (-4.1, 1.9) circle (1.8pt); \draw[dashed,gray!60] (-4.1,1.9) -- (g1);
\fill[black] (-3.9,-1.6) circle (1.8pt); \draw[dashed,gray!60] (-3.9,-1.6) -- (g1);
\fill[black] (-2.0, 1.6) circle (1.8pt); \draw[dashed,gray!60] (-2.0,1.6) -- (g1);
\fill[black] (-4.4, 0.3) circle (1.8pt); \draw[dashed,gray!60] (-4.4,0.3) -- (g1);
\fill[black] (-2.6,-2.0) circle (1.8pt); \draw[dashed,gray!60] (-2.6,-2.0) -- (g1);

% closest point for g1
\coordinate (r1) at (-2.7,0.85);
\draw[red,thick] (r1) circle (2.2pt);
\draw[dashed,gray!60] (r1) -- (g1);
\node[red] at ($(r1)+(0.25,0.35)$) {\small $\mu_1$};

% =======================
% Points around g2
\fill[black] (-1.2, 2.0) circle (1.8pt); \draw[dashed,gray!60] (-1.2,2.0) -- (g2);
\fill[black] ( 1.4,-1.7) circle (1.8pt); \draw[dashed,gray!60] (1.4,-1.7) -- (g2);
\fill[black] ( 2.0, 1.7) circle (1.8pt); \draw[dashed,gray!60] (2.0,1.7) -- (g2);
\fill[black] (-0.8,-2.1) circle (1.8pt); \draw[dashed,gray!60] (-0.8,-2.1) -- (g2);
\fill[black] ( 0.6, 2.2) circle (1.8pt); \draw[dashed,gray!60] (0.6,2.2) -- (g2);

% closest point for g2
\coordinate (r2) at (0.3,1.05);
\draw[red,thick] (r2) circle (2.2pt);
\draw[dashed,gray!60] (r2) -- (g2);
\node[red] at ($(r2)+(0.25,0.35)$) {\small $\mu_2$};

% =======================
% Points around g3
\fill[black] ( 2.0, 2.8) circle (1.8pt); \draw[dashed,gray!60] (2.0,2.8) -- (g3);
\fill[black] ( 4.6, 0.2) circle (1.8pt); \draw[dashed,gray!60] (4.6,0.2) -- (g3);
\fill[black] ( 4.0, 3.1) circle (1.8pt); \draw[dashed,gray!60] (4.0,3.1) -- (g3);
\fill[black] ( 2.6,-1.4) circle (1.8pt); \draw[dashed,gray!60] (2.6,-1.4) -- (g3);
\fill[black] ( 4.3,-1.1) circle (1.8pt); \draw[dashed,gray!60] (4.3,-1.1) -- (g3);

% closest point for g3
\coordinate (r3) at (3.4,2.05);
\draw[red,thick] (r3) circle (2.2pt);
\draw[dashed,gray!60] (r3) -- (g3);
\node[red] at ($(r3)+(0.25,0.35)$) {\small $\mu_3$};

\end{tikzpicture}
\caption{The gray parallelogram represents the ground-truth subspace $\calP$, and the blue points denote the ground-truth means $\mu_1^*,\mu_2^*,\mu_3^*$.
For each $1 \le i \le m$, we assumed $\mu_i$ as the initialized parameter in $S_i$ whose distance to $\calP$ is minimal.
The remaining parameters in $S_i$ are denoted by black points.
At the end of Stage~1, the distance from $\mu_i$ to $\calP$ is smaller than that of any other $\mu_j \in S_i$ by at least $d^{\epsilon}$.}
\label{fig:stage1}
\end{figure}

The main idea of this lemma is to use induction to prove that in $\big\lceil1000n^{10}d^{\epsilon-0.5}/\eta\big\rceil$ steps, all properties in \cref{thm:init} still hold with some small relaxations.
By such properties, we can use \cref{lm:gradient} to analyze the gradient.
Recall the formula in \cref{lm:gradient}.
We prove that for each $i\in S_l$, only the $(\mu_i-\mu_l^*)$ term dominates the gradient.
Therefore, for any $i\in S_j$, intuitively we have
\begin{equation}
\label{eq:sketch:stage1}
\nabla_{\mu_i}\calL\approx \pi_j^*\E_{x\sim j}\left[\psi_i(x)^2(\mu_i-\mu_j^*)\right].
\end{equation}
Finally, we show that when $d$ is large enough, for each $j\in S_i$ such that $j\neq i$,
\[
\E_{x\sim i}[\psi_i(x)^2]\geq \E_{x\sim i}[\psi_j(x)^2]+\operatorname{poly}(1/n).
\]
Therefore, since the update rule is
\[
\mu_i^{(t+1)}=\mu_i^{(t)}-\eta\nabla_{\mu_i}\calL(\pi^{(t)},\mu^{(t)}),
\]
by (\ref{eq:sketch:stage1}), we can prove that for each $1\leq i\leq m$, $\mu_i$ moves towards $\mu_i^*$ at a higher speed than any other $\mu_j$ for $j\in S_i$.
Also, for each $j\in S_i$, the moving direction of $\mu_j$ is almost towards $\mu_i^*$.
Therefore, by a careful analysis, we can prove \cref{thm:stage1}.

\subsection{Stage 2}\label{sec:stage2}
In this subsection, we are given the parameters generated by Stage 1.
We prove the following lemma:

\begin{lemma}
\label{thm:stage2}
Given the initialization corresponding to Stage 1 (i.e., it satisfies the four conditions in \cref{thm:stage1}), after $\lceil\log d/(\pi_1^*\eta)\rceil$ steps, the following are true:
\begin{itemize}
\item $\|\mu_1-\parmu_1\|\leq 1$.
\item For any $2\leq i\leq n$, we have $\pi_i\leq 2\exp(-d^{2\epsilon/3}/3)$.
\item For each $1\leq i\leq m$, $\|\parmu_1-\mu_i^*\|\leq d^{0.5-\epsilon}+2n+\frac{800}{(\pigap)^2}$.
\item For any $2\leq i\leq n$, 
$\sqrt{d}-100d^{0.5-2\epsilon}/\pigap\leq\|\mu_i-\parmu_i\|
\leq\sqrt d+2n+800/(\pigap)^2$.
\item For $2\leq i\leq n,i\in S_j$, $\|\parmu_i-\mu_j^*\|\leq 2n+\frac{800}{(\pigap)^2}$.
\end{itemize}
\end{lemma}

This lemma shows that after Stage 2, only one $\mu_i$ is close to $\calP$, while others are far away from $\calP$. The proof of this lemma is in \Cref{subsec:stage2}. Here we still outline the main idea of this lemma.
The intuition is that given the stage 1 input, for any $i>m$, $\|\nabla_{\mu_i}\calL\|$ is exponentially small, and for $1\leq i\leq m$, $\mu_i$ moves towards $\calP$ at a speed proportional to $\pi_i^*$ at first. As we assumed $\pi_1^*$ to be the largest, and the separation of each two $\mu_j^*$ is $d^{0.5-\epsilon}$, at last $\mu_1$ becomes the closest parameter to $\calP$, while other $\mu_k$'s are far away from $\calP$. An illustration is in Figure~\ref{fig:stage2}.

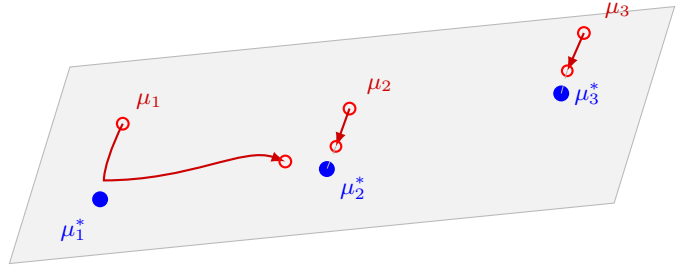
\begin{figure}
\centering
\begin{tikzpicture}[scale=1.0, >=latex]

% --- Plane as a parallelogram ---
\fill[gray!10] (-4.2,-1.0) -- (3.8,-0.2) -- (4.6,2.4) -- (-3.4,1.6) -- cycle;
\draw[gray!55] (-4.2,-1.0) -- (3.8,-0.2) -- (4.6,2.4) -- (-3.4,1.6) -- cycle;

% --- Ground truths (mu_i^*) ---
\coordinate (g1) at (-3.0,-0.15);
\coordinate (g2) at ( 0.0, 0.25);
\coordinate (g3) at ( 3.1, 1.25);

\fill[blue] (g1) circle (3.0pt);
\fill[blue] (g2) circle (3.0pt);
\fill[blue] (g3) circle (3.0pt);

\node[blue] at (-3.35,-0.55) {\small $\mu_1^*$};
\node[blue] at (0.35, 0.00) {\small $\mu_2^*$};
\node[blue] at ( 3.45, 1.25) {\small $\mu_3^*$};

% --- Start-of-stage points (mu_i) ---
\coordinate (m1) at (-2.7,0.85);
\coordinate (m2) at ( 0.3,1.05);
\coordinate (m3) at ( 3.4,2.05);

\draw[red,thick] (m1) circle (2.2pt);
\draw[red,thick] (m2) circle (2.2pt);
\draw[red,thick] (m3) circle (2.2pt);

\node[red!80!black] at (-2.35,1.15) {\small $\mu_1$};
\node[red!80!black] at ( 0.70,1.35) {\small $\mu_2$};
\node[red!80!black] at ( 3.85,2.35) {\small $\mu_3$};

% --- Stage-2 dynamics ---
% mu_1: first goes toward mu_1^*, then curves and ends on P, still closer to mu_1^*
\coordinate (m1mid)   at (-2.95,0.10);   % along the direction to g1
\coordinate (m1final) at (-0.55,0.35);   % final point on the plane (your suggested endpoint)

\draw[thick, red!80!black, ->]
  (m1) .. controls (-2.85,0.55) and (-2.95,0.25) .. (m1mid)
       .. controls (-1.80,0.10) and (-1.10,0.55) .. (m1final);

\draw[red,thick] (m1final) circle (2.2pt);
\node[red!80!black] at (-0.15,0.70) {};

% mu_2, mu_3: move toward their targets but stop short
\coordinate (m2end) at (0.12,0.55);
\coordinate (m3end) at (3.18,1.55);

\draw[thick, red!80!black, ->] (m2) -- (m2end);
\draw[thick, red!80!black, ->] (m3) -- (m3end);

\draw[red,thick] (m2end) circle (2.0pt);
\draw[red,thick] (m3end) circle (2.0pt);

% optional dashed "remaining gap" to show they cannot reach
\draw[dashed, gray!60] (m2end) -- (g2);
\draw[dashed, gray!60] (m3end) -- (g3);

\end{tikzpicture}
\caption{
Illustration of Stage 2 dynamics.
Similar as in Figure~\ref{fig:stage1}, the gray parallelogram represents the ground-truth subspace $\calP$, and the blue points denote the ground-truth means $\mu_1^*,\mu_2^*,\mu_3^*$.
During this stage, at first each $\mu_i$ moves towards $\mu_i^*$ for $1\leq i\leq m$, but at last $\mu_1$ becomes very close to $\calP$ while others stay far from $\calP$.
}
\label{fig:stage2}
\end{figure}

More precisely, \cref{lm:stage2:early} first proves that every representative has weight at
least half its true weight during an initial interval. Before that interval ends, component~1
becomes closer than every non-representative to every true source. Thereafter non-representatives
are suppressed by comparison with component~1, without any assumption about their weights
relative to the other representatives. These responsibility estimates give the refined gradients
that close the Stage~2 induction.

The collapse becomes quantitative after $t>10d^{-2\epsilon}/(\eta\pigap)$. At that point
$\E_{x\sim p^*}[\psi_i(x)]$ is exponentially small for every $i\geq2$, so only $\mu_1$ experiences a
non-negligible update. Since the number of iterations before this point is small, the motion parallel
to $\calP$ remains controlled.
In this regime, we show that
\[
\nabla_{\mu_1}\calL \approx \sum_j \pi_j^* \E_{x\sim j}\!\left[\psi_1(x)^2(\mu_1-\mu_j^*)\right].
\]
Therefore, $\mu_1$ is driven toward the subspace $\calP$ and converges to a point near both $\calP$ and the convex hull of the ground-truth means $\{\mu_j^*\}$.
This describes the mechanism formalized in \cref{thm:stage2}.

\subsection{Stage 3}\label{sec:stage3}
In this section, we prove that given the input of Stage 2, when $t\leq \exp\left(d/100\right)/\eta$, the loss is lower bounded by $\frac{(1-\pi_1^*)^2}{3}$.
We need to prove the following lemma.

\begin{lemma}
\label{thm:stage3}
Given the output of Stage~2 (i.e., the conditions in \cref{thm:stage2}), when
$t\leq \exp\left(d/100\right)/\eta$,
\[
\calL\left(\pi^{(t)},\mu^{(t)}\right)\geq \frac{(1-\pi_1^*)^2}{3}.
\]
\end{lemma}

The proof of this lemma is in \Cref{subsec:stage3}. 
Here we also outline the main ideas.
Using the output properties in \cref{thm:stage2}, we prove that every $\mu_i$, $i\geq2$, moves only
an exponentially small distance for $t\leq \exp(d/100)/\eta$. Therefore, we prove the following lemma:
\begin{lemma}[See also \cref{lm:stage3}]
Given the initialization corresponding to \cref{thm:stage2}, for any $t\leq \exp(d/100)/\eta$, we always have
\begin{itemize}
\item $\|\mu_1-\parmu_1\|\leq 1$.
\item At local time $t=0$, $\pi_i\leq2\exp(-d^{2\epsilon/3}/3)$ for $i\geq2$;
at every $t\geq1$, $\pi_i\leq2\exp(-d/20)$.
\item For each $1\leq i\leq m$, $\|\parmu_1-\mu_i^*\|\leq d^{0.5-\epsilon}+2n+\frac{800}{(\pigap)^2}$.
\item For any $2\leq i\leq n$, 
$\sqrt{d}-200d^{0.5-2\epsilon}/\pigap\leq \|\mu_i-\parmu_i\|\leq 2\sqrt{d}$.
\item For any $2\leq i\leq n$ and $1\leq j\leq m$, $\|\parmu_i-\mu_j^*\|\leq 2d^{0.5-\epsilon}$.
\end{itemize}
\end{lemma}
Then we can prove the lower bound of loss based on the following lemma:
\begin{lemma}[See also \cref{lm:stage3:loss}]
If for all $j\geq 2$ and $i\in [m]$, $\|\mu_j-\mu_i^*\|\geq \sqrt{d}/2$, then the loss satisfies
$\calL\geq \frac{(1-\pi_1^*)^2}{3}$.
\end{lemma}
Notice that during Stage 1 and Stage 2, we also always have $\|\mu_j-\mu_i^*\|\geq \sqrt{d}/2$ for $j\geq2$ and $i\in[m]$. This completes the proof of \cref{thm:main}.

\section{Experiments}
\label{sec:experiments}
In this section, we present several numerical simulations that illustrate the behavior of our updates and support our results. 

In all experiments, we use the EM update for the mixing weights $\pi_i$ and the gradient EM update for the means $\mu_i$, both computed on fresh samples at each iteration. 
For the experiments shown here, the learning rate is set to $0.05$ and the batch size is $8192$. We train the student model for 2000 steps.
We consider a mixture model with two ground-truth components,
\[
\pi_1^* = 0.3,\qquad 
\pi_2^* = 0.7,\qquad
\mu_1^*=(3,0,0,\ldots),\qquad 
\mu_2^*=( -3,0,0,\ldots).
\]
The initialization of $(\pi_i, \mu_i)$ follows Algorithm~\ref{alg:em_update}. We use $n = 10$ components to fit the two ground-truth components.

We examine two cases with different dimension $d$, using seed 0 in both:
\begin{itemize}
    \item $d=10$: a low-dimensional setting in which the algorithm successfully fits both ground-truth components.
    \item $d=500$: a high-dimensional setting in which most learned means fail to approach the 1D signal subspace, resulting in an \emph{underfitting regime}.
\end{itemize}

Figure~\ref{fig:traj} shows the trajectories of all learned $\mu_i$'s under the 2D projection of
\[
x_{i}(t) = \mu_{i,1}^{(t)},\qquad 
y_{i}(t) = \big\|\mu_{i,2:d}^{(t)}\big\|_2.
\]
The ground-truth means appear as red stars on the horizontal axis ($y=0$), and the evolution of each learned component is plotted as a curve.

The empirical results are consistent with our theoretical analysis.
When the dimension $d$ is small, learning the ground-truth distribution is relatively easy.
In the $d=500$ run shown, a single fitted component receives essentially all the weight and its mean moves between $\mu_1^*$ and $\mu_2^*$, demonstrating the difficulty of learning Gaussian mixtures when the separation is small relative to the dimension.
Additional experiments across three random seeds and with dimension-dependent separation at $d=10^6$ are reported in Appendix~\ref{app:simulation-additional}. The code used to reproduce the experiments in this paper is publicly available at \url{https://github.com/Ryanzhang887/EM_Separation}.

\begin{figure}[t]
    \centering
    \includegraphics[width=\linewidth]{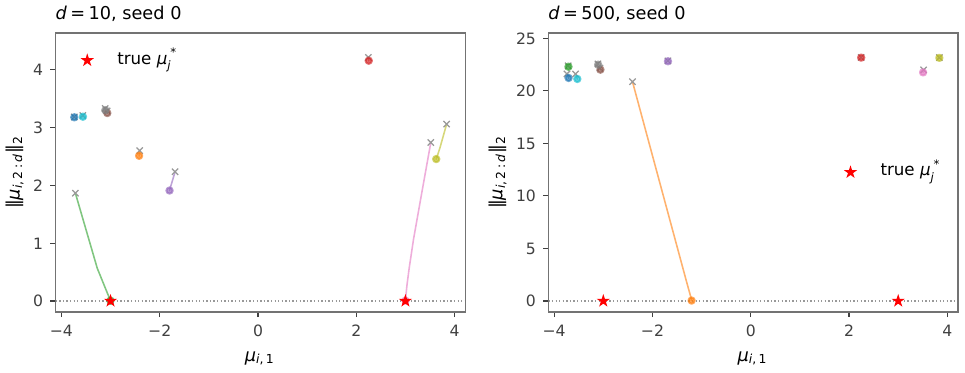}
    \caption{Trajectories of the learned means $\mu_i$ for $d=10,500$, generated by the released code with seed 0. \textbf{Left:} $d=10$. Both ground-truth components are fit by some $\mu_i$.
    \textbf{Right:} $d=500$. Only one component converges toward the signal axis, while the remaining ones stay far away, demonstrating underfitting in this run. Gray crosses mark initialization, colored dots mark final means, and red stars mark true means. All three seeds are shown in Appendix~\ref{app:simulation-additional}.}
    \label{fig:traj}
\end{figure}

\section{Conclusion}
\label{sec:conclusion}
In this paper, we study the ground-truth separation required for global convergence of gradient EM in the high-dimensional regime. Our results show that $\Omega(\sqrt{d})$ separation is nearly optimal when the dimension $d$ is very large. Our analysis of the training dynamics characterizes how the parameters converge to the ground truth. However, our work focuses on the regime where the number of components $n$ is much smaller than the dimension $d$. One natural direction for future work is to investigate whether heavier over-parameterization, with $n$ comparable to or larger than $d$, can enable global convergence under weaker separation. Another direction is to extend our analysis to the classical EM algorithm. Finally, we believe that balanced mixtures, in which all ground-truth mixing weights are equal, may exhibit a different separation threshold; understanding this setting remains an open problem.

\section*{Acknowledgements}

SSD was supported by NSF DMS 2134106, NSF IIS 2143493, NSF IIS 2229881, Sloan Fellowship, and the AI2050 program at Schmidt Sciences.
MF, MZ, and WX acknowledge the support of NSF TRIPODS II DMS 2023166. The work of MF was also supported by awards NSF CCF 2212261 and NSF CCF 2312775 and the Moorthy Family Professorship at UW.

\bibliographystyle{plainnat}
\bibliography{references}

% \section*{References}

% References follow the acknowledgments in the camera-ready paper. Use unnumbered first-level heading for
% the references. Any choice of citation style is acceptable as long as you are
% consistent. It is permissible to reduce the font size to \verb+small+ (9 point)
% when listing the references.
% Note that the Reference section does not count towards the page limit.
% \medskip

% {
% \small

% [1] Alexander, J.A.\ \& Mozer, M.C.\ (1995) Template-based algorithms for
% connectionist rule extraction. In G.\ Tesauro, D.S.\ Touretzky and T.K.\ Leen
% (eds.), {\it Advances in Neural Information Processing Systems 7},
% pp.\ 609--616. Cambridge, MA: MIT Press.

% [2] Bower, J.M.\ \& Beeman, D.\ (1995) {\it The Book of GENESIS: Exploring
%   Realistic Neural Models with the GEneral NEural SImulation System.}  New York:
% TELOS/Springer--Verlag.

% [3] Hasselmo, M.E., Schnell, E.\ \& Barkai, E.\ (1995) Dynamics of learning and
% recall at excitatory recurrent synapses and cholinergic modulation in rat
% hippocampal region CA3. {\it Journal of Neuroscience} {\bf 15}(7):5249-5262.
% }

%%%%%%%%%%%%%%%%%%%%%%%%%%%%%%%%%%%%%%%%%%%%%%%%%%%%%%%%%%%%

\appendix

\section{Further Preliminaries}
\begin{lemma}[Chi distribution]\label{lem:chi-density}
Let $X\sim\mathcal N(0,I_d)$ and $R:=\|X\|$. Then $R\sim\chi_d$ with density
\[
f_d(r)=\frac{1}{2^{\frac d2-1}\Gamma(\tfrac d2)}\,r^{d-1}e^{-r^2/2},\qquad r\ge 0.
\]
Moreover, for all $d\geq 2$,
\[
\|f_d\|_\infty \leq 1.
\]
\end{lemma}

\begin{proof}
The density $f_d$ is Chi distribution (see, e.g., \cite{johnson1995continuous}). 

We first maximize $f_d(r)$ over $r\ge0$. Since $f_d(r)>0$ for $r>0$, we equivalently maximize
\[
\log f_d(r)=(d-1)\log r-\frac{r^2}{2}+C_d,
\]
where $C_d$ is a constant independent of $r$. The differentiate of it corresponding to $r$ is:
\[
(\log f_d(r))^\prime=\frac{d-1}{r}-r,
\]
so $f_d(r)$ is maximized at $r=\sqrt{d-1}$. We have
\[
f_d(\sqrt{d-1})=\frac{(d-1)^{\frac{d-1}{2}}e^{-\frac{d-1}{2}}}{2^{\frac d2-1}\Gamma(\frac d2)}.
\]
A standard Gamma function lower bound (see, e.g., section 5.6 in \cite{olver2010nist}) shows
\[
\Gamma(x)>\sqrt{2\pi}\cdot x^{x-0.5}e^{-x}.
\]
Therefore,
\begin{align*}
\|f_d\|_{\infty}&=f_d(\sqrt{d-1})=\frac{(d-1)^{\frac{d-1}{2}}e^{-\frac{d-1}{2}}}{2^{\frac d2-1}\Gamma(\frac d2)}\\
&\leq \frac{\sqrt{e}}{\sqrt{\pi}}\left(1-\frac{1}{d}\right)^{\frac{d-1}{2}}\leq 1.
\end{align*}
\end{proof}

We also need the following Mill's inequality, which can be found in \cite{wasserman2013all}.
\begin{lemma}[Mill's Inequality]
\label{lm:mill}
For any $t>0$, we have
\[
\Pr_{x\sim\calN(0,1)}[x\geq t]\leq \frac{1}{t}\exp(-t^2/2).
\]
\end{lemma}

The following Gaussian tail bound is a direct consequence of Lemma 1 in \cite{laurent2000adaptive}.

\begin{lemma}[Gaussian Tail Bound]
\label{lm:gausstail}
Let $d$ be a positive integer. For any positive real number $t>0$, we have
\[
\Pr_{x\sim\calN(\boldsymbol{0},I_d)}\left[\Vert x\Vert_2\geq t+\sqrt{d}\right]\leq \exp\left(-t^2/2\right).
\]
Moreover, for any $0\leq t\leq \frac{1}{2}\sqrt{d}$, we have
\[
\Pr_{x\sim\calN(\boldsymbol{0},I_d)}\left[\Vert x\Vert_2\leq \sqrt{d}-t\right]\leq \exp\left(-t^2/2\right).
\]
\end{lemma}

\begin{lemma}\label{lemma:chernoff}
\emph{(Chernoff bounds (see, e.g., \cite{mitzenmacher2005probability})).}
Let $X_1,\dots,X_n$ be independent Bernoulli variables with $\E[X_i]=p_i$. 
Let $X=\sum_{i=1}^nX_i$ and $\mu=\E[X]=\sum_{i=1}^np_i$. 
For $\nu\geq6\mu$, one has
\(
\Pr[X\geq\nu]\leq 2^{-\nu}
\).
\end{lemma}

\section{Proof of \texorpdfstring{\cref{thm:main}}{the negative theorem}}
\label{sec:thm1proof}

Throughout this appendix, we assume $m\geq3$, relabel the true components so that
$\pi_1^*>\max_{2\leq i\leq m}\pi_i^*$, and assume
$n\geq 3(\pimin^*)^{-1}\log m$. We also impose the dimension condition
\cref{eq:standing-dimension}, with $0<\epsilon\leq0.01$.
The initialization and Stage~1 arguments below apply to arbitrary true means satisfying
\[
\min_{i\neq j}\|\mu_i^*-\mu_j^*\|\geq d^{0.5-\epsilon}.
\]
There is no upper bound on their diameter. Translate their affine span to contain the origin,
and choose a fixed $(m-1)$-dimensional linear subspace $\calP$ containing it.
Equal pairwise separation is imposed only in Stages~2 and~3 of the negative construction.

\subsection{Initialization Analysis}
\label{subsec:init}
In this section, we prove \cref{thm:init}.

\begin{proof}
We verify the five items in \cref{thm:init} on a common high-probability event.

For item~1, a particular true component $i$ receives no initialized fitted center with probability
$(1-\pi_i^*)^n$. Since $1-u\leq e^{-u}$ and
$n\pi_i^*\geq n\pimin^*\geq3\log m$,
\[
\Pr[|S_i|=0]=(1-\pi_i^*)^n\leq e^{-n\pi_i^*}\leq m^{-3}.
\]
A union bound over the $m$ true components gives failure probability at most $m^{-2}$.

We next prove items~2 and~4. Conditional on the source labels of the initialization, the vectors
$\mu_j-\parmu_j$, $1\leq j\leq n$, are independent and distributed as
$\calN(0,I_{d-m+1})$. For each pair $1\leq k<j\leq n$, \cref{lem:chi-density} gives
\[
\left|\|\mu_j-\parmu_j\|-\|\mu_k-\parmu_k\|\right|\leq \frac{1}{n^3}
\]
with probability at most $2/n^3$. There are at most $n^2/2$ pairs, so the probability that any two
orthogonal norms differ by at most $1/n^3$ is at most $1/n$. On the complementary event, each group
$S_i$ has a unique smallest orthogonal norm; relabeling the fitted centers inside each group makes
that minimizer index $i$ and proves item~2.

For item~4, \cref{lm:gausstail} gives, for each $j$,
\[
\Pr\left[\Big|\|\mu_j-\parmu_j\|-\sqrt{d-m+1}\Big|>n/2\right]\leq 2\exp(-n^2/8)\leq 1/n^3.
\]
Because $m\leq n$ and
\[
0\leq\sqrt d-\sqrt{d-m+1}
=\frac{m-1}{\sqrt d+\sqrt{d-m+1}}\leq n/2,
\]
the displayed concentration event implies
$\sqrt d-n\leq\|\mu_j-\parmu_j\|\leq\sqrt d+n$. A union bound over $j$ therefore proves item~4
except on an event of probability at most $2n\exp(-n^2/8)$.

For item~3, if $j\in S_i$, then $\parmu_j-\mu_i^*$ is a standard Gaussian vector on the
$(m-1)$-dimensional subspace $\calP$. Hence \cref{lm:gausstail} gives
\[
\Pr\!\left[\|\parmu_j-\mu_i^*\|>n\right]
\leq
\exp\left(-(n-\sqrt{m-1})^2/2\right)\leq \exp(-(0.75n)^2/2)\leq \exp(-n^2/4).
\]
There are $n$ initialized centers, so item~3 fails with probability at most
$n\exp(-n^2/4)$.

It remains to prove item~5. For this construction only, use the original sampling order,
before the data-dependent relabeling that selects the representatives. Process the vectors $\mu_i-\parmu_i$ in increasing order of $i$ and set
\[
\calP_i:=\operatorname{span}\{\mu_1-\parmu_1,\ldots,\mu_{i-1}-\parmu_{i-1}\},
\qquad 1\leq i\leq n+1.
\]
Every $\calP_i$ is contained in $\calP^\bot$. Let $\mu_i'$ be the orthogonal projection of
$\mu_i-\parmu_i$ onto $\calP_i$. Conditional on the preceding vectors, $\mu_i'$ is a standard
Gaussian vector on a space of dimension at most $i-1\leq n$. Thus \cref{lm:gausstail} implies
\[
\Pr[\|\mu_i'\|>n]\leq \exp(-(n-\sqrt{n})^2/2)\leq \exp(-n^2/5).
\]
On the event that $\|\mu_i'\|\leq n$ for every $i$, define
\[
\botmu_i:=\mu_i-\parmu_i-\mu_i'.
\]
By construction, $\botmu_i\perp(\calP+\calP_i)$, while
$\mu_i-\botmu_i=\parmu_i+\mu_i'\in\calP+\calP_i$. Therefore
$(\mu_i-\botmu_i)\bot\botmu_i$.
Moreover,
$\|\mu_i-\botmu_i-\parmu_i\|=\|\mu_i'\|\leq n$.
Finally, $\botmu_i$ is a linear combination of
$\mu_1-\parmu_1,\ldots,\mu_i-\parmu_i$, and hence belongs to $\calP_{i+1}$. If $i<j$, then
$\botmu_i\in\calP_j$ whereas $\botmu_j\perp\calP_j$, proving
$\botmu_i\perp\botmu_j$. This establishes all three parts of item~5 in the original sampling order.
Applying the same permutation to the fitted means, their projections, and these constructed
vectors preserves all three conclusions. A union bound over the original indices gives
failure probability at most $n\exp(-n^2/5)$.

Combining the preceding estimates, the total failure probability is at most
\[
m^{-2}+n^{-1}+2n\exp(-n^2/8)+n\exp(-n^2/4)+n\exp(-n^2/5)\leq \frac1m.
\]
Here the last inequality uses $\pimin^*\leq1/m$, and hence
$n\geq3m\log m$, together with $m\geq3$.

\end{proof}

\subsection{Stage 1 Analysis}
\label{subsec:stage1}
We reset the iteration counter to $t=0$ at the start of Stage~1 and write
\[
T_1:=\left\lceil1000n^{10}d^{\epsilon-0.5}/\eta\right\rceil.
\]
When analyzing a fixed update from time $t$ to time $t+1$, an unprimed parameter denotes its
time-$t$ value and a primed parameter denotes its time-$(t+1)$ value. We restore superscripts when
more than one time is present in the same calculation.

Recall that
\[
\nabla_{\mu_i}\calL=\sum_j \pi_j^*\E_{x\sim j}\left[\psi_i(x)\sum_k \psi_k(x)(\mu_k-\mu_j^*)\right].
\]
Without loss of generality, after the relabeling in \cref{thm:init}, for each $1\leq i\leq m$ we
have $i=\arg\min_{j\in S_i}\|\mu_j-\parmu_j\|$.
For clarity about the time-dependent notation, let $u_i$ be the unit vector in the direction of the
initial $\botmu_i$. At later iterations we use
$\botmu_i:=\langle\mu_i,u_i\rangle u_i$ and let $\parmu_i$ denote the orthogonal projection of
$\mu_i$ onto $\calP$. The directions $u_i$ are fixed and pairwise orthogonal. The remaining component
$\mu_i-\parmu_i-\botmu_i$ is orthogonal to $\botmu_i$.

\begin{lemma}
\label{lm:stage1}
For every integer $0\leq t\leq T_1$, the following invariants hold at time $t$:

\begin{itemize}
\item For $1\leq i\leq m$, $\pi_i=\max_{j\in S_i}\pi_j$. Moreover, $\pi_i>\pi_i^*/n$.
\item For each $1\leq i\leq m$, and any $k\in S_i,k\neq i$, $\|\mu_k-\parmu_k\|\geq \|\mu_i-\parmu_i\|+\frac{1}{n^3}$.
\item For each $1\leq i\leq m,j\in S_i$, $\|\parmu_j-\mu_i^*\|\leq n+2n\eta t\leq 2n$.
\item For each $1\leq j\leq n$, $\sqrt{d}-2d^{2\epsilon}\leq \|\mu_j-\parmu_j\|\leq \sqrt{d}+n+2n\eta t\leq \sqrt{d}+2n$.
\item For each $1\leq i\leq n$, $\botmu_i\bot\calP$ satisfy that 
\begin{enumerate}
\item $(\mu_i-\botmu_i)\bot\botmu_i$.
\item $\|\mu_i-\botmu_i-\parmu_i\|\leq n+2n\eta t\leq 2n$.
\item $\botmu_i,1\leq i\leq n$ are pairwise orthogonal.
\end{enumerate}
\end{itemize}

Moreover, at $t=T_1$, the following stronger conclusions hold:
\begin{itemize}
\item For $1\leq i\leq m$, $\pi_i>\pi_i^*/n$.
\item For any $i\in [m],j\in S_i$ such that $i\neq j$, $\|\mu_i-\parmu_i\|\leq \|\mu_j-\parmu_j\|-d^{\epsilon}$;
\item For each $1\leq i\leq m,j\in S_i$, $\|\parmu_j-\mu_i^*\|\leq 2n$;
\item For each $1\leq j\leq n$, $\sqrt{d}-2d^{2\epsilon}\leq \|\mu_j-\parmu_j\|\leq \sqrt{d}+2n$.
\end{itemize}
\end{lemma}

Define the fixed subspace
\[
\calP_0:=\operatorname{span}\bigl(\calP\cup
\{\mu_1^{(0)},\ldots,\mu_n^{(0)}\}\bigr).
\]
Every vector appearing in the gradient formula is a linear combination of the current fitted means
and the true means. It follows inductively that $\mu_i^{(t)}\in\calP_0$ for every $i$ and every
Stage~1 iteration $t$. The projections $(\parmu_i)^{(t)}$ and $(\botmu_i)^{(t)}$ also belong to
$\calP_0$. The densities share the factor
$\exp(-\|x-P_0x\|^2/2)$, where $P_0$ is the orthogonal projector onto $\calP_0$.
This factor cancels in every responsibility ratio. Thus the responsibilities, although not the
individual full-dimensional densities, depend on $x$ only through $P_0x$.

Let $d_0:=\dim(\calP_0)\leq m+n\leq2n$, and let $\bx$ denote the projection of $x$ onto
$\calP_0$. Then, for every $i,t$, and $x$,
\begin{equation}
\label{eq:stage1:projection-responsibility}
\psi_i^{(t)}(x)=\psi_i^{(t)}(\bx).
\end{equation}
Notice that $\bx\in\calP_0$ but need not belong to $\calP$; its component orthogonal to $\calP$
will be controlled using $\|\bx-\mu_i^*\|$.

\begin{lemma}
\label{lm:barx}
Let $x\sim \calN(\mu_i^*,I_d)$. With probability at least $1-\exp\left(-d^{4\epsilon}/3\right)$, we have
\[
\|\bx-\mu_i^*\|\leq d^{2\epsilon}.
\]
\end{lemma}

\begin{proof}
Since $\mu_i^*\in\calP_0$, projecting
$x\sim\calN(\mu_i^*,I_d)$ onto $\calP_0$ gives a standard Gaussian on this $d_0$-dimensional
subspace, centered at $\mu_i^*$. Thus \cref{lm:gausstail} gives
\[
\Pr_{\bx\sim \calN(\mu_i^*,I_{d_0})}\left[\|\bx-\mu_i^*\|>d^{2\epsilon}\right]\leq \exp\left(-\left(d^{2\epsilon}-\sqrt{d_0}\right)^2/2\right)\leq\exp\left(-d^{4\epsilon}/3\right).
\]
For the final inequality we used $d_0\leq2n$ and the consequence
$\sqrt{2n}\leq(1-\sqrt{2/3})d^{2\epsilon}$ of \cref{eq:standing-dimension}.
\end{proof}

\begin{lemma}
\label{lm:stage1close}
Fix $0\leq t\leq T_1$ and suppose that the invariants in \cref{lm:stage1} hold at time $t$.
If $\|\bx-\mu_i^*\|\leq d^{2\epsilon}$, then for every $j\in S_i$ and $k\notin S_i$,
\[
\|\bx-\mu_k\|^2-\|\bx-\mu_j\|^2\geq d^{1-2\epsilon}/2.
\]
\end{lemma}

\begin{proof}
All fitted parameters in this proof are evaluated at the fixed time $t$. Let $P_{\parallel}$ and
$P_{\bot}$ be the orthogonal projections onto $\calP$ and $\calP^\bot$, and set
$\bx_{\parallel}:=P_{\parallel}\bx$ and $\bx_{\bot}:=P_{\bot}\bx$. Since
$\mu_i^*\in\calP$ and $\|\bx-\mu_i^*\|\leq d^{2\epsilon}$,
\begin{equation}
\label{eq:stage1close:x-components}
\|\bx_{\parallel}-\mu_i^*\|\leq d^{2\epsilon},
\qquad
\|\bx_{\bot}\|\leq d^{2\epsilon}.
\end{equation}

Because the sets $S_1,\ldots,S_m$ partition $[n]$, there is a unique $\ell\neq i$ such that
$k\in S_\ell$. The third invariant in \cref{lm:stage1} and the true-center separation give
\begin{align}
\|\bx_{\parallel}-\parmu_k\|
&\geq\|\mu_i^*-\mu_\ell^*\|
-\|\bx_{\parallel}-\mu_i^*\|-\|\parmu_k-\mu_\ell^*\| \notag\\
&\geq d^{0.5-\epsilon}-d^{2\epsilon}-2n,
\label{eq:stage1close:parallel-lower}
\end{align}
whereas $j\in S_i$ implies
\begin{equation}
\label{eq:stage1close:parallel-upper}
\|\bx_{\parallel}-\parmu_j\|
\leq\|\bx_{\parallel}-\mu_i^*\|+\|\parmu_j-\mu_i^*\|
\leq d^{2\epsilon}+2n.
\end{equation}

For the orthogonal components, \cref{eq:stage1close:x-components} and the fourth invariant in
\cref{lm:stage1} yield
\begin{align}
\|P_{\bot}(\bx-\mu_k)\|
&\geq\|\mu_k-\parmu_k\|-\|\bx_{\bot}\|
\geq\sqrt d-3d^{2\epsilon},
\label{eq:stage1close:orth-lower}\\
\|P_{\bot}(\bx-\mu_j)\|
&\leq\|\mu_j-\parmu_j\|+\|\bx_{\bot}\|
\leq\sqrt d+2n+d^{2\epsilon}.
\label{eq:stage1close:orth-upper}
\end{align}

The dimension condition implies $2n\leq d^{2\epsilon}$. Applying the Pythagorean identity to
$\calP\oplus\calP^\bot$ and then using
\cref{eq:stage1close:parallel-lower,eq:stage1close:parallel-upper,eq:stage1close:orth-lower,eq:stage1close:orth-upper},
we obtain
\begin{align*}
&\|\bx-\mu_k\|^2-\|\bx-\mu_j\|^2\\
&\geq \left(\sqrt{d}-3d^{2\epsilon}\right)^2
+\left(d^{0.5-\epsilon}-2d^{2\epsilon}\right)^2
-\left(\sqrt{d}+2d^{2\epsilon}\right)^2
-\left(2d^{2\epsilon}\right)^2\\
&=d^{1-2\epsilon}-10d^{0.5+2\epsilon}-4d^{0.5+\epsilon}+5d^{4\epsilon}\\
&\geq d^{1-2\epsilon}/2.
\end{align*}
For the final inequality, note that $0<\epsilon\leq0.01$, so the exponents
$0.5+2\epsilon$, $0.5+\epsilon$, and $4\epsilon$ are all strictly smaller than
$1-2\epsilon$; the required simultaneous comparison is guaranteed by
\cref{eq:standing-dimension}.
\end{proof}

\begin{lemma}
\label{lm:stage1:psisum}
Under assumptions in \cref{lm:stage1}, for any $1\leq i\leq m$, we have
\[
\sum_{j\in S_i}\E_{x\sim p^*}[\psi_j(x)]\geq \pi_i^*\left(1-2\exp\left(-d^{4\epsilon}/3\right)\right).
\]
\end{lemma}

\begin{proof}
Fix a time at which the invariants of \cref{lm:stage1} hold, and let
$G_i:=\{\|\bx-\mu_i^*\|\leq d^{2\epsilon}\}$. By
\cref{eq:stage1:projection-responsibility}, it is enough to work with $\bx$. On $G_i$,
\cref{lm:stage1close} implies, simultaneously for every $j\in S_i$ and $k\notin S_i$,
\[
e^{-\|\bx-\mu_k\|^2/2}
\leq e^{-d^{1-2\epsilon}/4}e^{-\|\bx-\mu_j\|^2/2}.
\]
Let $q_i:=\sum_{j\in S_i}\pi_j$. After multiplying the preceding inequality by the weights and
summing, the total unnormalized responsibility outside $S_i$ is at most
$e^{-d^{1-2\epsilon}/4}(1-q_i)/q_i$ times the total unnormalized responsibility inside $S_i$.
Consequently,
\begin{align*}
\sum_{j\in S_i}\psi_j(\bx)&=\frac{\sum_{j\in S_i}\pi_j\exp(-\|\bx-\mu_j\|^2/2)}{\sum_{k\in [n]} \pi_k\exp(-\|\bx-\mu_k\|^2/2)}\\
&\geq \frac{q_i}{q_i+e^{-d^{1-2\epsilon}/4}(1-q_i)}.
\end{align*}
The first invariant gives
\[
q_i\geq \pi_i\geq \pi_i^*/n.
\]
Using $1/(1+u)\geq1-u$ for $u\geq0$, we obtain
\begin{align*}
\sum_{j\in S_i}\psi_j(\bx)
&\geq \frac{\pi_i^*/n}{\pi_i^*/n+e^{-d^{1-2\epsilon}/4}}\\
&\geq 1-\frac{n}{\pi_i^*}\exp\left(-d^{1-2\epsilon}/4\right)\\
&\geq 1-\exp(-d^{0.5}).
\end{align*}
The last inequality follows from \cref{eq:standing-dimension}, which implies
$\log(n/\pimin^*)\leq d^{1-2\epsilon}/4-d^{0.5}$.

Under $x\sim i$, \cref{lm:barx} gives
$\Pr(G_i)\geq1-e^{-d^{4\epsilon}/3}$. Since all responsibilities are nonnegative, retaining only
the contribution of true component $i$ and the event $G_i$ yields
\begin{align*}
\sum_{j\in S_i}\E_{x\sim p^*}[\psi_j(x)]&\geq \pi_i^*\left(1-\exp\left(-d^{4\epsilon}/3\right)\right)\left(1-\exp(-d^{0.5})\right)\\
&\geq \pi_i^*\left(1-2\exp\left(-d^{4\epsilon}/3\right)\right).
\end{align*}
In the last step we used $e^{-d^{0.5}}\leq e^{-d^{4\epsilon}/3}$, which follows from
$4\epsilon<0.5$.
\end{proof}

\begin{lemma}
\label{lm:stage1:psifar}
Under assumptions in \cref{lm:stage1}, for any $1\leq i\leq m,j\not\in S_i$, we have
\[
\E_{x\sim i}[\psi_j(x)]\leq 2\exp\left(-d^{4\epsilon}/3\right).
\]
\end{lemma}

\begin{proof}
Fix the current iteration and set $G_i:=\{\|\bx-\mu_i^*\|\leq d^{2\epsilon}\}$.
Equation~\eqref{eq:stage1:projection-responsibility} gives $\psi_j(x)=\psi_j(\bx)$, and
\cref{lm:barx} gives $\Pr_{x\sim i}(G_i^c)\leq e^{-d^{4\epsilon}/3}$.
On $G_i$, compare fitted component $j\notin S_i$ with the representative $i\in S_i$ in the
denominator. By \cref{lm:stage1close} and the first invariant in \cref{lm:stage1},
\begin{align*}
\psi_j(x)&=\frac{\pi_j\exp(-\|\bx-\mu_j\|^2/2)}{\sum_{k\in[n]}\pi_k\exp(-\|\bx-\mu_k\|^2/2)}\\
&\leq \frac{\pi_j}{\pi_i}\exp\left(\left(\|\bx-\mu_i\|^2-\|\bx-\mu_j\|^2\right)/2\right)\\
&\leq \frac{n}{\pi_i^*}\exp(-d^{1-2\epsilon}/4)\\
&\leq \exp(-d^{0.4}).
\end{align*}
The last comparison is guaranteed by \cref{eq:standing-dimension}. Since $0\leq\psi_j\leq1$ on
$G_i^c$,
\[
\E_{x\sim i}[\psi_j(x)]\leq \exp(-d^{0.4})+\exp\left(-d^{4\epsilon}/3\right)\leq 2\exp\left(-d^{4\epsilon}/3\right).
\]
Here $0.4>4\epsilon$ because $\epsilon\leq0.01$.
\end{proof}

\begin{lemma}
\label{lm:far-weighted}
Assume \cref{eq:standing-dimension}, and suppose that all current means and true
means belong to the fixed subspace $\calP_0$, where $\dim(\calP_0)\leq 2n$.
For any $1\leq i\leq m$ such that
\[
\pi_i\geq \frac{\pi_i^*}{n},
\qquad
\|\mu_i-\mu_i^*\|\leq \frac{8}{7}\sqrt{d},
\]
and any $1\leq j\leq n$ such that
$s:=\|\mu_j-\mu_i^*\|\geq 8\sqrt{d}$, we have
\begin{align*}
\Pr_{x\sim i}\left[s\psi_j(x)\geq \exp(-s^2/10)\right]
&\leq \frac{1}{s}\exp(-s^2/20),\\
s\E_{x\sim i}[\psi_j(x)]
&\leq 2\exp(-s^2/20).
\end{align*}
\end{lemma}

\begin{proof}
Fix $i,j$ as in the statement, and let
$G_i:=\{\|\bx-\mu_i^*\|\leq s/3\}$.
Recall that $\bx$ is the projection of $x$ onto $\calP_0$ and
$d_0:=\dim(\calP_0)\leq 2n$.
By \cref{lm:gausstail} and \cref{eq:standing-dimension},
\[
\Pr_{x\sim i}[G_i^c]
\leq \exp\left(-\left(s/3-\sqrt{d_0}\right)^2/2\right)
\leq \frac{1}{s}\exp(-s^2/20).
\]

On $G_i$, since $\|\mu_i-\mu_i^*\|\leq (8/7)\sqrt{d}\leq s/7$,
the triangle inequality gives
\begin{align*}
\|\bx-\mu_j\|^2-\|\bx-\mu_i\|^2
&\geq \left(s-\|\bx-\mu_i^*\|\right)^2
-\left(\|\bx-\mu_i^*\|+\|\mu_i-\mu_i^*\|\right)^2\\
&\geq \left(\frac{2s}{3}\right)^2
-\left(\frac{s}{3}+\frac{s}{7}\right)^2\\
&=\frac{96}{441}s^2\\
&>\frac{s^2}{5}+2\log\frac{ns}{\pi_i^*}.
\end{align*}
The last inequality follows from \cref{eq:standing-dimension} and
$s\geq 8\sqrt{d}$.

Using $\psi_j(x)=\psi_j(\bx)$ and comparing with component $i$ in the
denominator, we obtain
\begin{align*}
\psi_j(x)
&\leq \frac{\pi_j}{\pi_i}
\exp\left(\left(\|\bx-\mu_i\|^2-\|\bx-\mu_j\|^2\right)/2\right)\\
&<\frac{n}{\pi_i^*}
\exp\left(-\frac{s^2}{10}-\log\frac{ns}{\pi_i^*}\right)\\
&=\frac{1}{s}\exp(-s^2/10).
\end{align*}
Therefore,
\[
\Pr_{x\sim i}\left[s\psi_j(x)\geq \exp(-s^2/10)\right]
\leq \Pr_{x\sim i}[G_i^c]
\leq \frac{1}{s}\exp(-s^2/20).
\]
Finally, splitting the expectation over $G_i$ and $G_i^c$ and using
$0\leq \psi_j(x)\leq 1$, we have
\begin{align*}
s\E_{x\sim i}[\psi_j(x)]
&\leq \exp(-s^2/10)+s\Pr_{x\sim i}[G_i^c]\\
&\leq \exp(-s^2/10)+\exp(-s^2/20)\\
&\leq 2\exp(-s^2/20).
\end{align*}
\end{proof}

Under the Stage~1 invariants, this lemma and \cref{lm:stage1:psifar} imply the uniform bound
\begin{equation}
\label{eq:stage1:weighted-cross}
\|\mu_k-\mu_i^*\|\E_{x\sim i}[\psi_k(x)]
\leq16\sqrt d\,e^{-d^{4\epsilon}/3}
\qquad(k\notin S_i).
\end{equation}
Indeed, if the distance is at most $8\sqrt d$, use \cref{lm:stage1:psifar}; otherwise use
\cref{lm:far-weighted}, with reference $i$. Its reference-radius condition follows from
$\|\mu_i-\mu_i^*\|\leq\sqrt d+4n\leq(8/7)\sqrt d$.

\begin{lemma}
\label{lem:stage1-2variable}
Let $a_1,a_2\geq\sqrt{d/2}$ such that $a_1+\frac{1}{8n^3}\leq a_2$.
Let $X_1,X_2$ be two independent random variables and $X_1,X_2\sim \calN(0,1)$.
For $i=1,2$, let $Y_i=(a_i-X_i)^2$.
Then  we have
\begin{enumerate}
\item 
\[
\Pr[Y_1+d^{5\epsilon}\leq Y_2]\geq \Pr[Y_2-d^{5\epsilon}\leq Y_1]+\frac{1}{1000n^3}.
\]
\item For any positive real number $M>d/10$,
\[
\Pr[Y_1+d^{5\epsilon}\leq M,Y_2]\geq \Pr[Y_2-d^{5\epsilon}\leq M,Y_1].
\]
\item For any positive integer $3\leq s\leq n$, when $M>(a_1-\sqrt{2\log s})^2$, 
\[
\Pr[Y_1+d^{5\epsilon}\leq M,Y_2]\geq \Pr[Y_2-d^{5\epsilon}\leq M,Y_1]+\frac{1}{100n^3s^6}.
\]

\end{enumerate}
\end{lemma}

\begin{proof}
Let $D:=d^{5\epsilon}$. Define
$g(x)=(2\pi)^{-1/2}e^{-x^2/2}$ and $G(x)=\int_{-\infty}^{x}g(t)\,dt$, the standard Gaussian
density and distribution function. We prove the three claims in order.

\begin{enumerate}
\item First consider
\[
p=\Pr\left[X_2\leq X_1+\frac{1}{10n^3}\right].
\]
Then
\[
p=\int_{-\infty}^{+\infty}g(x)G\left(x+\frac{1}{10n^3}\right)\,dx.
\]
When $|x|\leq 1$, $g(x)\geq 0.1$. Thus for $-1\leq x\leq 0$,
\[
G\left(x+\frac{1}{10n^3}\right)-G(x)=\int_{x}^{x+\frac{1}{10n^3}}g(y)\,dy\geq \frac{1}{100n^3}.
\]
Therefore,
\begin{align*}
p&\geq \int_{-1}^{0}g(x)\left(G\left(x+\frac{1}{10n^3}\right)-G(x)\right)\,dx+\int_{-\infty}^{+\infty}g(x)G(x)\,dx\\
&\geq \frac{1}{1000n^3}+\Pr[X_2\leq X_1]=\frac{1}{1000n^3}+\frac{1}{2}.
\end{align*}

By a union bound and the one-dimensional Gaussian tail estimate, the standing dimension condition
implies
\[
\Pr[\max\{|X_1|,|X_2|\}\geq d^{\epsilon/3}]
\leq4e^{-d^{2\epsilon/3}/2}\leq \frac{1}{2000n^3}.
\]
Therefore, the following two conditions hold simultaneously with probability at least
$1/2+1/(2000n^3)$:
\begin{itemize}
\item $X_1,X_2<d^{\epsilon/3}$;
\item $X_2\leq X_1+\frac{1}{10n^3}$,
\end{itemize}
On this event,
$a_2-X_2\geq a_1-X_1+1/(8n^3)-1/(10n^3)=a_1-X_1+1/(40n^3)$, and hence
\begin{align*}
Y_2-Y_1&=(a_2-X_2)^2-(a_1-X_1)^2\\
&\geq (a_1-X_1+1/(40n^3))^2-(a_1-X_1)^2\\
&\geq \frac{1}{20n^3}(a_1-X_1)\\
&\geq \frac{\sqrt{d}}{40n^3}>d^{5\epsilon}.
\end{align*}
The penultimate inequality uses
$a_1-X_1\geq\sqrt{d/2}-d^{\epsilon/3}\geq\sqrt d/2$, and the final inequality is a consequence
of \cref{eq:standing-dimension}.

Thus
\[
\Pr[Y_1+D\leq Y_2]\geq \frac12+\frac{1}{2000n^3}.
\]
Apart from a probability-zero equality event, $\{Y_2-D\leq Y_1\}$ is the complement of
$\{Y_1+D<Y_2\}$. Therefore
\[
\Pr[Y_2-D\leq Y_1]\leq \frac12-\frac{1}{2000n^3},
\]
and subtracting proves item~1.

\item Fix $M>d/10$. For readability, write
\begin{align*}
A&:=a_1-\sqrt{M-D},\qquad B:=a_2-\sqrt{M+D},\\
p_1&:=\Pr\left[X_1\geq A,\ X_2\leq X_1+\frac{1}{10n^3}\right],\\
p_2&:=\Pr\left[X_2\geq B,\ X_2\geq X_1+\frac{1}{10n^3}\right].
\end{align*}
The standing dimension condition ensures $D<M$ and makes both square roots well defined. If
$A\leq X_1\leq a_1-\sqrt d/5$ and $X_2\leq X_1+1/(10n^3)$, then the calculation from item~1 gives
\[
Y_2\geq \left(a_1-X_1+\frac{1}{40n^3}\right)^2\geq Y_1+D.
\]
The condition $X_1\geq A$ is exactly $Y_1\leq M-D$ on the relevant side of the Gaussian tail,
and therefore $Y_1+D\leq M$. We have proved that the event defining $p_1$, except possibly for
$X_1>a_1-\sqrt d/5$, is contained in $\{Y_1+D\leq M,\,Y_2\}$. Hence
\begin{align}
\Pr[Y_1+D\leq M,Y_2]&\geq p_1-\Pr\left[X_1>a_1-\sqrt{d}/5\right] \nonumber\\
&\geq p_1-\exp\left(-\left(a_1-\frac{\sqrt{d}}{5}-1\right)^2/2\right). \label{eq:p1}
\end{align}
The final line is a deliberately loose form of the Gaussian tail bound; the extra $-1$ avoids
tracking a prefactor.

Conversely, suppose $Y_2-D\leq M$ and $Y_2-D\leq Y_1$. The first inequality implies
$X_2\geq B$ on the relevant side. If also
$X_2-1/(10n^3)\leq X_1\leq a_1-\sqrt d/5$, then
\[
Y_2-Y_1\geq \left(a_1-X_1+\frac{1}{40n^3}\right)^2- (a_1-X_1)^2>D,
\]
contradicting $Y_2-D\leq Y_1$. Thus the target event is contained in the union of the event
defining $p_2$ and $\{X_1>a_1-\sqrt d/5\}$. Consequently,
\begin{align}
\Pr[Y_2-D\leq M,Y_1]&\leq p_2+\Pr\left[X_1>a_1-\sqrt{d}/5\right] \nonumber\\
&\leq p_2+\exp\left(-\left(a_1-\frac{\sqrt{d}}{5}-1\right)^2/2\right). \label{eq:p2}
\end{align}

We next compare the two integration thresholds. Since
\[
\sqrt{M+D}-\sqrt{M-D}=\frac{2D}{\sqrt{M+D}+\sqrt{M-D}}
\leq d^{-5\epsilon}
\]
under \cref{eq:standing-dimension}, and since
$a_2-a_1\geq1/(8n^3)$ and $d^{-5\epsilon}\leq1/(40n^3)$, we have
\[
A\leq B-\frac{1}{10n^3}.
\]

We prove the second item in this lemma by considering the following two cases (as in item 3, we let $3\leq s\leq n$ to be a positive integer):

\textbf{Case 1: $M\leq (a_1-\sqrt{2\log s})^2$.}

The definitions of $p_1,p_2$ and independence give
$p_1=\int_A^\infty g(x)G(x+1/(10n^3))\,dx$ and
$p_2=\int_B^\infty g(x)G(x-1/(10n^3))\,dx$. The threshold ordering $A\leq B-1/(10n^3)$
leaves an interval of length at least $1/(10n^3)$ that contributes only to $p_1$. Moreover, in
this case $A\geq\sqrt{2\log s}>0$, so $G(x+1/(10n^3))\geq1/2$ on that interval. Therefore,
Put $h=1/(10n^3)$. Then
\begin{align*}
p_1&=\int_A^\infty g(x)G(x+h)\,dx\\
&\geq\int_B^\infty g(x)G(x-h)\,dx
+\frac12\int_A^{A+h}g(x)\,dx\\
&\geq p_2+\frac{h}{2\sqrt{2\pi}}e^{-(A+h)^2/2}\\
&\geq p_2+2\exp\!\left[-\frac12\left(a_1-\frac{\sqrt d}{5}-1\right)^2\right].
\end{align*}
The last inequality uses $M>d/10$ and the standing dimension condition; it says that the mass of
the extra interval dominates twice the tail error appearing in
\cref{eq:p1,eq:p2}. Combining \cref{eq:p1}, \cref{eq:p2} and the interval bound gives
\begin{align*}
\Pr[Y_1+d^{5\epsilon}\leq M,Y_2]&\geq p_1-\exp\left(-\left(a_1-\frac{\sqrt{d}}{5}-1\right)^2/2\right)\\
&\geq p_2+\exp\left(-\left(a_1-\frac{\sqrt{d}}{5}-1\right)^2/2\right)\\
&\geq\Pr[Y_2-d^{5\epsilon}\leq M,Y_1].
\end{align*}

\textbf{Case 2: $M>(a_1-\sqrt{2\log s})^2$.}

In this case, the difference between the two Gaussian cdf factors supplies a quantitative gain.
Indeed, integrating
$G(x+1/(10n^3))-G(x-1/(10n^3))$ over
$[\sqrt{2\log s}+0.01,\sqrt{2\log s}+1.01]$ gives
\begin{align*}
p_1&=\int_{a_1-\sqrt{M-d^{5\epsilon}}}^{+\infty}g(x)G\left(x+\frac{1}{10n^3}\right)\,dx\\
&\geq \int_{a_1-\sqrt{M-d^{5\epsilon}}}^{+\infty}g(x)G\left(x-\frac{1}{10n^3}\right)\,dx+\frac{1}{10n^3}\int_{\sqrt{2\log s}+0.01}^{\sqrt{2\log s}+1.01}g(x)^2\,dx\\
&\geq \int_{a_1-\sqrt{M-d^{5\epsilon}}}^{+\infty}g(x)G\left(x-\frac{1}{10n^3}\right)\,dx+\frac{1}{10n^3}\cdot\frac{1}{2\pi}\exp\left(-\left(\sqrt{2\log s}+1.01\right)^2\right)\\
&\geq \int_{a_1-\sqrt{M-d^{5\epsilon}}}^{+\infty}g(x)G\left(x-\frac{1}{10n^3}\right)\,dx+\frac{1}{20\pi n^3}\cdot\frac{1}{s^6}\\
&\geq p_2+\frac{1}{20\pi n^3s^6}.
\end{align*}
Here $g(x)^2=(2\pi)^{-1}e^{-x^2}$ and
$(\sqrt{2\log s}+1.01)^2\leq6\log s$ for $s\geq3$.
The integration interval lies above $A$, since $A\leq\sqrt{2\log s}+d^{-5\epsilon}<\sqrt{2\log s}+0.01$.
On this interval $g$ is convex, so
$G(x+h)-G(x-h)\geq2h g(x)\geq h g(x)$ for $h=1/(10n^3)$.
Each of the two tail corrections in \cref{eq:p1,eq:p2} is at most
$1/(1000n^3s^6)$ under \cref{eq:standing-dimension}; consequently the net gain is at least
$(1/(20\pi)-2/1000)/(n^3s^6)>1/(100n^3s^6)$. Hence
\begin{align*}
\Pr[Y_1+d^{5\epsilon}\leq M,Y_2]&\geq p_1-\exp\left(-\left(a_1-\frac{\sqrt{d}}{5}-1\right)^2/2\right)\\
&\geq \left(p_2+\exp\left(-\left(a_1-\frac{\sqrt{d}}{5}-1\right)^2/2\right)\right)+\frac{1}{100n^3s^6}\\
&\geq\Pr[Y_2-d^{5\epsilon}\leq M,Y_1]+\frac{1}{100n^3s^6}.
\end{align*}

Cases~1 and~2 together prove item~2 for every $M>d/10$.

\item Under the additional hypothesis in item~3, only Case~2 applies. The last display in that case
is exactly the claimed improvement by $1/(100n^3s^6)$.

\end{enumerate}
\end{proof}

\begin{lemma}
\label{lm:stage1:orthoprob}
Fix a positive integer $2\leq s\leq n$. 
Let $a_1,a_2,\cdots,a_s\geq \sqrt{d/2}$ and for $2\leq j\leq s$, $a_1+\frac{1}{8n^3}\leq a_j$.
Also, let $\delta_1,\cdots,\delta_s$ be real numbers such that $\delta_1\leq \delta_2,\cdots,\delta_s$.
Consider $s$ independent Gaussian variables $X_1,\cdots,X_s\sim\calN(0,1)$.
We define
\[
P_1=\Pr\left[(a_1-X_1)^2+\delta_1+d^{5\epsilon}\leq (a_j-X_j)^2+\delta_j, 2\leq j\leq s\right]
\]
For any $2\leq i\leq s$, we define
\[
P_i=\Pr\left[(a_i-X_i)^2+\delta_i-d^{5\epsilon}\leq (a_j-X_j)^2+\delta_j, 1\leq j\leq s,j\neq i\right]
\]
Then $P_1\geq P_i+\frac{1}{300n^3s^6}$.
% \[
% \Pr[X_1+a_1-\delta>X_j+a_j,2\leq j\leq s]\geq \Pr[X_i+a_i+\delta>X_j+a_j,1\leq j\leq s,j\neq i]+\delta\cdot\Omega(1/s^3).
% \]
\end{lemma}

\begin{proof}
The indices $2,\ldots,s$ are symmetric, so it is enough to prove the claim for $i=2$. Holding all
other parameters fixed, $P_1$ is nondecreasing in $\delta_2$, while $P_2$ is nonincreasing in
$\delta_2$. Hence the smallest possible value of $P_1-P_2$ occurs at $\delta_2=\delta_1$, and we
may assume equality. If $s=2$, item~1 of \cref{lem:stage1-2variable} gives the stronger gap
$1/(1000n^3)$. We therefore assume $s\geq3$.

In this case, let $M$ be the random variable
\[
M=\min_{3\leq j\leq s}\left\{(a_j-X_j)^2+\delta_j-\delta_1\right\}.
\]
Set $Y_1=(a_1-X_1)^2$ and $Y_2=(a_2-X_2)^2$. Since $M$ depends only on
$X_3,\ldots,X_s$, it is independent of $(Y_1,Y_2)$. Conditional on $M$, the two events become
\[
P_1=\Pr[Y_1+d^{5\epsilon}\leq M,Y_2],\qquad
P_2=\Pr[Y_2-d^{5\epsilon}\leq M,Y_1].
\]
Let $P_1(M),P_2(M)$ denote the corresponding conditional probabilities over $(Y_1,Y_2)$.
Item~2 of \cref{lem:stage1-2variable} gives $P_1(M)\geq P_2(M)$ whenever $M>d/10$.

If $M\leq d/10$, then, because $\delta_j-\delta_1\geq0$, some $j\in\{3,\ldots,s\}$ satisfies
$|a_j-X_j|\leq\sqrt{d/10}$. Since $a_j\geq\sqrt{d/2}$, this implies
$X_j\geq\sqrt{d/2}-\sqrt{d/10}\geq\sqrt d/3$. Thus the Gaussian tail bound and a union bound give
\[
\Pr[M\leq d/10]\leq \sum_{j=3}^s\Pr[X_j\geq \sqrt{d}/3]\leq (s-2)\exp(-d/20).
\]

Next consider the region in which item~3 of \cref{lem:stage1-2variable} gives a strict gain. If
$X_j\leq\sqrt{2\log s}$ for every $3\leq j\leq s$, then $a_j\geq a_1$ and
$\delta_j\geq\delta_1$ imply
$M\geq(a_1-\sqrt{2\log s})^2$. Moreover, \cref{lm:mill} gives
\[
\Pr[X_j>\sqrt{2\log s}]\leq
\frac{1}{s\sqrt{4\pi\log s}}.
\]
and a union bound over $3\leq j\leq s$ shows that these inequalities hold simultaneously with
probability at least $1/2$. Therefore,
\[
\Pr\left[M>(a_1-\sqrt{2\log s})^2\right]\geq \Pr[X_j\leq \sqrt{2\log s},3\leq j\leq s]\geq \frac{1}{2}.
\]

Finally, average the conditional difference over $M$. On $M\leq d/10$ use the trivial lower bound
$P_1(M)-P_2(M)\geq-1$; on $M>d/10$ use item~2 of
\cref{lem:stage1-2variable}; and on the favorable subevent use the additional gain from item~3.
This gives
\begin{align*}
P_1-P_2&=\E_M\left[P_1(M)-P_2(M)\right]\\
&\geq -\Pr[M\leq d/10]+\frac{1}{100n^3s^6}\Pr[M>(a_1-\sqrt{2\log s})^2]\\
&\geq \frac{1}{200n^3s^6}-(s-2)\exp(-d/20)\\
&\geq \frac{1}{300n^3s^6}.
\end{align*}
The final comparison follows from \cref{eq:standing-dimension}.
\end{proof}

\begin{lemma}
\label{lm:stage1:psiclose}
Under assumptions in \cref{lm:stage1}, for any $1\leq i\leq m,j\in S_i$ such that $j\neq i$, we have
\[
\Pr_{x\sim i}\left[\psi_i(x)\geq1-\exp\left(-d^{\epsilon}\right)\right]\geq \Pr_{x\sim i}[\psi_j(x)>\exp(-d^\epsilon)]+\frac{1}{400n^9}.
\]
\end{lemma}

\begin{proof}
Fix the current Stage~1 iteration and evaluate all fitted parameters at this time. Let
$G_i:=\{\|\bx-\mu_i^*\|\leq d^{2\epsilon}\}$. For each fitted index $l$, let $u_l$ be the fixed
unit direction introduced before \cref{lm:stage1}, set
$y_l:=\langle\bx,u_l\rangle$, $\bx_l:=y_lu_l$, and $a_l:=\|\botmu_l\|$. We orient $u_l$ so that
$\botmu_l=a_lu_l$. Since $(\mu_l-\botmu_l)\perp u_l$, Pythagoras gives
\begin{equation}
\label{eq:stage1:one-dimensional-decomposition}
\|\bx-\mu_l\|^2=(y_l-a_l)^2+
\|\bx-\bx_l-(\mu_l-\botmu_l)\|^2.
\end{equation}

For $l\in S_i$, the third invariant and item~5(b) of \cref{lm:stage1} give
$\|\mu_l-\botmu_l-\mu_i^*\|\leq4n$. On $G_i$, orthogonality of $u_l$ and $\mu_i^*$ gives
$|y_l|\leq d^{2\epsilon}$ and
$\|\bx-\bx_l-\mu_i^*\|\leq2d^{2\epsilon}$. Since
$4n\leq d^{2\epsilon}$ under \cref{eq:standing-dimension},
\[
\|\bx-\bx_l-(\mu_l-\botmu_l)\|\leq3d^{2\epsilon}.
\]
Substituting this estimate into \cref{eq:stage1:one-dimensional-decomposition} yields
\begin{equation}
\label{eq:stage1:energy-one-dimensional}
\left|\|\bx-\mu_l\|^2-(a_l-y_l)^2\right|\leq9d^{4\epsilon}
\qquad(l\in S_i).
\end{equation}
Because the $u_l$'s are pairwise orthogonal and orthogonal to $\calP$, the variables
$\{y_l:l\in S_i\}$ are independent $\calN(0,1)$ variables under $x\sim i$.

We next verify the hypotheses of \cref{lm:stage1:orthoprob}. Put
$r_l:=\|\mu_l-\parmu_l\|$ and
$b_l:=\|\mu_l-\parmu_l-\botmu_l\|$. Then $a_l^2=r_l^2-b_l^2$. For
$k\in S_i\setminus\{i\}$, the second invariant and item~5(b) imply
\begin{align*}
a_k^2-a_i^2
&=(r_k-r_i)(r_k+r_i)-b_k^2+b_i^2\\
&\geq \frac{r_i}{n^3}-4n^2
\geq\frac{\sqrt d}{2n^3}.
\end{align*}
The final inequality uses $r_i\geq\sqrt d-2d^{2\epsilon}$ and the standing dimension condition.
The Stage~1 orthogonal-radius bound directly gives
$a_l\leq r_l\leq\sqrt d+2n\leq2\sqrt d$, independently of the true-center diameter. Hence
\[
a_k-a_i=\frac{a_k^2-a_i^2}{a_k+a_i}\geq\frac{1}{8n^3}.
\]
Likewise,
$a_l^2=r_l^2-b_l^2\geq(\sqrt d-2d^{2\epsilon})^2-(2n)^2\geq d/2$, so
$a_l\geq\sqrt{d/2}$.

Define the offsets $\delta_l:=-2\log\pi_l$. The first invariant in \cref{lm:stage1} gives
$\delta_i\leq\delta_l$ for $l\in S_i$. Let $P_1$ be the probability of
\begin{equation}
\label{eq:stage1:P1-event}
(a_i-y_i)^2+\delta_i+d^{5\epsilon}
\leq(a_l-y_l)^2+\delta_l
\qquad\text{for every }l\in S_i\setminus\{i\}.
\end{equation}
On $G_i$, \cref{eq:stage1:energy-one-dimensional} and
$18d^{4\epsilon}\leq d^{5\epsilon}/2$ imply that the unnormalized responsibility of component $i$
exceeds that of every other component in $S_i$ by a factor at least $e^{d^{5\epsilon}/4}$.
For $k\notin S_i$, \cref{lm:stage1close} and $\pi_i\geq\pi_i^*/n$ give the stronger comparison
\[
\frac{\pi_i e^{-\|\bx-\mu_i\|^2/2}}
{\pi_k e^{-\|\bx-\mu_k\|^2/2}}
\geq \frac{\pi_i^*}{n}e^{d^{1-2\epsilon}/4}
\geq e^{d^{1-2\epsilon}/6}.
\]
Therefore, on $G_i$ and the event in \cref{eq:stage1:P1-event},
\[
\psi_i(x)\geq\frac{e^{d^{5\epsilon}/4}}{e^{d^{5\epsilon}/4}+n}
\geq1-e^{-d^\epsilon}.
\]
Since \cref{lm:barx} gives $\Pr(G_i^c)\leq e^{-d^{4\epsilon}/3}$,
\begin{equation}
\label{eq:stage1:P1-to-responsibility}
\Pr_{x\sim i}[\psi_i(x)\geq1-e^{-d^\epsilon}]
\geq P_1-e^{-d^{4\epsilon}/3}.
\end{equation}

For the fixed competitor $j\in S_i\setminus\{i\}$, let $P_2$ be the probability of
\begin{equation}
\label{eq:stage1:P2-event}
(a_j-y_j)^2+\delta_j-d^{5\epsilon}
\leq(a_l-y_l)^2+\delta_l
\qquad\text{for every }l\in S_i\setminus\{j\}.
\end{equation}
If this event fails, then some $l\in S_i$ beats component $j$ by an energy margin
$d^{5\epsilon}$. On $G_i$, \cref{eq:stage1:energy-one-dimensional} then gives
$\psi_j(x)\leq\exp[-(d^{5\epsilon}-9d^{4\epsilon})/2]
\leq e^{-d^{5\epsilon}/4}<e^{-d^\epsilon}$. Consequently,
\begin{equation}
\label{eq:stage1:responsibility-to-P2}
\Pr_{x\sim i}[\psi_j(x)>e^{-d^\epsilon}]
\leq P_2+e^{-d^{4\epsilon}/3}.
\end{equation}

Apply \cref{lm:stage1:orthoprob} with $s=|S_i|$, distinguished index $i$, competitor $j$, radii
$a_l$, and offsets $\delta_l$. Since $s\leq n$,
\[
P_1-P_2\geq\frac{1}{300n^3s^6}\geq\frac{1}{300n^9}.
\]
Combining this inequality with
\cref{eq:stage1:P1-to-responsibility,eq:stage1:responsibility-to-P2}, and using
$2e^{-d^{4\epsilon}/3}\leq1/(1200n^9)$, proves
\[
\Pr_{x\sim i}[\psi_i(x)\geq1-e^{-d^\epsilon}]
\geq\Pr_{x\sim i}[\psi_j(x)>e^{-d^\epsilon}]+\frac{1}{400n^9}.
\]
\end{proof}

\begin{lemma}
\label{lm:stage1:nocorrolation}
For any $i,j\in [n],i\neq j$, let $r=\|\mu_i-\mu_j\|$. Then for any $k\in [m]$,
$\E_{x\sim k}[\psi_i(x)\psi_j(x)]\leq \frac{1}{1.1r}$ when $r>0$.
If $r=0$, the bound $\psi_i\psi_j\leq1/4$ applies instead. Moreover, under the assumptions in
\cref{lm:stage1}, we have
$\E_{x\sim k}[\psi_i(x)\psi_j(x)]\leq \frac{1}{1.1\sqrt{d}}$.
\end{lemma}

\begin{proof}
Fix the current iteration and assume $r>0$. Retaining only the two terms indexed by $i$ and $j$ in the responsibility
denominators gives
\[
\Delta(x) := (\|x-\mu_i\|^2 - \|x-\mu_j\|^2)/2
= (\mu_j-\mu_i)^\top x + \frac{\|\mu_i\|^2 - \|\mu_j\|^2}{2}.
\]
Let $a := \pi_j/\pi_i$. Then
\[
\psi_i(x) \le \frac{1}{1 + a e^{\Delta(x)}}, \qquad
\psi_j(x) \le \frac{1}{1 + a^{-1} e^{-\Delta(x)}}.
\]
Thus
\[
\psi_i(x)\psi_j(x)
\le \frac{1}{(1 + a e^{\Delta(x)})(1 + a^{-1} e^{-\Delta(x)})}
\leq \frac{1}{e^{|\Delta(x)+\log a|}}.
\]

Let
$\zeta := \frac{\|\mu_i\|^2 - \|\mu_j\|^2}{2} + \log a+(\mu_j-\mu_i)^\top\mu_k^*$
and $\sigma := r$. Under $x\sim\calN(\mu_k^*,I_d)$,
\[
Z := \Delta(x) + \log a
= (\mu_j-\mu_i)^\top x + \left(\frac{\|\mu_i\|^2 - \|\mu_j\|^2}{2} + \log a\right)
\sim \mathcal N(\zeta,\sigma^2).
\]
Thus
\[
\E[\psi_i(X)\psi_j(X)]
\le \E[e^{-|Z|}].
\]

Let $\varphi_{\zeta,\sigma}$ be the density of $\mathcal N(\zeta,\sigma^2)$. Then
\[
\E[e^{-|Z|}]
= \int_{\R} e^{-|z|}\varphi_{\zeta,\sigma}(z)\,dz
= \int_0^\infty e^{-t}\big(\varphi_{\zeta,\sigma}(t) + \varphi_{\zeta,\sigma}(-t)\big)\,dt.
\]
Since $\varphi_{\zeta,\sigma}(z)\le (\sqrt{2\pi}\sigma)^{-1}$ for every $z$,
\[
\E[e^{-|Z|}]
\le \frac{2}{\sqrt{2\pi}\sigma}\int_0^\infty e^{-t}\,dt
= \frac{2}{\sqrt{2\pi}\sigma}
= \frac{2}{\sqrt{2\pi}r}.
\]
Hence
\[
\E[\psi_i(X)\psi_j(X)]
\le \frac{2}{\sqrt{2\pi}r}
\leq \frac{1}{1.1r}.
\]
It remains to derive the uniform bound $(1.1\sqrt d)^{-1}$ under the Stage~1 invariants. If $i$ and
$j$ belong to different groups, then for every true source $k$, at least one of $i,j$ is outside
$S_k$. Hence \cref{lm:stage1:psifar} and $\psi_i\psi_j\leq\min\{\psi_i,\psi_j\}$ give
\[
\E_{x\sim k}[\psi_i(x)\psi_j(x)]\leq2e^{-d^{4\epsilon}/3}
\leq\frac{1}{1.1\sqrt d}.
\]
If $i$ and $j$ are in the same group, item~5(c) gives $\botmu_i\perp\botmu_j$, while the other
components are bounded by the third and fifth invariants. Therefore
\begin{align*}
\|\mu_i-\mu_j\|&\geq \|\botmu_i-\botmu_j\|-\|\mu_i-\botmu_i-\parmu_i\|-\|\mu_j-\botmu_j-\parmu_j\|-\|\parmu_i-\parmu_j\|\\
&\geq \sqrt{2}\cdot (\sqrt{d}-3d^{2\epsilon})-8n\geq \sqrt{d}.
\end{align*}
Applying the already-proved $1/(1.1r)$ bound completes this case. All final comparisons follow from
\cref{eq:standing-dimension}.

\end{proof}

Now we are ready to prove \cref{lm:stage1}.

\begin{proof}[Proof of \cref{lm:stage1}]
We prove all invariants simultaneously by induction on $t$. At $t=0$, the geometric statements are
exactly the conclusions of \cref{thm:init}. The algorithm initializes every fitted weight to $1/n$,
so the first invariant also holds. Assume the invariants hold at some $t<T_1$. Throughout the proof,
unprimed quantities are evaluated at time $t$ and primed quantities at time $t+1$. We first propagate the invariants and then record the strengthened terminal gap.

\paragraph{Gaps inside one group.}
Fix a group $S_q$ ($q\in [m]$) and a competitor $k\in S_q\setminus\{q\}$. Put
$\alpha:=e^{-d^\epsilon}$. By \cref{lm:stage1:psiclose},
\[
\Pr_{x\sim q}[\psi_q(x)\geq1-\alpha]
\geq\Pr_{x\sim q}[\psi_k(x)>\alpha]+\frac{1}{400n^9}.
\]
Since $0\leq\psi_l\leq1$, this probability inequality implies both a first-moment and a
second-moment gap. Indeed,
\begin{align}
\E_{x\sim q}[\psi_q(x)]
&\geq(1-\alpha)\Pr_{x\sim q}[\psi_q(x)\geq1-\alpha] \notag\\
&\geq\Pr_{x\sim q}[\psi_k(x)>\alpha]+\frac{1}{450n^9} \notag\\
&\geq\E_{x\sim q}[\psi_k(x)]+\frac{1}{500n^9},
\label{eq:stage1:first-moment-gap}\\
\E_{x\sim q}[\psi_q(x)^2]
&\geq(1-\alpha)^2\Pr_{x\sim q}[\psi_q(x)\geq1-\alpha] \notag\\
&\geq\E_{x\sim q}[\psi_k(x)^2]+\frac{1}{500n^9}.
\label{eq:stage1:second-moment-gap}
\end{align}
For the last line of each estimate, split the expectation of $\psi_k$ or $\psi_k^2$ according to
$\{\psi_k>\alpha\}$ and use
$\alpha\leq1/(10^4n^9)$. This comparison follows from \cref{eq:standing-dimension}.

\paragraph{Proof of the first item.}
The exact EM weight update is
$\pi_l'=\E_{x\sim p^*}[\psi_l(x)]$. For $l\neq q$ where $l\in [m]$, both $q$ and $k$ are outside $S_l$; hence
\cref{lm:stage1:psifar} bounds the magnitude of their total contribution from true groups other than
$q$ by $4e^{-d^{4\epsilon}/3}$. Combining this with
\cref{eq:stage1:first-moment-gap} gives
\begin{equation}
\label{eq:stage1:global-weight-gap}
\pi_q'-\pi_k'
\geq\frac{\pi_q^*}{500n^9}-4e^{-d^{4\epsilon}/3}
\geq\frac{\pi_q^*}{600n^9}>0.
\end{equation}
Thus the representative $\pi_q$ remains the unique largest-weight member of $S_q$ after the update.

It remains to show $\pi_q'>\pi_q^*/n$. Let $s:=|S_q|$. By
\cref{lm:stage1:psisum}, the updated total weight of $S_q$ is at least
$\pi_q^*(1-2e^{-d^{4\epsilon}/3})$. If $s=1$, this already proves the claim. If $s\geq2$, then
\cref{eq:stage1:global-weight-gap} and elementary averaging imply
\begin{align*}
\pi_q'
&\geq\frac{1}{s}\sum_{l\in S_q}\pi_l'
+\frac{s-1}{s}\frac{\pi_q^*}{600n^9}\\
&\geq\frac{\pi_q^*}{n}\left(1-2e^{-d^{4\epsilon}/3}\right)
+\frac{\pi_q^*}{1200n^9}
>\frac{\pi_q^*}{n}.
\end{align*}
This proves the first invariant at time $t+1$.

\paragraph{A one-step gradient decomposition.}
Fix a fitted index $r\in S_q$ and define
\[
A_r:=\pi_q^*\E_{x\sim q}[\psi_r(x)^2],
\qquad
e_r:=\nabla_{\mu_r}\calL-A_r(\mu_r-\mu_q^*).
\]
For $l\in S_h$, the induction hypotheses give the local radius bound
\begin{equation}
\label{eq:stage1:uniform-diameter}
\|\mu_l-\mu_h^*\|\leq\sqrt d+4n\leq2\sqrt d.
\end{equation}
Put $\beta=e^{-d^{4\epsilon}/3}$. For source $q$, the same-group terms with $l\neq r$
are bounded using \cref{lm:stage1:nocorrolation}; the other fitted indices are bounded using
\cref{eq:stage1:weighted-cross}. For each other source $h\neq q$, use
$\sum_{l\in S_h}\psi_l\leq1$ and $\E_h\psi_r\leq2\beta$ for its local terms, and the
distance-weighted bound for its nonlocal terms. Since $\psi_r\leq1$, this gives
\begin{align*}
\|e_r\|
&\leq\frac{2n\pi_q^*}{1.1}
+16n\pi_q^*\sqrt d\,\beta
+4\sqrt d\,\beta+16n\sqrt d\,\beta\\
&\leq2n\pi_q^*.
\end{align*}
The last comparison uses the standing dimension, in particular $\pi_q^*\geq\pimin^*$.
It follows also that
$\|\nabla_{\mu_r}\calL\|\leq A_r(\sqrt d+4n)+2n\pi_q^*\leq2\sqrt d$.
No upper bound on a cross-group distance has been used.
Thus the one-step mean update can be written as
\begin{equation}
\label{eq:stage1:one-step-decomposition}
\mu_r'=(1-\eta A_r)\mu_r+\eta A_r\mu_q^*-\eta e_r,
\qquad \|e_r\|\leq2n\pi_q^*.
\end{equation}
The learning-rate assumption ensures $0\leq\eta A_r\leq1$.

\paragraph{Proof of the second item.}
Let $r_l:=\|\mu_l-\parmu_l\|$. Projecting
\cref{eq:stage1:one-step-decomposition} onto $\calP^\bot$ gives
\begin{equation}
\label{eq:stage1:orthogonal-update}
P_{\calP^\bot}\mu_l'=(1-\eta A_l)P_{\calP^\bot}\mu_l
-\eta P_{\calP^\bot}e_l.
\end{equation}
For $k\in S_q\setminus\{q\}$, let $\delta:=r_k-r_q\geq1/n^3$. The reverse and ordinary triangle
inequalities, \cref{eq:stage1:second-moment-gap}, and $\|e_l\|\leq2n\pi_q^*$ imply
\begin{align*}
r_k'-r_q'
&\geq(1-\eta A_k)r_k-(1-\eta A_q)r_q
-\eta(\|e_k\|+\|e_q\|)\\
&=\delta-\eta A_k\delta+\eta(A_q-A_k)r_q
-\eta(\|e_k\|+\|e_q\|)\\
&\geq\delta+\frac{\eta\pi_q^*\sqrt d}{800n^9}.
\end{align*}
In the last step we used $r_q\geq\sqrt d-2d^{2\epsilon}$ and the consequences
$\eta A_k\delta+4\eta n\pi_q^*\leq\eta\pi_q^*\sqrt d/(4000n^9)$ of
\cref{eq:standing-dimension}. In particular, the gap remains at least $1/n^3$, proving the second
invariant at time $t+1$. Iterating the strict increment for $T_1$ steps gives
\[
r_k(T_1)-r_q(T_1)
\geq\frac1{n^3}+T_1\frac{\eta\pi_q^*\sqrt d}{800n^9}
\geq d^\epsilon,
\]
where the last inequality follows from
$T_1\geq1000n^{10}d^{\epsilon-0.5}/\eta$ and
$n\pi_q^*\geq n\pimin^*\geq3\log m>1$. This proves the strengthened terminal gap.

\paragraph{Proof of the third and fifth items.}
Let $Q_r$ denote the orthogonal projection onto the complement of
$\calP\oplus\operatorname{span}\{u_r\}$. Then
$Q_r\mu_r=\mu_r-\parmu_r-\botmu_r$. Projecting
\cref{eq:stage1:one-step-decomposition} onto $\calP$ and onto the range of $Q_r$ gives
\begin{align*}
\left(\parmu_r\right)'-\mu_q^*
&=(1-\eta A_r)(\parmu_r-\mu_q^*)-\eta P_{\calP}e_r,\\
Q_r\mu_r'&=(1-\eta A_r)Q_r\mu_r-\eta Q_re_r.
\end{align*}
Because $0\leq1-\eta A_r\leq1$, both norms can increase by at most
$\eta\|e_r\|\leq2n\eta\pi_q^*\leq2n\eta$ in one step. Starting from the initialization bounds
in \cref{thm:init}, we obtain
\begin{align*}
\|(\parmu_r)^{(t)}-\mu_q^*\|&\leq n+2n\eta t,\\
\|\mu_r^{(t)}-(\parmu_r)^{(t)}-(\botmu_r)^{(t)}\|&\leq n+2n\eta t.
\end{align*}
The fixed directions $u_r$ remain pairwise orthogonal, so the corresponding vectors
$(\botmu_r)^{(t)}$ remain pairwise orthogonal as well. Finally,
$2n\eta T_1\leq n$ under \cref{eq:standing-dimension}; this proves the third and fifth invariants,
including their displayed upper bounds by $2n$.

\paragraph{Proof of the fourth item.}
Equation~\eqref{eq:stage1:orthogonal-update} shows
$r_r'\leq r_r+2n\eta\pi_q^*$, because the principal term is a contraction. Hence
\[
r_r(t)\leq r_r(0)+2n\eta t\leq\sqrt d+n+2n\eta t.
\]
For the lower bound, use the crude gradient estimate following
\cref{eq:stage1:uniform-diameter}. Orthogonal projection is nonexpansive, so
\begin{align*}
r_r(t)&\geq r_r(0)-\sum_{s=0}^{t-1}\eta
\|\nabla_{\mu_r}\calL(\pi^{(s)},\mu^{(s)})\|\\
&\geq\sqrt d-n-2\eta t\sqrt d\\
&\geq\sqrt d-2d^{2\epsilon}.
\end{align*}
The last step uses
$n+2\eta T_1\sqrt d\leq2d^{2\epsilon}$, which follows from
\cref{eq:standing-dimension}. This proves the fourth invariant.

All five invariants have now been propagated from $t$ to $t+1$. Their strengthened versions at
$t=T_1$ were established in the same argument, completing the induction.
\end{proof}

\subsection{Stage 2 Analysis}
\label{subsec:stage2}
We reset $t=0$ at the beginning of this stage. For $1\leq i\leq n$, let
$\lambda_i:=\|\mu_i^{(0)}-(\parmu_i)^{(0)}\|$; thus $\lambda_i$ is fixed throughout Stage~2.
Write $\Delta=d^{0.5-\epsilon}$ and define
\begin{equation}
\label{eq:stage2:overlap-time}
K:=\left\lceil
\frac{\left(1+\frac{\pimin^*}{4\pi_1^*}\right)d^{-2\epsilon}}{2\pi_1^*\eta}
\right\rceil.
\end{equation}
Beginning in this stage, we assume that for all $1 \le i < j \le m$, $\|\mu_i^* - \mu_j^*\| = d^{0.5-\epsilon}.$
Such a configuration exists in an $(m-1)$-dimensional linear subspace: writing
$\mathbf 1_m=(1,\ldots,1)$, one may take
\[
\mu_j^*=\frac{d^{0.5-\epsilon}}{\sqrt2}
\left(e_j-\frac1m\mathbf 1_m\right),\qquad 1\leq j\leq m.
\]

By \cref{thm:stage1}, we know $\sqrt{d}-2d^{2\epsilon}\leq \lambda_i\leq \sqrt{d}+2n$.
Throughout the proof below we analyze a fixed iteration $t$. Unless a time superscript is displayed,
$\mu_i$, $\pi_i$, and $\psi_i$ mean the parameters and responsibilities computed from
$(\mu^{(t)},\pi^{(t)})$. We write $\mu_i'$ and $\pi_i'$ for the parameters after the update, i.e.,
at time $t+1$. Thus every occurrence of $\mu_i$ in
\cref{lm:stage2:similar,lm:stage2:mu1small,lm:stage2:differentpsi,lm:stage2:psiprod,lm:stage2:larget,lm:stage2:gradient}
refers to a time-$t$ iterate. This convention is important because these lemmas are invoked inside the
outer induction proving \cref{lm:stage2}.

\begin{lemma}
\label{lm:stage2}
For every $0\leq t\leq T_2:=\lceil\log d/(\pi_1^*\eta)\rceil$, the following invariants hold:
\begin{itemize}
\item We always have $\pi_1\geq1/n^2$.
\item For $2\leq i\leq n,i\in S_j$, $\|\parmu_i-\mu_j^*\|\leq 2n+\frac{800}{(\pigap)^2}$. Also, $\|\mu_i-\parmu_i\|\leq \sqrt{d}+2n+\frac{800}{(\pigap)^2}$.
If $i>m$, then in addition
\[
\|\mu_i-\parmu_i\|\geq\lambda_i-\eta t\exp(-d^{\epsilon/2}).
\]
\item For any $j\in S_i$, $j\neq i$, we have
$\|\mu_j-\parmu_j\|\geq\|\mu_i-\parmu_i\|+d^\epsilon/2$ and
$\|\mu_j-\mu_i^*\|\geq \|\mu_i-\mu_i^*\|+d^{\epsilon}/3$;
\item For any $2\leq i\leq m$, $\|\mu_i-\parmu_i\|\geq \lambda_i-\pi_i^* \eta t\left(\sqrt{d}+80d^{2\epsilon}/\pigap\right)$.
\item For any $2\leq i\leq n$, 
$\|\mu_i-\parmu_i\|\geq \sqrt{d}-100d^{0.5-2\epsilon}/\pigap$.
\item $\|\mu_1-\parmu_1\|\leq \lambda_1\left(1-\pi_1^* \eta+d^{-0.4}\eta\right)^t$.
\item When $1\leq t\leq d^{-2\epsilon}/(10\eta)$, for $1\leq i\leq m$,
$\pi_i\geq \pi_i^*/2$.
\item When $t\leq d^{-2\epsilon}/(10\eta)$,
$\|\parmu_1-\mu_1^*\|\leq 3n+10/\pigap$ and, for every $1\leq j\leq n$,
\begin{equation}
\label{eq:stage2:mu1-lower}
\|\mu_j-\parmu_j\|\geq \|\mu_j^{(0)}-(\parmu_j)^{(0)}\|-2\eta t\sqrt d
\geq \sqrt d-2d^{2\epsilon}-\frac15d^{0.5-2\epsilon}.
\end{equation}
\item For each $1\leq i\leq m$, $\|\parmu_1-\mu_i^*\|\leq d^{0.5-\epsilon}+2n+\frac{800}{(\pigap)^2}$.
\end{itemize}

Moreover, at $t=T_2$ we have
\begin{itemize}
\item $\|\mu_1-\parmu_1\|\leq 1$.
\item $\pi_1\geq1-2n\exp(-d^{2\epsilon/3}/3)$, and
$\pi_i\leq2\exp(-d^{2\epsilon/3}/3)$ for every $i\geq2$.
\end{itemize}
\end{lemma}

The proof first establishes \cref{lm:stage2:early} directly from the Stage~1 output.
This preliminary lemma does not assume any Stage~2 invariant. We then prove the properties in \cref{lm:stage1} by induction.

\begin{lemma}
\label{lm:stage2:similar}
Let $x\sim \calN(\mu_i^*,I_d)$. With probability at least $1-\exp\left(-d^{2\epsilon/3}/3\right)$, we have
\[
\|\bx-\mu_i^*\|\leq d^{\epsilon/3}.
\]
Here $\bx$ is the projection of $x$ on $\calP_0$.
\end{lemma}

\begin{proof}
Let $P_0$ be the orthogonal projector onto $\calP_0$ and put $d_0=\dim(\calP_0)$.
Because $\mu_i^*\in\calP_0$ and $x=\mu_i^*+z$ with $z\sim\calN(0,I_d)$, we have
$\bx=P_0x=\mu_i^*+P_0z$. In an orthonormal basis of $\calP_0$, the vector $P_0z$ has distribution
$\calN(0,I_{d_0})$. Hence \cref{lm:gausstail} gives
\[
\Pr\!\left[\|\bx-\mu_i^*\|>d^{\epsilon/3}\right]
\leq \exp\!\left(-\frac{(d^{\epsilon/3}-\sqrt{d_0})^2}{2}\right).
\]
The initialization construction gives $d_0\leq m+n$, while
\cref{eq:standing-dimension} implies $\sqrt{d_0}\leq d^{\epsilon/3}/10$. Therefore
$(d^{\epsilon/3}-\sqrt{d_0})^2/2\geq d^{2\epsilon/3}/3$, which proves the claim.
\end{proof}

\begin{lemma}
\label{lm:stage2:early}
For the trajectory initialized by the Stage~1 output, for every integer $0\leq t\leq K$,
\begin{align}
&\left\|\mu_i^{(t)}-\mu_i^*
-(1-\eta\pi_i^*)^t(\mu_i^{(0)}-\mu_i^*)\right\|
\leq20n\sqrt d\exp(-d^{2\epsilon/3}/3)\eta t\leq1
\quad(1\leq i\leq m),\label{eq:stage2:early-path}\\
&\|\mu_j^{(t)}-\mu_j^{(0)}\|
\leq20n\sqrt d\exp(-d^{2\epsilon/3}/3)\eta t\leq1
\quad(j>m).
\label{eq:stage2:early-stationary}
\end{align}
For $1\leq t\leq K$, $\pi_i^{(t)}\geq\pi_i^*/2$ for every $1\leq i\leq m$.
Moreover, for $0\leq t\leq K$, if $\|\bx-\mu_i^*\|\leq d^{\epsilon/3}$, then
\[
\psi_j^{(t)}(\bx)<\exp(-d^\epsilon)\qquad(j\neq i).
\]
\end{lemma}

\begin{proof}
We induct on the two path estimates and the weight lower bounds, starting from the Stage~1
output. At time zero $\pi_i\geq\pi_i^*/n$, and at each later time the induction
hypothesis gives $\pi_i\geq\pi_i^*/2$. All estimates below concern the current time.
The standing dimension and step size imply
\[
\eta K\leq\frac{d^{-2\epsilon}}{\pi_1^*},
\qquad
20n\sqrt d\exp(-d^{2\epsilon/3}/3)\eta K\leq1.
\]

The path hypotheses give a parallel displacement at most $2n+1$ from the appropriate true
mean. They also give orthogonal radii within one of $(1-\eta\pi_i^*)^t\lambda_i$
for representatives and within one of $\lambda_j$ for non-representatives.
For $j\in S_i\setminus\{i\}$, the initial gap and $(1-\eta\pi_i^*)^t\leq1$ imply
\[
\|\mu_j-\mu_i^*\|-\|\mu_i-\mu_i^*\|
\geq\lambda_j-\lambda_i-2-4n^2/\sqrt d\geq d^\epsilon/2.
\]
The sum of these two distances is at least $\sqrt d/2$. Thus their squared-distance gap is at
least $d^{0.5+\epsilon}/4$.

Now fix two representatives $1\leq i,j\leq m$ with $j\neq i$.
Since both true means belong to $\calP$, orthogonality gives the exact identity
\begin{align*}
&\|\mu_j-\mu_i^*\|^2-\|\mu_i-\mu_i^*\|^2\\
&=\|\parmu_j-\mu_i^*\|^2-\|\parmu_i-\mu_i^*\|^2
+\|\mu_j-\parmu_j\|^2-\|\mu_i-\parmu_i\|^2.
\end{align*}
We bound the parallel and orthogonal terms separately.

For the parallel terms, the path hypotheses and the Stage~1 output give
$\|\parmu_i-\mu_i^*\|,\|\parmu_j-\mu_j^*\|\leq2n+1$.
As $\|\mu_j^*-\mu_i^*\|=\Delta$ and $\Delta>2n+1$, the triangle inequality yields
\begin{align*}
\|\parmu_j-\mu_i^*\|^2-\|\parmu_i-\mu_i^*\|^2
&\geq(\Delta-2n-1)^2-(2n+1)^2\\
&=\Delta^2-(4n+2)\Delta.
\end{align*}

For the orthogonal terms, projecting the path hypothesis onto $\calP^\perp$ and
applying the reverse triangle inequality gives, for each representative $i$,
\[
\left|\|\mu_i-\parmu_i\|-(1-\eta\pi_i^*)^t\lambda_i\right|\leq1.
\]
Here $0\leq(1-\eta\pi_i^*)^t\leq1$. Consequently,
\begin{align*}
&\left|\|\mu_i-\parmu_i\|^2-(1-\eta\pi_i^*)^{2t}\lambda_i^2\right|\\
&\leq
\left|\|\mu_i-\parmu_i\|-(1-\eta\pi_i^*)^t\lambda_i\right|
\left(\|\mu_i-\parmu_i\|+(1-\eta\pi_i^*)^t\lambda_i\right)\\
&\leq2(1-\eta\pi_i^*)^t\lambda_i+1
\leq2(\sqrt d+2n)+1
\leq5\sqrt d.
\end{align*}
The same estimate holds for representative $j$.
Squaring the Stage~1 bounds
$\sqrt d-2d^{2\epsilon}\leq\lambda_i\leq\sqrt d+2n$ gives
\[
d-\lambda_i^2\leq4d^{0.5+2\epsilon},
\qquad
\lambda_i^2-d\leq4n\sqrt d+4n^2\leq5d^{0.5+2\epsilon}.
\]
Thus $|\lambda_i^2-d|\leq5d^{0.5+2\epsilon}$, and the same bound holds for
$\lambda_j$. Combining these estimates, we obtain
\begin{align*}
&\|\mu_j-\parmu_j\|^2-\|\mu_i-\parmu_i\|^2\\
&\geq(1-\eta\pi_j^*)^{2t}\lambda_j^2
-(1-\eta\pi_i^*)^{2t}\lambda_i^2-10\sqrt d\\
&\geq d\left((1-\eta\pi_j^*)^{2t}-(1-\eta\pi_i^*)^{2t}\right)
-10d^{0.5+2\epsilon}-10\sqrt d.
\end{align*}
The last line uses that both contraction factors lie in $[0,1]$.

Bernoulli's inequality, together with
$\exp(-u)\leq1-u+u^2/2$ for $u\geq0$, gives
\[
1-2\pi_i^*\eta t
\leq(1-\eta\pi_i^*)^{2t}
\leq\exp(-2\pi_i^*\eta t)
\leq1-2\pi_i^*\eta t+2(\pi_i^*\eta t)^2.
\]
Using the lower bound for the factor indexed by $j$ and the upper bound for the
factor indexed by $i$, we have
\begin{align*}
(1-\eta\pi_j^*)^{2t}-(1-\eta\pi_i^*)^{2t}
&\geq(1-2\pi_j^*\eta t)
-\left(1-2\pi_i^*\eta t+2(\pi_i^*\eta t)^2\right)\\
&=2(\pi_i^*-\pi_j^*)\eta t-2(\pi_i^*\eta t)^2.
\end{align*}
Since $\pi_i^*\leq1$, $\eta t\leq d^{-2\epsilon}/\pi_1^*$, and
$\sqrt d\leq d^{0.5+2\epsilon}$, it follows that
\begin{align*}
&\|\mu_j-\parmu_j\|^2-\|\mu_i-\parmu_i\|^2\\
&\geq2(\pi_i^*-\pi_j^*)\eta td
-2d(\pi_i^*\eta t)^2-10d^{0.5+2\epsilon}-10\sqrt d\\
&\geq2(\pi_i^*-\pi_j^*)\eta td
-\frac{2d^{1-4\epsilon}}{(\pi_1^*)^2}-20d^{0.5+2\epsilon}.
\end{align*}
Adding the parallel-distance bound now proves
\begin{equation}
\label{eq:stage2:early-cross-gap}
\begin{aligned}
&\|\mu_j-\mu_i^*\|^2-\|\mu_i-\mu_i^*\|^2\\
&\geq\Delta^2+2(\pi_i^*-\pi_j^*)\eta td
-(4n+2)\Delta-20d^{0.5+2\epsilon}
-\frac{2d^{1-4\epsilon}}{(\pi_1^*)^2}\\
&\geq\Delta^2+2(\pi_i^*-\pi_j^*)\eta td\\
&\quad-40\left(d^{0.5+2\epsilon}+(n+1)\Delta+(n+1)^2
+\frac{d^{1-4\epsilon}}{(\pi_1^*)^2}+\eta d\right).
\end{aligned}
\end{equation}
The last line enlarges the error bound, using $4n+2\leq40(n+1)$.
By \cref{eq:standing-dimension},
\[
40\left(d^{0.5+2\epsilon}+(n+1)\Delta+(n+1)^2
+\frac{d^{1-4\epsilon}}{(\pi_1^*)^2}+\eta d\right)
\leq\frac{\pimin^*\Delta^2}{8\pi_1^*}.
\]
Since $\pi_j^*-\pi_i^*\leq\pi_1^*-\pimin^*$ and
\[
\eta t\leq
\frac{\left(1+\frac{\pimin^*}{4\pi_1^*}\right)d^{-2\epsilon}}{2\pi_1^*}+\eta,
\]
we have
\begin{align*}
\Delta^2-2(\pi_1^*-\pimin^*)\eta td
&\geq\left[1-\left(1-\frac{\pimin^*}{\pi_1^*}\right)
\left(1+\frac{\pimin^*}{4\pi_1^*}\right)\right]\Delta^2-2\pi_1^*\eta d\\
&\geq\frac{3\pimin^*\Delta^2}{4\pi_1^*}-2\pi_1^*\eta d.
\end{align*}
Thus the squared-distance gap in \cref{eq:stage2:early-cross-gap} is at least
$\pimin^*\Delta^2/(4\pi_1^*)$.
For a non-representative $j>m$ with $j\notin S_i$, the bound
$\|\mu_j^{(t)}-\mu_j^{(0)}\|\leq1$ gives
$\left|\|\mu_j-\parmu_j\|-\lambda_j\right|\leq1$.
Its parallel distance to $\mu_i^*$ is at least $\Delta-2n-1$, because its own true
mean is at distance $\Delta$ from $\mu_i^*$.
The preceding squared-radius estimates therefore give
\begin{align*}
&\|\mu_j-\mu_i^*\|^2-\|\mu_i-\mu_i^*\|^2\\
&\geq\Delta^2-(4n+2)\Delta
+\lambda_j^2-(1-\eta\pi_i^*)^{2t}\lambda_i^2-10\sqrt d\\
&\geq\Delta^2-(4n+2)\Delta
+d\left(1-(1-\eta\pi_i^*)^{2t}\right)
-10d^{0.5+2\epsilon}-10\sqrt d\\
&\geq\Delta^2-(4n+2)\Delta-20d^{0.5+2\epsilon}\\
&\geq\Delta^2-\frac{\pimin^*\Delta^2}{8\pi_1^*}
\geq\Delta^2/2.
\end{align*}
Here we used $(1-\eta\pi_i^*)^{2t}\leq1$ and the same bound on the error terms
as in \cref{eq:stage2:early-cross-gap}.

Every fitted mean is within $2\sqrt d$ of every true mean during this induction.
Let $G_i:=\{\|\bx-\mu_i^*\|\leq d^{\epsilon/3}\}$.
Changing the sample from $\mu_i^*$ to $\bx\in G_i$ changes any squared-distance
difference by at most $8\sqrt d\,d^{\epsilon/3}$. The standing dimension ensures
\[
\frac12\left(\min\left\{\frac{d^{0.5+\epsilon}}4,
\frac{\pimin^*\Delta^2}{4\pi_1^*}\right\}
-8\sqrt d\,d^{\epsilon/3}\right)
>d^\epsilon+\log\frac{n}{\pimin^*}.
\]
Comparison with representative $i$, whose weight is at least $\pi_i^*/n$, proves
the stated pointwise responsibility bound. By \cref{lm:stage2:similar},
$\E_{x\sim i}[\psi_j(x)]\leq2\exp(-d^{2\epsilon/3}/3)$ for $j\neq i$.
The weight update consequently gives
\[
\pi_i'\geq\pi_i^*\left(1-2n\exp(-d^{2\epsilon/3}/3)\right)\geq\pi_i^*/2.
\]

Finally, isolate the source-$i$, fitted-index-$i$ term in the gradient of a representative.
Using $1-\psi_i^2\leq2\sum_{j\neq i}\psi_j$ under source $i$ and the bound
$\|\mu_j-\mu_i^*\|\leq2\sqrt d$ for every fitted index $j$ and true index $i$ gives
\begin{align*}
\|\nabla_{\mu_i}\calL-\pi_i^*(\mu_i-\mu_i^*)\|
&\leq(12n+4)\sqrt d\exp(-d^{2\epsilon/3}/3)\\
&\leq20n\sqrt d\exp(-d^{2\epsilon/3}/3).
\end{align*}
For a non-representative, the same estimates give
\[
\|\nabla_{\mu_j}\calL\|\leq4\sqrt d\exp(-d^{2\epsilon/3}/3)
\leq20n\sqrt d\exp(-d^{2\epsilon/3}/3).
\]
The representative update is therefore a contraction toward its true mean plus an error
of norm at most $20\eta n\sqrt d\exp(-d^{2\epsilon/3}/3)$; summing these errors proves
\cref{eq:stage2:early-path,eq:stage2:early-stationary} at time $t+1$.
This closes the induction. The same current-time comparisons at $t=K$ prove the last
responsibility assertion at the endpoint as well.
\end{proof}

\begin{lemma}
\label{lm:stage2:mu1small}
Under assumptions in \cref{lm:stage2}, for any $2\leq i\leq m$, we always have
\[
\|\mu_i-\parmu_i\|-\|\mu_1-\parmu_1\|\geq \min\left\{d^{0.5-1.5\epsilon},\frac{\pigap}{2}\eta t\sqrt{d}-3d^{2\epsilon}\right\}.
\]
\end{lemma}

\begin{proof}
Write $r_i(t):=\|\mu_i^{(t)}-(\parmu_i)^{(t)}\|$ and
$q:=\pi_1^*-d^{-0.4}$. For
$t\leq t_0:=10d^{-1.5\epsilon}/(\eta\pigap)$, the standing dimension and step-size
conditions imply $0\leq q\eta t\leq1/10$. The elementary estimate
$(1-a)^t\leq1-at+2(at)^2$, valid when $0\leq at\leq1/10$, together with the fourth and
sixth invariants gives
\begin{align*}
r_i(t)-r_1(t)
&\geq \lambda_i-\pi_i^*\eta t(\sqrt d+80d^{2\epsilon}/\pigap)
-\lambda_1(1-q\eta)^t\\
&\geq \lambda_i-\lambda_1
 +\eta t\bigl(\lambda_1q-\pi_i^*\sqrt d-80\pi_i^*d^{2\epsilon}/\pigap\bigr)
 -2\lambda_1(q\eta t)^2.
\end{align*}
To justify the simplification of the last line, write
\begin{align*}
\lambda_1q-\pi_i^*\sqrt d
&=(\pi_1^*-\pi_i^*)\sqrt d
  -\pi_1^*(\sqrt d-\lambda_1)-\lambda_1d^{-0.4}.
\end{align*}
Since $|\sqrt d-\lambda_1|\leq2d^{2\epsilon}+2n$ and
$\pi_1^*-\pi_i^*\geq\pigap$, \cref{eq:standing-dimension} implies
\[
\pi_1^*|\sqrt d-\lambda_1|+\lambda_1d^{-0.4}
+80\pi_i^*d^{2\epsilon}/\pigap+2\lambda_1q^2\eta t
\leq\frac{\pi_1^*-\pi_i^*}{2}\sqrt d.
\]
Using also $\lambda_i-\lambda_1\geq-3d^{2\epsilon}$ therefore yields
\begin{align*}
r_i(t)-r_1(t)
&\geq \frac{\pi_1^*-\pi_i^*}{2}\eta t\sqrt d-3d^{2\epsilon}\\
&\geq \frac{\pigap}{2}\eta t\sqrt d-3d^{2\epsilon}.
\end{align*}

At the real comparison threshold $t_0$, the affine lower bound is at least
$5d^{0.5-1.5\epsilon}-3d^{2\epsilon}\geq d^{0.5-1.5\epsilon}$.
No iterate at the possibly noninteger time $t_0$ is needed in the following comparison of envelopes.
For integer $t\geq t_0$, the sixth invariant gives the decreasing upper envelope
$r_1(t)\leq U(t):=\lambda_1(1-q\eta)^t$, whereas the fifth invariant gives the uniform lower
envelope
\[
r_i(t)\geq L:=\sqrt d-100d^{0.5-2\epsilon}/\pigap.
\]
Moreover, $(1-q\eta)^{t_0}\leq1-q\eta t_0/2$ and
$q\geq\pi_1^*/2\geq\pigap/2$, so
\begin{align*}
L-U(t_0)
&\geq \sqrt d-100d^{0.5-2\epsilon}/\pigap-\lambda_1
   +\frac12\lambda_1q\eta t_0\\
&\geq d^{0.5-1.5\epsilon}.
\end{align*}
The last inequality is another explicit consequence of \cref{eq:standing-dimension}: the positive
term is at least $2d^{0.5-1.5\epsilon}$, while
$|\lambda_1-\sqrt d|+100d^{0.5-2\epsilon}/\pigap$ is at most
$d^{0.5-1.5\epsilon}$. Since $U(t)\leq U(t_0)$ for $t\geq t_0$, we conclude
$r_i(t)-r_1(t)\geq L-U(t_0)\geq d^{0.5-1.5\epsilon}$. Combining the two time ranges proves the
claimed minimum bound.
\end{proof}

\begin{lemma}
\label{lm:stage2:differentpsi}
Under assumptions in \cref{lm:stage2},
for $1\leq i\leq m$ and $2\leq j\leq n$ such that $j\neq i$, when $\|\bx-\mu_i^*\|\leq d^{\epsilon/3}$,
\[
\psi_j(\bx)\leq \exp(-d^\epsilon).
\]
Moreover, when $t\leq d^{-2\epsilon}/(10\eta)$ and $2\leq i\leq m$, the same bound also holds for
$j=1$.

\end{lemma}

\begin{proof}
For any reference component $r$ with $\pi_r>0$,
\begin{equation}
\label{eq:stage2:responsibility-comparison}
\psi_j(\bx)\leq\frac{\pi_j}{\pi_r}
\exp\!\left[-\frac12\bigl(\|\bx-\mu_j\|^2-\|\bx-\mu_r\|^2\bigr)\right].
\end{equation}
Writing $h=\bx-\mu_i^*$, the exact identity is
\begin{equation}
\label{eq:stage2:sample-perturbation}
\|\bx-\mu_j\|^2-\|\bx-\mu_r\|^2
=\|\mu_i^*-\mu_j\|^2-\|\mu_i^*-\mu_r\|^2
-2\langle h,\mu_j-\mu_r\rangle.
\end{equation}
The current geometric invariants imply $\|\mu_j-\mu_r\|\leq4\sqrt d$.
Thus the last term has absolute value at most $8\sqrt d\,d^{\epsilon/3}$ on the stated event.

For $t\leq K$, all claims follow directly from \cref{lm:stage2:early}.
In particular, this covers the additional assertion for $j=1$, since
$d^{-2\epsilon}/(10\eta)<K$.
For the remainder of the proof suppose $t>K$.

First let $j>m$ be any non-representative, with no restriction on its source group.
The second invariant gives
$r_j:=\|\mu_j-\parmu_j\|\geq\sqrt d-3d^{2\epsilon}$.
The sixth invariant and $t\geq K$ give
\begin{align*}
r_1^2
&\leq\lambda_1^2\exp\left(-2(\pi_1^*-d^{-0.4})\eta K\right)\\
&\leq d-\left(1+\frac{\pimin^*}{4\pi_1^*}\right)\Delta^2
+10d^{0.5+2\epsilon}
+\frac{10d^{1-4\epsilon}}{(\pi_1^*)^2}
+\frac{4d^{0.6-2\epsilon}}{\pi_1^*}\\
&\leq d-\left(1+\frac{\pimin^*}{8\pi_1^*}\right)\Delta^2.
\end{align*}
For the second line, use $e^{-u}\leq1-u+u^2/2$, the initial radius bounds, and
\[
2(\pi_1^*-d^{-0.4})\eta K
\geq\frac{\pi_1^*-d^{-0.4}}{\pi_1^*}
\left(1+\frac{\pimin^*}{4\pi_1^*}\right)d^{-2\epsilon}.
\]
All error terms are at most $\pimin^*\Delta^2/(8\pi_1^*)$ by the standing dimension.
The ninth invariant consequently yields, for every true source $i$,
\begin{align*}
\|\mu_j-\mu_i^*\|^2-\|\mu_1-\mu_i^*\|^2
&\geq(\sqrt d-3d^{2\epsilon})^2
-\left[d-\left(1+\frac{\pimin^*}{8\pi_1^*}\right)\Delta^2\right]\\
&\quad-\left(\Delta+2n+\frac{800}{(\pigap)^2}\right)^2\\
&\geq\frac{\pimin^*\Delta^2}{16\pi_1^*}.
\end{align*}
This is a direct comparison with component~1, even when $j\in S_i$.

It remains to consider a representative $2\leq j\leq m$ with $j\neq i$.
Since $t>K>d^{-2\epsilon}/(10\eta)$, \cref{lm:stage2:mu1small} and the fifth invariant give
\[
r_j-r_1\geq\frac{\pigap}{30}d^{0.5-2\epsilon},\qquad
r_j+r_1\geq\sqrt d/2.
\]
Moreover,
\[
\|\parmu_j-\mu_i^*\|\geq\Delta-2n-\frac{800}{(\pigap)^2},
\qquad
\|\parmu_1-\mu_i^*\|\leq\Delta+2n+\frac{800}{(\pigap)^2},
\]
where the second inequality is the ninth invariant. Hence
\begin{align*}
\|\mu_j-\mu_i^*\|^2-\|\mu_1-\mu_i^*\|^2
&\geq\frac{\pigap}{60}\Delta^2
+\left(\Delta-2n-\frac{800}{(\pigap)^2}\right)^2
-\left(\Delta+2n+\frac{800}{(\pigap)^2}\right)^2\\
&=\frac{\pigap}{60}\Delta^2-4\left(2n+\frac{800}{(\pigap)^2}\right)\Delta\\
&\geq\frac{\pigap}{100}\Delta^2.
\end{align*}
In both cases the sample perturbation is at most $8\sqrt d\,d^{\epsilon/3}$, and
$\pi_j/\pi_1\leq n^2$. Half of the remaining energy margin exceeds
$d^\epsilon+2\log n$. Substituting $r=1$ in
\cref{eq:stage2:responsibility-comparison} proves the claim for every $t>K$.
\end{proof}

\begin{lemma}
\label{lm:stage2:psiprod}
Under assumptions in \cref{lm:stage2}, for any $2\leq i\leq m$, we have
\[
\E_{x\sim i}[\psi_1(x)\psi_i(x)]\leq \frac{30}{\pigap d^{0.5-2\epsilon}}.
\]
\end{lemma}

\begin{proof}
When $t\leq d^{-2\epsilon}/(10\eta)$, \cref{lm:stage2:differentpsi} shows that, on the event
$G_i:=\{\|\bx-\mu_i^*\|\leq d^{\epsilon/3}\}$,
$\psi_1(\bx)\leq\exp(-d^\epsilon)$ because $i\neq1$. Since $0\leq\psi_i,\psi_1\leq1$,
splitting the expectation over $G_i$ and $G_i^c$ and invoking \cref{lm:stage2:similar} gives
\[
\E_{x\sim i}[\psi_1(x)\psi_i(x)]
\leq e^{-d^\epsilon}\Pr(G_i)+\Pr(G_i^c)
\leq e^{-d^\epsilon}+e^{-d^{2\epsilon/3}/3}
\leq 2e^{-d^{2\epsilon/3}/3}.
\]

When $t>d^{-2\epsilon}/(10\eta)$, by \cref{lm:stage2:mu1small},
\[
\|\mu_i-\parmu_i\|-\|\mu_1-\parmu_1\|\geq \frac{\pigap}{20}d^{0.5-2\epsilon}-3d^{2\epsilon}\geq \frac{\pigap}{30}d^{0.5-2\epsilon}.
\]
Orthogonal projection is non-expansive, and hence
\[
\|\mu_i-\mu_1\|
\geq\left|\|\mu_i-\parmu_i\|-\|\mu_1-\parmu_1\|\right|
\geq\frac{\pigap}{30}d^{0.5-2\epsilon}.
\]
Applying \cref{lm:stage1:nocorrolation} with $r=\|\mu_i-\mu_1\|$ now yields
\[
\E_{x\sim i}[\psi_1(x)\psi_i(x)]\leq \frac{30}{\pigap d^{0.5-2\epsilon}}.
\]
\end{proof}

\begin{lemma}
\label{lm:stage2:larget}
Under assumptions in \cref{lm:stage2}, when $t>10d^{-2\epsilon}/(\eta\pigap)$, 
for every $2\leq i\leq m$,
\[
\E_{x\sim i}[\psi_i(x)]\leq 2\exp(-d^{2\epsilon/3}/3).
\]
\end{lemma}

\begin{proof}
Put $r_k=\|\mu_k-\parmu_k\|$. At
$t\geq10d^{-2\epsilon}/(\eta\pigap)$, \cref{lm:stage2:mu1small} gives
\begin{align*}
r_i-r_1
&\geq\min\left\{d^{0.5-1.5\epsilon},
5d^{0.5-2\epsilon}-3d^{2\epsilon}\right\}
\geq4d^{0.5-2\epsilon}.
\end{align*}
The fifth invariant gives $r_i\geq\sqrt d/2$. Hence
\begin{equation}
\label{eq:stage2:larget-radial-gap}
r_i^2-r_1^2=(r_i-r_1)(r_i+r_1)\geq2d^{1-2\epsilon}.
\end{equation}
Because the true means lie in $\calP$ and the orthogonal and parallel components are perpendicular,
\begin{align*}
&\|\mu_i-\mu_i^*\|^2-\|\mu_1-\mu_i^*\|^2\\
&=\|\parmu_i-\mu_i^*\|^2-\|\parmu_1-\mu_i^*\|^2+r_i^2-r_1^2\\
&\geq2d^{1-2\epsilon}
-\left(d^{0.5-\epsilon}+2n+800/(\pigap)^2\right)^2\\
&\geq\frac12d^{1-2\epsilon}.
\end{align*}
In the last line we used the ninth invariant and \cref{eq:standing-dimension}.
On $G_i:=\{\|\bx-\mu_i^*\|\leq d^{\epsilon/3}\}$,
\cref{eq:stage2:sample-perturbation} shows that replacing $\mu_i^*$ by $\bx$ reduces this energy
gap by at most $8\sqrt d\,d^{\epsilon/3}$. Therefore
\[
\|\bx-\mu_i\|^2-\|\bx-\mu_1\|^2\geq\frac14d^{1-2\epsilon}.
\]
Since $\pi_1\geq1/n^2$ and $\pi_i\leq1$, comparison with component~$1$ in the responsibility
denominator yields
\[
\psi_i(\bx)
\leq n^2\exp\!\left(-\frac18d^{1-2\epsilon}\right)
\leq e^{-\sqrt d}.
\]
Finally, $0\leq\psi_i\leq1$ and \cref{lm:stage2:similar} imply
\[
\E_{x\sim i}[\psi_i(x)]
\leq e^{-\sqrt d}+\Pr(G_i^c)
\leq e^{-\sqrt d}+e^{-d^{2\epsilon/3}/3}
\leq2e^{-d^{2\epsilon/3}/3}.
\]
\end{proof}

\begin{lemma}
\label{lm:stage2:gradient}
Under assumptions in \cref{lm:stage2}, we have the following inequalities:
\begin{enumerate}
\item For $i\geq m+1$,
$\|\nabla_{\mu_i}\calL\|\leq \exp\left(-d^{\epsilon/2}\right)$;
\item For $2\leq i\leq m$, $\left\|\nabla_{\mu_i}\calL-\pi_i^*\E_{x\sim i}\left[\psi_i(x)^2(\mu_i-\mu_i^*)\right]\right\|\leq 70\pi_i^*d^{2\epsilon}/\pigap$
\item $\left\|\nabla_{\mu_1}\calL-\sum_j \pi_j^*\E_{x\sim j}\left[\psi_1(x)^2(\mu_1-\mu_j^*)\right]\right\|\leq 70d^{2\epsilon}/\pigap$.
\end{enumerate}

Moreover, when $t>10d^{-2\epsilon}/(\eta\pigap)$,
\begin{enumerate}
\item For $2\leq i\leq m$, $\left\|\nabla_{\mu_i}\calL-\pi_i^*\E_{x\sim i}\left[\psi_i(x)^2(\mu_i-\mu_i^*)\right]\right\|\leq 10n\sqrt{d}\exp\left(-d^{2\epsilon/3}/3\right)$
\item $\left\|\nabla_{\mu_1}\calL-\sum_j \pi_j^*\E_{x\sim j}\left[\psi_1(x)^2(\mu_1-\mu_j^*)\right]\right\|\leq 10n\sqrt{d}\exp\left(-d^{2\epsilon/3}/3\right)$.
\end{enumerate}
\end{lemma}

\begin{proof}
Recall that
\[
\nabla_{\mu_i}\calL=\sum_j \pi_j^*\E_{x\sim j}\left[\psi_i(x)\sum_k \psi_k(x)(\mu_k-\mu_j^*)\right].
\]
The current-time invariants imply, for every fitted index $k$ and true index $j$,
\[
\|\mu_k-\mu_j^*\|
\leq\|\mu_k-\parmu_k\|+\|\parmu_k-\mu_j^*\|
\leq2\sqrt d,
\]
where the last inequality uses the diameter $\Delta=d^{0.5-\epsilon}$ and
\cref{eq:standing-dimension}. Because the responsibilities sum to one, it follows that
\begin{equation}
\label{eq:stage2:gradient-integrand}
\left\|\sum_k\psi_k(x)(\mu_k-\mu_j^*)\right\|\leq2\sqrt d.
\end{equation}

If $i>m$, then $i\neq j$ for every true component $1\leq j\leq m$. On the good event of
\cref{lm:stage2:similar}, \cref{lm:stage2:differentpsi} gives $\psi_i\leq e^{-d^\epsilon}$;
on its complement we use $\psi_i\leq1$. Thus, for every $j$,
\[
\E_{x\sim j}[\psi_i(x)]
\leq e^{-d^\epsilon}+e^{-d^{2\epsilon/3}/3}
\leq2e^{-d^{2\epsilon/3}/3}.
\]
Combining this estimate with \cref{eq:stage2:gradient-integrand} gives
\[
\|\nabla_{\mu_i}\calL\|
\leq4\sqrt d\,e^{-d^{2\epsilon/3}/3}
\leq e^{-d^{\epsilon/2}},
\]
where the final comparison follows from \cref{eq:standing-dimension}. This proves item~1.

For $2\leq i\leq m$, isolate from the gradient the summand with source $j=i$ and fitted index
$k=i$. The triangle inequality and \cref{eq:stage2:gradient-integrand} give
\begin{align*}
&\left\|\nabla_{\mu_i}\calL-\pi_i^*\E_{x\sim i}\left[\psi_i(x)^2(\mu_i-\mu_i^*)\right]\right\|\\
&\leq 2\sqrt d\sum_{j\neq i}\pi_j^*\E_{x\sim j}[\psi_i(x)]\\
&\quad+2\pi_i^*\sqrt d\sum_{k\neq i}
\E_{x\sim i}[\psi_i(x)\psi_k(x)]\\
&\leq8n\sqrt d\exp\left(-d^{2\epsilon/3}/3\right)
+2\pi_i^*\sqrt d\E_{x\sim i}[\psi_i(x)\psi_1(x)].
\end{align*}
For the last line, the good-event/bad-event split above controls the source terms $j\neq i$.
For $k\neq1,i$, the same split controls $\psi_k$ under source $i$; the only interaction not
exponentially small at all times is the pair $(1,i)$, which is retained explicitly.

For component~$1$, isolate the fitted-index $k=1$ term separately for every true source. This gives
\begin{align*}
&\left\|\nabla_{\mu_1}\calL-\sum_j \pi_j^*\E_{x\sim j}\left[\psi_1(x)^2(\mu_1-\mu_j^*)\right]\right\|\\
&\leq2\sqrt d\sum_j\pi_j^*\sum_{k\neq1}
\E_{x\sim j}[\psi_1(x)\psi_k(x)]\\
&\leq8n\sqrt d\exp\left(-d^{2\epsilon/3}/3\right)
+2\sqrt d\sum_{2\leq j\leq m}\pi_j^*
\E_{x\sim j}[\psi_j(x)\psi_1(x)].
\end{align*}
Indeed, for source $j=1$ every $k\neq1$ is exponentially small, while for source $j>1$ every
$k\notin\{1,j\}$ is exponentially small; this leaves only the displayed $(1,j)$ interaction.

For all $t$ and $2\leq i\leq m$, by \cref{lm:stage2:psiprod} we have 
\[
\E_{x\sim i}\left[\psi_i(x) \psi_1(x)\right]\leq \frac{30}{\pigap d^{0.5-2\epsilon}}.
\]
Therefore,
\[
8n\sqrt{d}\exp\left(-d^{2\epsilon/3}/3\right)
+\frac{60\pi_i^*d^{2\epsilon}}{\pigap}
\leq 70\pi_i^*d^{2\epsilon}/\pigap
\qquad(2\leq i\leq m).
\]
The exponentially small term is at most $10\pimin^*d^{2\epsilon}/\pigap$.
This proves the representative estimate with its required factor $\pi_i^*$.
For component~1, use $\sum_{i=2}^m\pi_i^*\leq1$ in the displayed weighted sum to obtain
the bound $70d^{2\epsilon}/\pigap$.
When $t>10d^{-2\epsilon}/(\eta\pigap)$, by \cref{lm:stage2:larget}, for any $2\leq i\leq m$,
\[
\E_{x\sim i}\left[\psi_i(x) \psi_1(x)\right]\leq \E_{x\sim i}\left[\psi_i(x)\right]\leq 2\exp(-d^{2\epsilon/3}/3).
\]
Therefore, for such large $t$ we have
\[
8n\sqrt{d}\exp\left(-d^{2\epsilon/3}/3\right)
+4\sqrt d\exp\left(-d^{2\epsilon/3}/3\right)
\leq 10n\sqrt{d}\exp\left(-d^{2\epsilon/3}/3\right).
\]
The same estimate applies to the weighted sum in the component~$1$ bound, completing the proof.
\end{proof}

\begin{proof}[Proof of \cref{lm:stage2}]
We prove the invariants simultaneously by induction on $t$. They hold at $t=0$ by
\cref{thm:stage1}: the Stage~1 terminal bounds give the required projected and orthogonal estimates,
and $\lambda_j-\lambda_i\geq d^\epsilon$ for $j\in S_i\setminus\{i\}$.
For the first invariant, $\pi_1\geq\pi_1^*/n\geq1/(mn)\geq1/n^2$.
The initial full-distance gap follows from the orthogonal gap and
$\sqrt{r^2+p^2}-r\leq p^2/(2r)$, using $p\leq2n$ and $r\geq\sqrt d/2$. Assume now that all invariants
hold at time $t<T_2$. In accordance with the convention preceding the lemma, all unprimed
quantities below are evaluated at this same time $t$.

Recall that
\[
\nabla_{\mu_i}\calL=\sum_j \pi_j^*\E_{x\sim j}\left[\psi_i(x)\sum_k \psi_k(x)(\mu_k-\mu_j^*)\right].
\]

The induction hypotheses give
\begin{align*}
\|\mu_k-\mu_j^*\|
&\leq\|\mu_k-\parmu_k\|+\|\parmu_k-\mu_j^*\|\\
&\leq\sqrt d+2n+800/(\pigap)^2+\Delta+2n+800/(\pigap)^2
\leq2\sqrt d
\end{align*}
for every fitted index $k$ and true index $j$. (For $k=1$, use the sixth and ninth invariants.)
Consequently, the gradient formula, the triangle inequality, and $\sum_k\psi_k=1$ give
\begin{equation}
\label{eq:stage2:crude-gradient}
\begin{aligned}
\|\nabla_{\mu_i}\calL\|
&\leq\sum_j\pi_j^*\E_{x\sim j}\left[
\psi_i(x)\sum_k\psi_k(x)\|\mu_k-\mu_j^*\|\right]\\
&\leq2\sqrt d\sum_j\pi_j^*\E_{x\sim j}[\psi_i(x)]
\leq2\sqrt d.
\end{aligned}
\end{equation}
Let $R_j(v):=\|\mu_j^{(v)}-(\parmu_j)^{(v)}\|$. Since orthogonal projection is
non-expansive, the update and \cref{eq:stage2:crude-gradient} imply
\[
R_j(t+1)\geq R_j(t)-\eta\|\nabla_{\mu_j}\calL\|
\geq R_j(t)-2\eta\sqrt d.
\]
Iterating this inequality from time $0$ gives, whenever $t+1\leq d^{-2\epsilon}/(10\eta)$,
\[
R_j(t+1)
\geq R_j(0)-2\eta(t+1)\sqrt d
\geq\sqrt d-2d^{2\epsilon}-\frac15d^{0.5-2\epsilon}.
\]
Thus the new orthogonal lower bound is propagated to time $t+1$ before invoking
\cref{lm:stage2:differentpsi}; its time-$t$ instance, including the $j=1$ case, was already part of
the induction hypothesis.

We next verify the weight invariants using the exact update
\begin{equation}
\label{eq:stage2:weight-update}
\pi_k'=\sum_{q=1}^m\pi_q^*\E_{x\sim q}[\psi_k(x)].
\end{equation}
For source~1, \cref{lm:stage2:differentpsi,lm:stage2:similar} give
$\E_{x\sim1}\psi_1\geq1-2n\delta$, where $\delta=e^{-d^{2\epsilon/3}/3}$.
Thus $\pi_1'\geq\pi_1^*(1-2n\delta)\geq1/(2m)\geq1/n^2$.
This proves the all-time weight invariant without comparing two small weights.

In the early window, for a sample from true component $i$, all fitted components other than $i$
have responsibility at most $e^{-d^\epsilon}$ on the good event (the additional $j=1$ part of
\cref{lm:stage2:differentpsi} is used when $i\neq1$). Hence
\begin{align*}
\E_{x\sim i}[\psi_i(x)]
&\geq1-ne^{-d^\epsilon}-e^{-d^{2\epsilon/3}/3},\\
\pi_i'&\geq\pi_i^*\E_{x\sim i}[\psi_i(x)]
\geq\pi_i^*\bigl(1-2n e^{-d^{2\epsilon/3}/3}\bigr)
\geq\pi_i^*/2.
\end{align*}
This proves the first and seventh invariants at time $t+1$. Notice that the all-time first invariant
does not assert $\pi_i>\pi_i^*/n$ for $i\geq2$; such an assertion would be incompatible with the
terminal collapse of those weights. The other representatives need an absolute weight lower bound only during the initial interval,
where \cref{lm:stage2:early} supplies it.

For the \textbf{second} item, take a representative $2\leq i\leq m$ and put
$a_i:=\pi_i^*\E_{x\sim i}[\psi_i(x)^2]\in[0,\pi_i^*]$. The refined gradient estimate has the form
$\nabla_{\mu_i}\calL=a_i(\mu_i-\mu_i^*)+e_i$. Since $\eta a_i\in[0,1]$, the principal update
$(\mu_i-\mu_i^*)\mapsto(1-\eta a_i)(\mu_i-\mu_i^*)$ cannot increase either its full norm or the
norm of its projection onto $\calP^\perp$. The only possible increase comes from $e_i$.
Let $t_s:=10d^{-2\epsilon}/(\eta\pigap)$.
There are at most $\lfloor t_s\rfloor+1$ updates whose starting time is at most $t_s$.
Summing the coarse error over these updates and the sharper error over the remaining updates gives
\[
\eta(\lfloor t_s\rfloor+1)\frac{70d^{2\epsilon}}{\pigap}
+\eta T_2\,10n\sqrt d\,e^{-d^{2\epsilon/3}/3}
\leq\frac{700}{(\pigap)^2}+2
\leq\frac{800}{(\pigap)^2}.
\]
For a non-representative $i>m$, item~1 of \cref{lm:stage2:gradient} gives the one-step estimate
$\|\mu_i'-\mu_i\|\leq\eta e^{-d^{\epsilon/2}}$. Summing it proves both
\[
\|\mu_i^{(t)}-\mu_i^{(0)}\|\leq\eta t e^{-d^{\epsilon/2}}
\quad\text{and}\quad
\|\mu_i-\parmu_i\|\geq\lambda_i-\eta t e^{-d^{\epsilon/2}}.
\]
Combining these estimates with the Stage~1 input bounds proves every assertion in the second
invariant at time $t+1$.

For the \textbf{third} item, both kinds of representative have a principal update that cannot
increase their orthogonal radius. For $i\geq2$ this is a contraction toward $\mu_i^*$.
For $i=1$, write $a_j:=\pi_j^*\E_{x\sim j}[\psi_1(x)^2]$.
Since $\sum_j a_j\leq1$, its principal update is
\[
\mu_1\longmapsto\left(1-\eta\sum_j a_j\right)\mu_1+\eta\sum_j a_j\mu_j^*.
\]
The total gradient-error budget is at most $800/(\pigap)^2$ for each representative.
A non-representative moves by at most $\eta T_2e^{-d^{\epsilon/2}}$.
Consequently, for every $j\in S_i\setminus\{i\}$,
\[
\|\mu_j-\parmu_j\|-\|\mu_i-\parmu_i\|
\geq\lambda_j-\lambda_i-\eta T_2e^{-d^{\epsilon/2}}-800/(\pigap)^2
\geq d^\epsilon/2.
\]
For the full-distance gap, the principal update for $i=1$ satisfies
\begin{align*}
&\left\|\left(1-\eta\sum_j a_j\right)\mu_1
+\eta\sum_j a_j\mu_j^*-\mu_1^*\right\|\\
&\leq\left(1-\eta\sum_j a_j\right)\|\mu_1-\mu_1^*\|
+\eta\sum_j a_j\Delta\\
&\leq\max\{\|\mu_1-\mu_1^*\|,\Delta\}.
\end{align*}
Its initial full radius exceeds $\Delta$, so its radius is at most its initial value plus the
same accumulated error budget. The analogous statement for $i\geq2$ follows from contraction
toward $\mu_i^*$. Initially the full-distance gap is at least
$d^\epsilon-4n^2/\sqrt d$. Subtracting the two motion budgets gives a gap at least
$d^\epsilon/3$ for every group, including $S_1$.

Now we consider the \textbf{fourth} item.
For $2\leq i\leq m$, the second invariant and the standing dimension give
\[
\|\mu_i-\mu_i^*\|\leq \|\mu_i-\parmu_i\|+\|\parmu_i-\mu_i^*\|
\leq \sqrt{d}+10d^{2\epsilon}/\pigap.
\]
The principal term in \cref{lm:stage2:gradient} has norm at most
$\pi_i^*\|\mu_i-\mu_i^*\|$, because $\E_i[\psi_i^2]\leq1$. Adding the error bound $70\pi_i^*d^{2\epsilon}/\pigap$ gives
\[
\|\nabla_{\mu_i}\calL\|
\leq\pi_i^*\left(\sqrt d+80d^{2\epsilon}/\pigap\right).
\]
Thus, by non-expansiveness of orthogonal projection,
\begin{align*}
\|\mu_i'-\left(\parmu_i\right)'\|
&\geq\|\mu_i-\parmu_i\|-\eta\|\nabla_{\mu_i}\calL\|\\
&\geq\lambda_i-\pi_i^*\eta(t+1)
\left(\sqrt d+80d^{2\epsilon}/\pigap\right),
\end{align*}
which is precisely the fourth invariant at time $t+1$.

For the \textbf{fifth} item, first take $2\leq i\leq m$.
Through the integer time $\lfloor10d^{-2\epsilon}/(\eta\pigap)\rfloor+1$, the fourth item gives
\[
\|\mu_i-\parmu_i\|\geq\sqrt d-20d^{0.5-2\epsilon}/\pigap.
\]
For every update starting at $t\geq\lfloor t_s\rfloor+1$, \cref{lm:stage2:larget} bounds the coefficient of
the principal gradient by $2\pi_i^*\delta$. Together with the late error estimate, the entire
gradient is bounded by $12n\sqrt d\,\delta$. Its remaining accumulated norm is at most
$12\eta T_2n\sqrt d\,\delta\leq1$.
This proves the stated lower bound with constant $100$ for representatives.
For $i>m$, use the second invariant directly:
$\|\mu_i-\parmu_i\|\geq\sqrt d-2d^{2\epsilon}-\eta T_2e^{-d^{\epsilon/2}}$,
which is stronger than the required bound.

For the \textbf{sixth} item, we consider two cases for $t$.
Put
\[
r_1(t):=\|\mu_1^{(t)}-(\parmu_1)^{(t)}\|,
\qquad a:=1-\pi_1^*\eta+d^{-0.4}\eta,
\qquad \delta:=e^{-d^{2\epsilon/3}/3}.
\]
Projecting the principal term in the component-$1$ gradient onto $\calP^\perp$ multiplies the
orthogonal component by a nonnegative scalar. Retaining only the contribution of true component~$1$
and dropping the other nonnegative contraction coefficients therefore gives
\begin{equation}
\label{eq:stage2:r1-recurrence}
r_1(t+1)
\leq r_1(t)\left(1-\eta\pi_1^*\E_{x\sim1}[\psi_1(x)^2]\right)+\eta E_t,
\end{equation}
where $E_t=70d^{2\epsilon}/\pigap$ for all $t$ and
$E_t=10n\sqrt d\,\delta$ after $t_s=10d^{-2\epsilon}/(\eta\pigap)$.

For a source-$1$ sample, \cref{lm:stage2:differentpsi,lm:stage2:similar} and a good-event split
give $\sum_{k>1}\E_{x\sim1}[\psi_k(x)]\leq2n\delta$. Hence
\[
\E_{x\sim1}[\psi_1(x)]\geq1-2n\delta,
\qquad
\E_{x\sim1}[\psi_1(x)^2]
\geq\E_{x\sim1}[\psi_1(x)]^2
\geq1-4n\delta,
\]
where the middle inequality is Jensen's inequality.

Assume inductively that $r_1(t)\leq U_t:=\lambda_1a^t$. If $t\leq t_s$, then
$U_t\geq\lambda_1(1-\pi_1^*\eta t)\geq\sqrt d/2$. The standing dimension gives
\[
4n\pi_1^*\delta+\frac{70d^{2\epsilon}}{\pigap U_t}\leq d^{-0.4}.
\]
Substituting these inequalities into \cref{eq:stage2:r1-recurrence} yields
\[
r_1(t+1)
\leq U_t\left(1-\eta\pi_1^*+\eta d^{-0.4}\right)
=U_{t+1}.
\]
If $t>t_s$, then
$U_t\geq\lambda_1\exp(-2\pi_1^*\eta T_2)\geq1/(2d^2)$. In this range the sharper error estimate
and \cref{eq:standing-dimension} give
\[
4n\pi_1^*\delta+\frac{10n\sqrt d\,\delta}{U_t}\leq d^{-0.4},
\]
so the same substitution again gives $r_1(t+1)\leq U_{t+1}$. This proves the sixth invariant in
both time ranges.

Now we consider the $\textbf{eighth}$ item.
By \cref{lm:stage2:differentpsi}, when $t\leq d^{-2\epsilon}/(10\eta)$, for any $i\neq 1$,
\[
\E_{x\sim i}[\psi_1^2(x)]\leq 2\exp(-d^{2\epsilon/3}/3).
\]
Therefore,
\[
\left\|\nabla_{\mu_1}\calL-\pi_1^*\E_{x\sim 1}\left[\psi_1(x)^2(\mu_1-\mu_1^*)\right]\right\|\leq 80d^{2\epsilon}/\pigap
\]
Indeed, the contribution of every source $i\neq1$ is bounded by the displayed responsibility
estimate, and the remaining cross-fitted-index terms are covered by the third estimate of
\cref{lm:stage2:gradient}; their sum is at most the stated constant. After projection onto $\calP$,
the principal term is a nonnegative multiple of $\parmu_1-\mu_1^*$ and therefore cannot increase
$\|\parmu_1-\mu_1^*\|$. Telescoping only the error term over the early window gives
\[
\|\parmu_1-\mu_1^*\|
\leq3n+\eta\frac{d^{-2\epsilon}}{10\eta}
\frac{80d^{2\epsilon}}{\pigap}
\leq3n+10/\pigap.
\]
Together with the crude-gradient argument at the beginning of the induction step, this proves both
parts of the eighth invariant.

For the \textbf{ninth} item, again write
$a_j:=\pi_j^*\E_{x\sim j}[\psi_1(x)^2]$, so that $\sum_j a_j\leq1$.
For every true index $i$, the principal projected update is a convex combination, so
\begin{align*}
&\left\|\left(1-\eta\sum_j a_j\right)\parmu_1
+\eta\sum_j a_j\mu_j^*-\mu_i^*\right\|\\
&\leq\left(1-\eta\sum_j a_j\right)\|\parmu_1-\mu_i^*\|
+\eta\sum_j a_j\Delta\\
&\leq\max\{\|\parmu_1-\mu_i^*\|,\Delta\}.
\end{align*}
The initial distance is at most $\Delta+2n$. Summing the gradient errors over the
$\lfloor t_s\rfloor+1$ coarse updates and the subsequent fine updates, exactly as for the second item, gives
\[
\|\parmu_1-\mu_i^*\|\leq\Delta+2n+800/(\pigap)^2.
\]
This explicitly retains the Stage~1 projection buffer and proves the ninth invariant.

It remains to verify the two terminal claims. Since
$d^{-0.4}\leq\pi_1^*/5$ under \cref{eq:standing-dimension},
\[
\begin{aligned}
\|\mu_1-\parmu_1\|
&\leq(\sqrt d+2n)
\left(1-\pi_1^*\eta+d^{-0.4}\eta\right)^{T_2}\\
&\leq(\sqrt d+2n)
\exp\!\left(-0.8\pi_1^*\eta T_2\right)\\
&\leq(\sqrt d+2n)e^{-0.8\log d}\leq1.
\end{aligned}
\]
Finally fix $i\geq2$ in the weight update preceding time $T_2$. If $i\leq m$, then
\cref{lm:stage2:larget} controls its responsibility under source $i$; if $i>m$, then $i$ is not a
true-source index. Under every other source, \cref{lm:stage2:differentpsi} and
\cref{lm:stage2:similar} give the same bound. Hence, in either case,
\[
\pi_i
=\sum_{q=1}^m\pi_q^*\E_{x\sim q}[\psi_i(x)]
\leq2e^{-d^{2\epsilon/3}/3}.
\]
Since the fitted weights sum to one,
\[
\pi_1=1-\sum_{i=2}^n\pi_i
\geq1-2n e^{-d^{2\epsilon/3}/3}.
\]
This completes the induction and the proof of all terminal conclusions.
\end{proof}

\subsection{Stage 3 Analysis}
\label{subsec:stage3}
We reset $t=0$ at the output of Stage~2. The first Stage~3 update is treated explicitly because the
Stage~2 weight bound is $\exp(-d^{2\epsilon/3}/3)$, whereas the stronger Stage~3 geometry makes the
next EM weight update $\exp(-\Omega(d))$.
As in Stage~2, all unmarked parameters in the subsidiary lemmas refer to one fixed current time $t$,
and primes denote the parameters after the update to time $t+1$.

\begin{lemma}
\label{lm:stage3}
Given the initialization corresponding to \cref{thm:stage2}, for any $t\leq \exp(d/100)/\eta$, we always have
\begin{itemize}
\item $\|\mu_1-\parmu_1\|\leq 1$.
\item At $t=0$, $\pi_i\leq2\exp(-d^{2\epsilon/3}/3)$ for every $i\geq2$; for every $t\geq1$,
$\pi_i\leq 2\exp(-d/20)$.
\item For each $1\leq i\leq m$, $\|\parmu_1-\mu_i^*\|\leq d^{0.5-\epsilon}+2n+\frac{800}{(\pigap)^2}$.
\item For any $2\leq i\leq n$, 
$\sqrt{d}-200d^{0.5-2\epsilon}/\pigap\leq \|\mu_i-\parmu_i\|\leq 2\sqrt{d}$.
\item For any $2\leq i\leq n$ and $1\leq j\leq m$, $\|\parmu_i-\mu_j^*\|\leq 2d^{0.5-\epsilon}$.
\end{itemize}
\end{lemma}

\begin{lemma}
\label{lm:stage3:samplepi}
Given assumptions in \cref{lm:stage3}, for any $2\leq i\leq n, 1\leq j\leq m$, consider an $\bx\in \calP_0$.
Let $A=\|\bx-\mu_j^*\|$. Then 
\[
\psi_i(\bx)\le 
\begin{cases}
2\pi_i\exp\left(\sqrt{d}A-\dfrac d2+10d^{1-\epsilon}\right),
& \text{if } A\le \sqrt{d}-\dfrac{300d^{0.5-2\epsilon}}{\pigap},\\[8pt]
2\pi_i\exp\!\left(\dfrac{(A+2d^{0.5-\epsilon})^2}{2}\right), 
& \text{if } A\ge \sqrt{d}-\dfrac{300d^{0.5-2\epsilon}}{\pigap}.
\end{cases}
\]
\end{lemma}

\begin{proof}
The second weight invariant and $\sum_k\pi_k=1$ imply
$\pi_1=1-\sum_{k\geq2}\pi_k\geq1/2$ (at $t=0$ use its weaker displayed bound; for $t\geq1$
use the stronger one). Comparing only with component~$1$ in the responsibility denominator gives
\begin{equation}
\label{eq:stage3:responsibility-ratio}
\psi_i(\bx)
\leq\frac{\pi_i}{\pi_1}
\exp\!\left(\frac{\|\bx-\mu_1\|^2-\|\bx-\mu_i\|^2}{2}\right).
\end{equation}
The first and third invariants, together with \cref{eq:standing-dimension}, give
\[
\|\mu_1-\mu_j^*\|
\leq\|\mu_1-\parmu_1\|+\|\parmu_1-\mu_j^*\|
\leq2d^{0.5-\epsilon}.
\]
Hence $\|\bx-\mu_1\|\leq A+2d^{0.5-\epsilon}$. On the other hand, the fourth invariant implies
\[
\|\mu_i-\mu_j^*\|
\geq\|\mu_i-\parmu_i\|
\geq\sqrt d-B,
\qquad B:=\frac{200d^{0.5-2\epsilon}}{\pigap},
\]
and therefore $\|\bx-\mu_i\|\geq\sqrt d-B-A$ whenever the latter quantity is nonnegative.
For $A\leq\sqrt d-300d^{0.5-2\epsilon}/\pigap$, it is nonnegative, so expansion gives
\begin{align*}
&\frac12\left[(A+2d^{0.5-\epsilon})^2-(\sqrt d-B-A)^2\right]\\
&=\sqrt d\,A-\frac d2
+2Ad^{0.5-\epsilon}-AB+B\sqrt d-\frac{B^2}{2}+2d^{1-2\epsilon}\\
&\leq\sqrt d\,A-\frac d2+10d^{1-\epsilon}.
\end{align*}
The final inequality follows term by term from $A\leq\sqrt d$, $B\leq d^{0.5-\epsilon}$, and
\cref{eq:standing-dimension}. Substitution into \cref{eq:stage3:responsibility-ratio}, with
$1/\pi_1\leq2$, proves the first case.
If $A\geq \sqrt{d}-300d^{0.5-2\epsilon}/\pigap$, we simply have 
\[
\psi_i(\bx)
\leq2\pi_i\exp\!\left(\frac{\|\bx-\mu_1\|^2}{2}\right)
\leq2\pi_i\exp\!\left(\frac{(A+2d^{0.5-\epsilon})^2}{2}\right),
\]
where we only used $e^{-\|\bx-\mu_i\|^2/2}\leq1$. This proves the second case.
\end{proof}

\begin{lemma}
\label{lm:stage3:psismall}
Given assumptions in \cref{lm:stage3}, for $2\leq i\leq n$, we have $\E_{x\sim p^*}[\psi_i(x)]\leq 2\exp(-d/20)$.
\end{lemma}

\begin{proof}
Fix a true component $j$ and define
$G_j:=\{\|\bx-\mu_j^*\|\leq\sqrt d/3\}$. For sufficiently large $d$, the threshold
$\sqrt d/3$ is in the first case of \cref{lm:stage3:samplepi}. Since $\pi_i\leq1$, that lemma gives,
on $G_j$,
\begin{align*}
\psi_i(x)=\psi_i(\bx)
&\leq2\exp\!\left(\frac d3-\frac d2+10d^{1-\epsilon}\right)\\
&\leq e^{-d/7}.
\end{align*}
Here $\psi_i(x)=\psi_i(\bx)$ is the projection identity established in Stage~1.

Under $x\sim\calN(\mu_j^*,I_d)$, the vector $\bx-\mu_j^*$ is a standard Gaussian in the
$d_0:=\dim(\calP_0)\leq m+n$ dimensional space $\calP_0$. By \cref{lm:gausstail} and the standing
dimension,
\[
\Pr(G_j^c)
\leq\exp\!\left[-\frac12\left(\frac{\sqrt d}{3}-\sqrt{d_0}\right)^2\right]
\leq e^{-d/20}.
\]
Splitting the expectation and using $0\leq\psi_i\leq1$ gives
\[
\E_{x\sim j}[\psi_i(x)]
\leq e^{-d/7}\Pr(G_j)+\Pr(G_j^c)
\leq2e^{-d/20}.
\]
Finally average this estimate with weights $\pi_j^*$ over $1\leq j\leq m$.
\end{proof}

\begin{proof}[Proof of \cref{lm:stage3}]
We argue by induction. The geometric items hold at $t=0$ by \cref{thm:stage2}; the weaker initial
weight bound is also part of that lemma. It implies
$\pi_1=1-\sum_{i=2}^n\pi_i\geq1/2$ by \cref{eq:standing-dimension}. Assume the invariants hold at
time $t$, with all unprimed quantities evaluated at that time. Recall that
\[
\nabla_{\mu_i}\calL
=\sum_j \pi_j^*\E_{x\sim j}\left[
\psi_i(x)\sum_k \psi_k(x)(\mu_k-\mu_j^*)\right].
\]

For the \textbf{first} item, separate the $k=1$ term in the inner sum. The invariants give the
uniform diameter bound $\|\mu_k-\mu_j^*\|\leq3\sqrt d$, so
\begin{equation}
\label{eq:stage3:mu1}
\begin{aligned}
&\left\|\nabla_{\mu_1}\calL-
\sum_j \pi_j^*\E_{x\sim j}\left[\psi_1^2(x)(\mu_1-\mu_j^*)\right]\right\|\\
&\quad\leq3\sqrt d\sum_j\pi_j^*\E_{x\sim j}
\left[\psi_1(x)\sum_{i=2}^n\psi_i(x)\right]\\
&\quad\leq3\sqrt d\sum_{i=2}^n\E_{x\sim p^*}[\psi_i(x)]
\leq6n\sqrt d\,e^{-d/20}.
\end{aligned}
\end{equation}
The pointwise inequality $z^2\geq2z-1$ for $0\leq z\leq1$ and
\cref{lm:stage3:psismall} yield
\[
\begin{aligned}
\E_{x\sim p^*}[\psi_1^2(x)]
&\geq2\E_{x\sim p^*}[\psi_1(x)]-1\\
&=1-2\sum_{i=2}^n\E_{x\sim p^*}[\psi_i(x)]\\
&\geq1-4n e^{-d/20}\geq\frac12.
\end{aligned}
\]
Because every true mean lies in $\calP$, the projection of the principal term in
\cref{eq:stage3:mu1} onto $\calP^\perp$ equals
$\E_{p^*}[\psi_1^2](\mu_1-\parmu_1)$. Projecting the update therefore gives
\[
\left\|\mu_1'-\left(\parmu_1\right)'\right\|
\leq(1-\eta/2)\|\mu_1-\parmu_1\|
+6\eta n\sqrt d\,e^{-d/20}
\leq\max\{1,\|\mu_1-\parmu_1\|\}.
\]
The last inequality uses $6n\sqrt d\,e^{-d/20}\leq1/2$. Since the radius is at most one at time
$0$, this proves the first invariant at time $t+1$.

For the \textbf{second} item, the exact EM update and \cref{lm:stage3:psismall} give, for every
$i\geq2$,
\[
\pi_i'=\E_{x\sim p^*}[\psi_i(x)]\leq2e^{-d/20}.
\]
Thus the weaker Stage~2 bound is used only at $t=0$, and the stronger bound holds from time $1$
onward.

For the \textbf{third} item, write
$a_j=\pi_j^*\E_j\psi_1^2$ and $a=\sum_j a_j\in[1/2,1]$, as proved above.
Put $B=2n+800/(\pigap)^2$ and $r=\|\parmu_1-\mu_i^*\|\leq\Delta+B$.
The principal projected update is a convex combination of $\parmu_1$ and the true means.
Therefore \cref{eq:stage3:mu1} gives
\begin{align*}
\|\left(\parmu_1\right)'-\mu_i^*\|
&\leq(1-\eta a)r+\eta a\Delta+6\eta n\sqrt d\,e^{-d/20}\\
&\leq\Delta+B-\eta aB+6\eta n\sqrt d\,e^{-d/20}
\leq\Delta+B.
\end{align*}
The last step uses $a\geq1/2$ and $6n\sqrt d\,e^{-d/20}\leq B/2$.
This proves the required invariance without discarding a quadratic step-size term.

For the \textbf{fourth} and \textbf{fifth} items, the exact gradient formula and the uniform
$3\sqrt d$ diameter bound give, for $i\geq2$,
\[
\|\nabla_{\mu_i}\calL\|
\leq3\sqrt d\E_{x\sim p^*}[\psi_i(x)]
\leq6\sqrt d\,e^{-d/20}.
\]
Over at most $\exp(d/100)/\eta$ updates, the total motion is therefore at most
\[
6\sqrt d\,e^{-d/20+d/100}=6\sqrt d\,e^{-d/25}.
\]
At time zero, \cref{thm:stage2} gives the lower orthogonal radius
$\sqrt d-100d^{0.5-2\epsilon}/\pigap$, and its upper and parallel bounds are separated from the
Stage~3 limits by polynomial margins. The exponentially small total motion is below each margin,
so non-expansiveness of the relevant projections proves the fourth and fifth invariants at time
$t+1$. This closes the induction.
\end{proof}

\begin{lemma}
\label{lm:stage3:loss}
If $\|\mu_j-\mu_i^*\|\geq \sqrt{d}/2$ for every student index $j\geq2$ and every true index
$i\in[m]$, then
\[
\calL\geq \frac{(1-\pi_1^*)^2}{3}.
\]
\end{lemma}

\begin{proof}
Let $\Delta=d^{0.5-\epsilon}$ and let $\bx$ denote the projection of $x$ onto $\calP_0$.
Because the true means are pairwise $\Delta$ apart, at most one of them can lie within
$\Delta/2.1$ of $\mu_1$. Let $k$ be this index if it exists, and choose an arbitrary $k$ otherwise.
Define the measurable event
\[
X:=\bigcup_{\ell\neq k}\left\{x:\|\bx-\mu_\ell^*\|\leq\Delta/5\right\}.
\]

If $x\sim\ell$ with $\ell\neq k$, Gaussian concentration in the at-most $(m+n)$-dimensional
subspace $\calP_0$ gives
\[
\Pr[x\notin X\mid x\sim\ell]
\leq\Pr\!\left[\|\bx-\mu_\ell^*\|>\Delta/5\right]
\leq\exp(-\Delta^2/100).
\]
For the second inequality, write $\bx-\mu_\ell^*$ as a standard Gaussian in
$d_0\leq m+n$ dimensions and apply \cref{lm:gausstail}; the standing dimension ensures that
$\Delta/5-\sqrt{d_0}\geq\Delta/6$.
Since $\pi_k^*\leq\pi_1^*$, it follows that
\begin{equation}
\label{eq:stage3:true-event}
p^*(X)\geq(1-\pi_1^*)\bigl(1-\exp(-\Delta^2/100)\bigr).
\end{equation}

We next bound the same event under the fitted mixture. We use the following elementary consequence
of the same Gaussian tail estimate: if $Y\sim\calN(\nu,I_d)$ and
$\|P_0\nu-c\|\geq R+s$ with $s\geq\sqrt{d_0}$, then
\begin{equation}
\label{eq:stage3:ball-entry}
\Pr[\|P_0Y-c\|\leq R]
\leq\Pr[\|P_0(Y-\nu)\|\geq s]
\leq\exp\!\left[-\frac12(s-\sqrt{d_0})^2\right].
\end{equation}
For the first fitted component,
$\|\mu_1-\mu_\ell^*\|\geq\Delta-\Delta/2.1>\Delta/2$ for every $\ell\neq k$ if an exceptional index
exists; if none exists, the distance is at least $\Delta/2.1$. Either bound leaves a constant multiple
of $\Delta$ between $\mu_1$ and each radius-$\Delta/5$ ball. Applying
\cref{eq:stage3:ball-entry} and a union bound over at most $m-1$ balls shows that a sample from
$\calN(\mu_1,I_d)$ enters $X$ with probability at most $m e^{-\Delta^2/100}$.

For every fitted component $j\geq2$, the hypothesis gives
$\|\mu_j-\mu_\ell^*\|\geq\sqrt d/2$. Since
$\Delta/5+\sqrt{d_0}\leq\sqrt d/3$, \cref{eq:stage3:ball-entry} and the same union bound give entry
probability at most $m e^{-d/100}$. Averaging these bounds over the fitted weights yields
\begin{equation}
\label{eq:stage3:fitted-event}
p(X)\leq m\exp(-\Delta^2/100)+m\exp(-d/100).
\end{equation}

By \cref{eq:standing-dimension}, the right-hand side of \cref{eq:stage3:fitted-event} and the error
in \cref{eq:stage3:true-event} are together at most $(1-\pi_1^*)/10$. Therefore
$|p^*(X)-p(X)|\geq 2(1-\pi_1^*)/3$. Pinsker's inequality now gives
\[
\calL=D_{\rm KL}(p^*\|p)\geq2\,|p^*(X)-p(X)|^2
\geq\frac{(1-\pi_1^*)^2}{3}.
\]

\end{proof}

\begin{proof}[Proof of \cref{thm:stage3}]
For every $j\geq2$ and every true index $i$, the fourth invariant of \cref{lm:stage3} gives
\[
\|\mu_j-\mu_i^*\|
\geq\|\mu_j-\parmu_j\|
\geq\sqrt d-200d^{0.5-2\epsilon}/\pigap
\geq\sqrt d/2.
\]
Thus the hypothesis of \cref{lm:stage3:loss} holds at every Stage~3 iteration up to
$\exp(d/100)/\eta$. Applying that lemma at each such iteration proves the lemma.
\end{proof}

\begin{proof}[Proof of \cref{thm:main}]
Let $T_1=\lceil1000n^{10}d^{\epsilon-0.5}/\eta\rceil$ and
$T_2=\lceil\log d/(\pi_1^*\eta)\rceil$. During Stages~1 and~2,
\cref{lm:stage1,lm:stage2} imply
$\|\mu_j-\mu_i^*\|\geq\sqrt d/2$ for every $j\geq2$ and $i\in[m]$: in Stage~1 this follows
from the orthogonal-radius invariants, and in Stage~2 it follows from the fifth invariant for
representatives and the non-representative lower bound in the second invariant. Hence
\cref{lm:stage3:loss} gives the desired loss bound for all global iterations up to $T_1+T_2$.
At global time $T_1+T_2$, the conclusions of \cref{thm:stage2} are exactly the initialization
hypotheses of \cref{thm:stage3}. Therefore that lemma applies to all subsequent iterations. A
Stage~3 local time is at most the corresponding global time, so its validity through
$\exp(d/100)/\eta$ covers every global time asserted by \cref{thm:main}.
\end{proof}

\section{Proof of \cref{thm:positive}}
\label{sec:positive:proof}
Throughout this appendix we assume the hypotheses of \cref{thm:positive}, including
$\|\mu_i^*-\mu_j^*\|\geq10\sqrt d$ for distinct true components. The proof has two
stages. We first apply the Stage~1 lemma with $\epsilon=1/100$ and convert its
orthogonal-radius bounds into bounds on full distances. We then prove convergence
under the stronger separation assumption.

We record some consequences of the dimension and learning-rate assumptions:
\begin{equation}
\label{eq:positive:parameters}
m\leq n,\qquad
\frac{n}{\pimin^*}\leq\frac{n}{\rho}
\leq\frac{d^{1/1500}}{100},\qquad 0<\eta\leq1.
\end{equation}
Here $\rho=\min\{\pimin^*,\pi_{\rm gap}^*\}$. We take the
universal constant $C_0$ in that theorem sufficiently large for the numerical
comparisons below.

\begin{lemma}
\label{lm:positive:stage1}
With probability at least $1-1/m$, after permuting the fitted-component indices
and setting
\[
T_1:=\left\lceil1000n^{10}d^{-0.49}/\eta\right\rceil,
\]
the parameters at time $T_1$ satisfy the following properties:
\begin{itemize}
\item For $1\leq i\leq m$, $\pi_i>\pi_i^*/n$.
\item For each $1\leq i\leq m$, the group $S_i$ is nonempty and contains its
representative $i$. Here $S_i$ consists of the fitted indices initialized from
$\calN(\mu_i^*,I_d)$.
\item For $i\in[m]$ and $j\in S_i\setminus\{i\}$,
$\|\mu_i-\mu_i^*\|\leq\|\mu_j-\mu_i^*\|-d^{0.01}/2$.
\item For $i\in[m]$ and $j\in S_i$,
$\sqrt d-d^{0.03}\leq\|\mu_j-\mu_i^*\|\leq\sqrt d+d^{0.03}$.
\end{itemize}
\end{lemma}

\begin{proof}
Condition on the initialization event in \cref{thm:init}, which has probability
at least $1-1/m$. Every $S_i$ is then nonempty. Use the representative labels from
that lemma: $i$ minimizes the initial orthogonal radius within $S_i$.

% The supplied source fragment does not include the LaTeX label of Lemma 7.
% Replace the literal reference below by its existing \cref{...}, if desired.
Apply the Stage~1 lemma (Lemma~7) with $\epsilon=1/100$. Its dimension and
learning-rate conditions are precisely those assumed here, and its separation
condition holds because $10\sqrt d\geq d^{0.49}$. This application uses only a
lower bound on true-center separation. At time $T_1$, the Stage~1 lemma gives
$\pi_i>\pi_i^*/n$ and, for every $j\in S_i$,
\begin{gather}
\|\parmu_j-\mu_i^*\|\leq2n,\qquad
\sqrt d-2d^{0.02}\leq\|\mu_j-\parmu_j\|\leq\sqrt d+2n,
\label{eq:positive:stage1-radii}\\
\|\mu_j-\parmu_j\|-\|\mu_i-\parmu_i\|\geq d^{0.01}
\qquad\text{if }j\neq i.
\label{eq:positive:stage1-orthogonal-gap}
\end{gather}

Fix $i\in[m]$. For $j\in S_i$, write
\[
r_j:=\|\mu_j-\parmu_j\|,\qquad
q_j:=\|\parmu_j-\mu_i^*\|,\qquad
R_j:=\|\mu_j-\mu_i^*\|.
\]
Since $\mu_i^*\in\calP$ and $\parmu_j$ is the orthogonal projection onto
$\calP$, we have $R_j^2=r_j^2+q_j^2$. In particular, $r_j\geq\sqrt d/2$ and
\begin{equation}
\label{eq:positive:radial-conversion}
0\leq R_j-r_j
=\frac{q_j^2}{R_j+r_j}
\leq\frac{q_j^2}{2r_j}
\leq\frac{4n^2}{\sqrt d}.
\end{equation}
The dimension assumption implies
\[
2d^{0.02}\leq d^{0.03},\qquad
2n+\frac{4n^2}{\sqrt d}\leq d^{0.03},\qquad
\frac{4n^2}{\sqrt d}\leq\frac{d^{0.01}}2.
\]
Therefore \cref{eq:positive:stage1-radii,eq:positive:radial-conversion} give
\[
\sqrt d-d^{0.03}\leq R_j\leq\sqrt d+d^{0.03}.
\]
For $j\in S_i\setminus\{i\}$, they also give
\[
R_j-R_i\geq r_j-r_i-\frac{4n^2}{\sqrt d}
\geq d^{0.01}-\frac{4n^2}{\sqrt d}
\geq\frac{d^{0.01}}2.
\]
These are the remaining two conclusions. The entire argument is conditional on
the same initialization event, so there is no additional probability loss.
\end{proof}

We reset the local iteration counter to zero at the end of the first stage.
The following lemma describes the second stage.

\begin{lemma}
\label{lm:stage5}
Start from parameters satisfying the four conclusions of
\cref{lm:positive:stage1}, with the groups and representative labels fixed there.
For every integer $0\leq t\leq\exp(d/50)/\eta$, the following hold:
\begin{itemize}
\item For $1\leq i\leq m$, $\pi_i>\pi_i^*/n$.
\item For $i\in[m]$ and $j\in S_i$,
$\|\mu_j-\mu_i^*\|\leq\sqrt d+2d^{0.03}$.
\item For $i\in[m]$ and $j\in S_i\setminus\{i\}$,
$\|\mu_i-\mu_i^*\|\leq\|\mu_j-\mu_i^*\|-d^{0.01}/2$ and
$\|\mu_j-\mu_i^*\|\geq\sqrt d-2d^{0.03}$.
\item If $t\geq8d/(\eta\pimin^*)$, then
$\|\mu_j-\mu_j^*\|\leq\exp(-d/50)$ for every $1\leq j\leq m$.
\item If $t\geq8d/(\eta\pimin^*)+1$, then
$|\pi_i-\pi_i^*|\leq4n\exp(-d/40)$ for every $1\leq i\leq m$.
\end{itemize}
\end{lemma}

For the second stage, let $\calP_0$ be the fixed span of the true means and the
fitted means at local time zero. The gradient formula expresses every mean
update as a linear combination of current fitted means and true means, so all
subsequent means remain in $\calP_0$. Thus $d_0:=\dim(\calP_0)\leq m+n$.
Write $\bx$ for the orthogonal projection of $x$ onto $\calP_0$. The common
factor contributed by $x-\bx$ cancels in the responsibilities, so they depend
only on $\bx$. Conditional on source $i$, $\bx-\mu_i^*$ is a standard Gaussian
on this $d_0$-dimensional subspace. At a fixed iteration, unprimed parameters
refer to the current time and primed parameters refer to the next time.

The next three lemmas assume only that the first three conclusions of
\cref{lm:stage5} hold at the current iteration. They will be used to prove that
these properties are preserved.

\begin{lemma}
\label{lm:stage5:far}
Suppose the first three conclusions of \cref{lm:stage5} hold at the current
iteration. For $i\in[m]$ and $j\in[n]\setminus\{i\}$,
\[
\Pr_{x\sim i}[\psi_j(x)\geq \exp(-d^{0.01})]\leq \exp(-d^{0.01}).
\]
\end{lemma}

\begin{proof}
Let $q=d^{0.01}/5$ and $G_i=\{\|\bx-\mu_i^*\|\leq q\}$.
The dimension assumption gives $\sqrt{d_0}\leq d^{0.01}/10$ and
$d^{0.02}/200\geq d^{0.01}$. Hence \cref{lm:gausstail} yields
\[
\Pr_{x\sim i}(G_i^c)
\leq\exp\!\left(-\frac12(q-\sqrt{d_0})^2\right)
\leq \exp(-d^{0.02}/200)\leq \exp(-d^{0.01}).
\]

Put $a=\|\mu_j-\mu_i^*\|$ and $b=\|\mu_i-\mu_i^*\|$.
If $j\in S_i\setminus\{i\}$, the third conclusion gives
$a-b\geq d^{0.01}/2$. If $j\notin S_i$, the true separation and the second
conclusion give
\[
a-b\geq10\sqrt d-2(\sqrt d+2d^{0.03})
=8\sqrt d-4d^{0.03}\geq7\sqrt d\geq d^{0.01}/2.
\]
Thus in both cases $a\geq d^{0.01}/2>q$. For $x\in G_i$, the triangle
inequalities and the nonnegativity of $a-q$ give
\begin{align*}
\|\bx-\mu_j\|^2-\|\bx-\mu_i\|^2
&\geq(a-q)^2-(b+q)^2\\
&=(a-b-2q)(a+b)\\
&\geq\frac{d^{0.01}}{10}\cdot\frac{d^{0.01}}2
=\frac{d^{0.02}}{20}.
\end{align*}
Comparing with component $i$ in the responsibility denominator gives
\[
\psi_j(x)
\leq\frac{\pi_j}{\pi_i}
\exp\!\left(-\frac{\|\bx-\mu_j\|^2-\|\bx-\mu_i\|^2}{2}\right)
\leq\frac{n}{\pi_i^*}\exp(-d^{0.02}/40)
<\exp(-d^{0.01}).
\]
The last strict inequality follows from
$\log(n/\pimin^*)<d^{0.02}/40-d^{0.01}$, using
\cref{eq:positive:parameters}. Consequently,
$\{\psi_j\geq \exp(-d^{0.01})\}\subseteq G_i^c$, and hence
\[
\Pr_{x\sim i}[\psi_j\geq\exp(-d^{0.01})]
\leq\Pr_{x\sim i}(G_i^c)
\leq\exp(-d^{0.01}).
\]
The pointwise bound holds everywhere on $G_i$, so there is only one exceptional
probability to bound. The Gaussian upper-tail inequality used above is one-sided.
\end{proof}

\begin{lemma}
\label{lm:stage5:far:muismall}
Suppose the first three conclusions of \cref{lm:stage5} hold at the current
iteration. If $i\in[m]$ satisfies $\|\mu_i-\mu_i^*\|\leq1$, then for every
$j\in[n]\setminus\{i\}$,
\[
\Pr_{x\sim i}[\psi_j(x)\geq \exp(-d/20)]\leq \exp(-d/40).
\]
\end{lemma}

\begin{proof}
Let $G_i=\{\|\bx-\mu_i^*\|\leq\sqrt d/4\}$.
Since $\sqrt{d_0}\leq\sqrt d/100$, Gaussian concentration gives
\[
\Pr_{x\sim i}(G_i^c)
\leq\exp\!\left(-\frac12(\sqrt d/4-\sqrt{d_0})^2\right)
\leq \exp(-18d/625)\leq \exp(-d/40).
\]
Put $a=\|\mu_j-\mu_i^*\|$ and $b=\|\mu_i-\mu_i^*\|\leq1$.
For $j\in S_i\setminus\{i\}$, the third conclusion gives
$a\geq\sqrt d-2d^{0.03}$. For $j\notin S_i$, true separation and the second
conclusion give $a\geq9\sqrt d-2d^{0.03}$, so the same lower bound holds.
In particular,
\[
a-b-\frac{\sqrt d}{2}
\geq\frac{\sqrt d}{2}-2d^{0.03}-1
\geq\frac{\sqrt d}{3},\qquad
a+b\geq\frac{\sqrt d}{2}.
\]
Also $a-\sqrt d/4\geq0$. Thus, on $G_i$,
\begin{align*}
\|\bx-\mu_j\|^2-\|\bx-\mu_i\|^2
&\geq(a-\sqrt d/4)^2-(b+\sqrt d/4)^2\\
&=(a-b-\sqrt d/2)(a+b)\geq d/6.
\end{align*}
Using $\pi_i>\pi_i^*/n$, $\pi_j\leq1$, and
$\log(n/\pimin^*)<d/30$, we obtain
\[
\psi_j(x)\leq\frac{n}{\pi_i^*}\exp(-d/12)<\exp(-d/20)
\qquad(x\in G_i).
\]
The exceptional responsibility event is therefore contained in $G_i^c$.
\end{proof}

\begin{lemma}
\label{lm:stage5:extremefar}
Suppose the first three conclusions of \cref{lm:stage5} hold at the current
iteration. For $i\in[m]$ and $j\notin S_i$, put
$s=\|\mu_j-\mu_i^*\|$. Then $s\geq8\sqrt d$ and
\[
\Pr_{x\sim i}[s\psi_j(x)\geq \exp(-s^2/10)]
\leq s^{-1}\exp(-s^2/20).
\]
\end{lemma}

\begin{proof}
True separation and the second conclusion give
$s\geq9\sqrt d-2d^{0.03}\geq8\sqrt d$.
For all such $s$, the standing dimension implies
\begin{equation}
\label{eq:positive:far-comparisons}
\sqrt{d_0}\leq s/120,\qquad
\log s\leq\frac{9s^2}{3200},\qquad
2\log\frac{ns}{\pimin^*}<\frac{13s^2}{735}.
\end{equation}
For completeness, the logarithmic inequalities hold at $s=8\sqrt d$ by
\cref{eq:positive:parameters}. The differences between their quadratic
right-hand sides and logarithmic left-hand sides are increasing for
$s\geq8\sqrt d$, which proves their uniform validity without an upper bound
on $s$.

Let $G_i(s)=\{\|\bx-\mu_i^*\|\leq s/3\}$. By Gaussian concentration and
\cref{eq:positive:far-comparisons},
\[
\Pr_{x\sim i}(G_i(s)^c)
\leq\exp\!\left(-\frac12(s/3-\sqrt{d_0})^2\right)
\leq \exp(-169s^2/3200)
\leq s^{-1}\exp(-s^2/20).
\]
The representative satisfies
$\|\mu_i-\mu_i^*\|\leq\sqrt d+2d^{0.03}\leq s/7$.
Consequently, on $G_i(s)$,
\begin{align*}
\|\bx-\mu_j\|^2-\|\bx-\mu_i\|^2
&\geq(2s/3)^2-(s/3+s/7)^2\\
&=\frac{96}{441}s^2\\
&>\frac{s^2}{5}+2\log\frac{ns}{\pi_i^*}.
\end{align*}
Here $96/441-1/5=13/735$, so the last step follows from
\cref{eq:positive:far-comparisons}. Comparing with component $i$ in the
responsibility denominator gives
\[
\psi_j(x)
<\frac{n}{\pi_i^*}
\exp\!\left(-\frac{s^2}{10}-\log\frac{ns}{\pi_i^*}\right)
=s^{-1}\exp(-s^2/10).
\]
Thus $\{s\psi_j\geq \exp(-s^2/10)\}\subseteq G_i(s)^c$, as required.
\end{proof}

\begin{proof}[Proof of \cref{lm:stage5}]
We prove the first three items by induction on the integer time $t$. They hold
at $t=0$ by \cref{lm:positive:stage1}. Assume they hold through time $t$, where
$t+1\leq\exp(d/50)/\eta$. Recall the gradient form
\[
\nabla_{\mu_j}\calL
=\sum_i\pi_i^*\E_{x\sim i}
\left[\psi_j(x)\sum_k\psi_k(x)(\mu_k-\mu_i^*)\right].
\]
We first consider the contribution of each $i,k$. If $k\notin S_i$, then
$s:=\|\mu_k-\mu_i^*\|\geq8\sqrt d$, and \cref{lm:stage5:extremefar} gives
\begin{align*}
\E_{x\sim i}[\psi_k(x)\|\mu_k-\mu_i^*\|]
&\leq\exp(-s^2/10)
+s\Pr_{x\sim i}[s\psi_k(x)\geq\exp(-s^2/10)]\\
&\leq2\exp(-s^2/20)\leq\exp(-d).
\end{align*}
If $k\in S_i$, then $\|\mu_k-\mu_i^*\|\leq2\sqrt d$. For $j\neq i$,
since $0\leq\psi_j(x)\leq1$, \cref{lm:stage5:far} gives
\begin{align*}
\E_{x\sim i}[\psi_j(x)]
&\leq\exp(-d^{0.01})
+\Pr_{x\sim i}[\psi_j(x)\geq\exp(-d^{0.01})]\\
&\leq2\exp(-d^{0.01}).
\end{align*}
Consequently,
\[
\E_{x\sim i}[\psi_j(x)\|\mu_k-\mu_i^*\|]
\leq4\sqrt d\,\exp(-d^{0.01}).
\]
If also $\|\mu_i-\mu_i^*\|\leq1$, then \cref{lm:stage5:far:muismall} gives
\begin{align*}
\E_{x\sim i}[\psi_j(x)\|\mu_k-\mu_i^*\|]
&\leq2\sqrt d\left(\exp(-d/20)
+\Pr_{x\sim i}[\psi_j(x)\geq\exp(-d/20)]\right)\\
&\leq4\sqrt d\,\exp(-d/40).
\end{align*}

Therefore, for any $j>m$, summing over $k$ and averaging over the true weights
gives
\[
\|\nabla_{\mu_j}\calL\|\leq4n\sqrt d\,\exp(-d^{0.01}).
\]
Moreover, if $\|\mu_i-\mu_i^*\|\leq1$ for every $1\leq i\leq m$, then
\[
\|\nabla_{\mu_j}\calL\|\leq4n\sqrt d\,\exp(-d/40)
\qquad(j>m).
\]
For a representative $1\leq j\leq m$, all terms in the gradient except
$i=k=j$ have total norm at most $4n\sqrt d\,\exp(-d^{0.01})$.
Indeed, when $i=j$ and $k\in S_j\setminus\{j\}$, use the bound on
$\E_{x\sim j}[\psi_k(x)]$; all other terms are covered by the preceding two
cases. Using $1-\psi_j(x)^2=(1+\psi_j(x))\sum_{l\neq j}\psi_l(x)$, we obtain
\begin{align}
&\|\nabla_{\mu_j}\calL-\pi_j^*(\mu_j-\mu_j^*)\|\notag\\
&\leq\pi_j^*\|\mu_j-\mu_j^*\|\E_{x\sim j}[1-\psi_j(x)^2]
+4n\sqrt d\,\exp(-d^{0.01})\notag\\
&\leq4\sqrt d\sum_{l\neq j}\E_{x\sim j}[\psi_l(x)]
+4n\sqrt d\,\exp(-d^{0.01})\notag\\
&\leq12n\sqrt d\,\exp(-d^{0.01}).
\label{eq:positive:coarse-gradient}
\end{align}
If $\|\mu_i-\mu_i^*\|\leq1$ for every $1\leq i\leq m$, the same calculation
using \cref{lm:stage5:far:muismall} gives
\begin{equation}
\label{eq:positive:fine-gradient}
\|\nabla_{\mu_j}\calL-\pi_j^*(\mu_j-\mu_j^*)\|
\leq12n\sqrt d\,\exp(-d/40).
\end{equation}

The dimension assumption implies
\begin{equation}
\label{eq:positive:error-sizes}
12n\sqrt d\,\exp(-d^{0.01})\leq\frac{\pimin^*}{20},\qquad
12n\sqrt d\,\exp(-d/40)\leq\frac{\pimin^*}{2}\exp(-d/50).
\end{equation}
By \cref{eq:positive:coarse-gradient}, for each $1\leq j\leq m$,
\begin{equation}
\label{eq:positive:coarse-recurrence}
\|\mu_j'-\mu_j^*\|
\leq(1-\eta\pi_j^*)\|\mu_j-\mu_j^*\|
+12\eta n\sqrt d\,\exp(-d^{0.01}).
\end{equation}
Thus, when $\|\mu_j-\mu_j^*\|\geq0.1$,
\[
\|\mu_j'-\mu_j^*\|
\leq(1-\eta\pi_j^*/2)\|\mu_j-\mu_j^*\|.
\]
Also, \cref{eq:positive:coarse-recurrence,eq:positive:error-sizes} show that
$\|\mu_j-\mu_j^*\|\leq1$ implies $\|\mu_j'-\mu_j^*\|\leq1$.
Starting from $\|\mu_j^{(0)}-\mu_j^*\|\leq2\sqrt d$, these bounds give
\begin{equation}
\label{eq:positive:unit-time}
\|\mu_j^{(u)}-\mu_j^*\|\leq1
\quad\text{whenever}\quad
\left\lceil\frac{10\log d}{\eta\pimin^*}\right\rceil\leq u\leq t+1.
\end{equation}
Indeed, before entering the unit ball the distance contracts at every step;
after the displayed number of steps its bound is at most
$2\sqrt d\,\exp(-5\log d)\leq1$. Once it enters the unit ball, it remains
there. This argument uses only updates through the current time $t$.
In particular, when $t\geq\lceil10\log d/(\eta\pimin^*)\rceil$, the stronger
gradient bounds apply to the current update.

We now verify the first three items at time $t+1$. For the first item,
\cref{lm:stage5:far} and the exact weight update give
\[
\pi_i'
\geq\pi_i^*\E_{x\sim i}[\psi_i(x)]
\geq\pi_i^*\left(1-2n\exp(-d^{0.01})\right)>\pi_i^*/n.
\]
For the second item, the principal update of a representative contracts toward
its true mean, while the entire gradient of a non-representative has the bounds
above. Therefore the total possible increase of $\|\mu_j-\mu_i^*\|$, for
$j\in S_i$, through time $t+1$ is at most
\begin{align}
&12\eta n\sqrt d\,\exp(-d^{0.01})
\left\lceil\frac{10\log d}{\eta\pimin^*}\right\rceil
+\frac{\exp(d/50)}{\eta}\,12\eta n\sqrt d\,\exp(-d/40)\notag\\
&\qquad\leq d^{0.03}.
\label{eq:positive:motion-budget}
\end{align}
Here the first term counts at most the displayed number of initial updates;
every later update has the stronger bound by \cref{eq:positive:unit-time}.
The last inequality follows from \cref{eq:positive:parameters} and
$\lceil x\rceil\leq x+1$. Combining this with \cref{lm:positive:stage1} gives
$\|\mu_j'-\mu_i^*\|\leq\sqrt d+2d^{0.03}$.

For $j\in S_i\setminus\{i\}$, the same sum with $4$ in place of $12$ bounds
the total displacement of $\mu_j$. Hence
$\|\mu_j'-\mu_i^*\|\geq\sqrt d-2d^{0.03}$.
If $\|\mu_i-\mu_i^*\|\leq1$, we also have $\|\mu_i'-\mu_i^*\|\leq1$, and
\[
\|\mu_j'-\mu_i^*\|-\|\mu_i'-\mu_i^*\|
\geq\sqrt d-2d^{0.03}-1\geq d^{0.01}/2.
\]
If $\|\mu_i-\mu_i^*\|>1$, then its distance decreases by at least
$\eta\pi_i^*/2$, whereas the non-representative distance decreases by at most
$4\eta n\sqrt d\,\exp(-d^{0.01})$. Consequently,
\begin{align*}
&\bigl(\|\mu_j'-\mu_i^*\|-\|\mu_i'-\mu_i^*\|\bigr)
-\bigl(\|\mu_j-\mu_i^*\|-\|\mu_i-\mu_i^*\|\bigr)\\
&\qquad\geq\eta\left(\frac{\pi_i^*}{2}
-4n\sqrt d\,\exp(-d^{0.01})\right)\geq0.
\end{align*}
Thus the gap is preserved. This proves the third item and closes the induction
for the first three items throughout the stated time interval.

For the fourth item, when $t\geq\lceil10\log d/(\eta\pimin^*)\rceil$, all
representatives are within one of their true means. By
\cref{eq:positive:fine-gradient},
\begin{equation}
\label{eq:positive:fine-recurrence}
\|\mu_j'-\mu_j^*\|
\leq(1-\eta\pi_j^*)\|\mu_j-\mu_j^*\|
+12\eta n\sqrt d\,\exp(-d/40).
\end{equation}
If $\|\mu_j-\mu_j^*\|\geq\exp(-d/50)$,
\cref{eq:positive:error-sizes} therefore implies
\[
\|\mu_j'-\mu_j^*\|
\leq(1-\eta\pi_j^*/2)\|\mu_j-\mu_j^*\|.
\]
If $\|\mu_j-\mu_j^*\|\leq\exp(-d/50)$, the same recurrence gives
$\|\mu_j'-\mu_j^*\|\leq\exp(-d/50)$.
Starting from the unit bound, at most $\lceil d/(20\eta\pimin^*)\rceil$
further steps suffice, since contraction over that many steps gives a bound
of $\exp(-d/40)\leq\exp(-d/50)$ unless the smaller ball has already been
reached. Finally,
\[
\left\lceil\frac{10\log d}{\eta\pimin^*}\right\rceil
+\left\lceil\frac{d}{20\eta\pimin^*}\right\rceil
\leq\frac{10\log d+d/20}{\eta\pimin^*}+2
\leq\frac{8d}{\eta\pimin^*}.
\]
This proves the fourth item, including persistence of the bound at later times
in the stated interval.

For the fifth item, fix an integer $t\geq8d/(\eta\pimin^*)$ with
$t+1\leq\exp(d/50)/\eta$. By \cref{lm:stage5:far:muismall},
\begin{equation}
\label{eq:positive:own-responsibility}
\E_{x\sim i}[1-\psi_i(x)]
=\sum_{j\neq i}\E_{x\sim i}[\psi_j(x)]
\leq2n\exp(-d/40).
\end{equation}
For $k\neq i$, we have $i\notin S_k$ and
$\|\mu_i-\mu_k^*\|\geq8\sqrt d\geq1$. The distance-weighted estimate at the
beginning of this proof therefore gives
\begin{equation}
\label{eq:positive:cross-responsibility}
\E_{x\sim k}[\psi_i(x)]
\leq\E_{x\sim k}[\psi_i(x)\|\mu_i-\mu_k^*\|]
\leq\exp(-d).
\end{equation}
Using the exact weight update,
\begin{align*}
|\pi_i'-\pi_i^*|
&=\left|-\pi_i^*\E_{x\sim i}[1-\psi_i(x)]
+\sum_{k\neq i}\pi_k^*\E_{x\sim k}[\psi_i(x)]\right|\\
&\leq2n\exp(-d/40)+\exp(-d)\leq4n\exp(-d/40).
\end{align*}
This is the required bound at time $t+1$, proving the fifth item with its
one-step offset.
\end{proof}

\begin{lemma}
\label{lm:last:loss}
Let $0<\delta\leq \exp(-d/100)$. If, for every $1\leq i\leq m$,
\[
|\pi_i-\pi_i^*|\leq\delta,\qquad
\|\mu_i-\mu_i^*\|\leq\delta,
\]
then $\calL\leq \exp(-d/200)$.
\end{lemma}

\begin{proof}
Let $\phi(\mu;x)$ denote the density of $\calN(\mu,I_d)$, and set
\[
\widetilde p(x):=\sum_{i=1}^m\pi_i\phi(\mu_i;x),\qquad
w_i(x):=\frac{\pi_i^*\phi(\mu_i^*;x)}{p^*(x)}.
\]
The dimension assumption ensures $\delta\leq\pimin^*/2$, so every matching
weight is positive. Since $p\geq\widetilde p$ and $\sum_iw_i(x)=1$, Jensen's
inequality for $-\log$ gives, pointwise,
\begin{align*}
\log\frac{p^*(x)}{p(x)}
&\leq-\log\left(\sum_{i=1}^m w_i(x)
\frac{\pi_i\phi(\mu_i;x)}{\pi_i^*\phi(\mu_i^*;x)}\right)\\
&\leq\sum_{i=1}^m w_i(x)
\log\frac{\pi_i^*\phi(\mu_i^*;x)}{\pi_i\phi(\mu_i;x)}.
\end{align*}
Integrating against $p^*$ and using the equal-covariance Gaussian density ratio
yields
\begin{align*}
\calL
&\leq\sum_{i=1}^m\pi_i^*\log\frac{\pi_i^*}{\pi_i}
+\frac12\sum_{i=1}^m\pi_i^*\|\mu_i-\mu_i^*\|^2.
\end{align*}
Indeed, the expectation of
$\log(\phi(\mu_i^*;x)/\phi(\mu_i;x))$ under
$x\sim\calN(\mu_i^*,I_d)$ is $\|\mu_i-\mu_i^*\|^2/2$.
This argument allows the matching fitted weights to sum to less than one.

For $0\leq u\leq1/2$, $-\log(1-u)\leq2u$. Consequently,
\[
\pi_i^*\log\frac{\pi_i^*}{\pi_i}
\leq\pi_i^*\left[-\log\left(1-\frac{\delta}{\pi_i^*}\right)\right]
\leq2\delta.
\]
It follows that
\[
\calL\leq2m\delta+\frac{\delta^2}{2}
\leq2m \exp(-d/100)+\frac12\exp(-d/50)
\leq \exp(-d/200),
\]
where the last comparison follows from \cref{eq:positive:parameters}.
\end{proof}

\begin{proof}[Proof of \cref{thm:positive}]
Condition on the event of \cref{lm:positive:stage1}, of probability at least
$1-1/m$, and let
\[
T_1:=\left\lceil1000n^{10}d^{-0.49}/\eta\right\rceil.
\]
The dimension and learning-rate conditions imply
$T_1\leq d/(\eta\pimin^*)$.
Apply \cref{lm:positive:stage1} for the first $T_1$ global iterations, and then
apply \cref{lm:stage5} with its local clock reset to zero. At every local integer
time $s\geq\lceil8d/(\eta\pimin^*)\rceil+1$ within its horizon, the last two
conclusions give
\[
\|\mu_i-\mu_i^*\|\leq \exp(-d/50),\qquad
|\pi_i-\pi_i^*|\leq4n \exp(-d/40)
\qquad(1\leq i\leq m).
\]
Both bounds are at most $\delta:=\exp(-d/100)$. Also,
\[
T_1+\left\lceil\frac{8d}{\eta\pimin^*}\right\rceil+1
\leq\frac{10d}{\eta\pimin^*}.
\]
For any global integer time $10d/(\eta\pimin^*)\leq t\leq \exp(d/50)/\eta$,
the local time $s=t-T_1$ therefore satisfies the required lower bound and is
at most $\exp(d/50)/\eta$. Thus \cref{lm:last:loss} applies throughout the claimed
global interval and gives $\calL\leq \exp(-d/200)$.
\end{proof}
% Include in the paper's appendix with \input{simulation_appendix}.
% Requires only amsmath, amssymb, and graphicx, as used in the main paper.
% Place the three em_appendix_*.pdf figures in the Overleaf project root.
% This file has no dependencies on other appendix or generated TeX files.
\section{Additional experiments}
\label{app:simulation-additional}

We supplement Section~\ref{sec:experiments} with three-seed evaluations
and comparisons with dimension-dependent separation. All settings use
two ground-truth components with weights $(0.3,0.7)$, ten fitted components,
and learning rate $\eta=0.05$. Initialization follows
Algorithm~\ref{alg:em_update}: independently draw
$Z_i\sim\operatorname{Categorical}(0.3,0.7)$ and
$\mu_i^{(0)}=\mu_{Z_i}^*+\varepsilon_i$, where
$\varepsilon_i\sim\mathcal N(0,I_d)$, and set $\pi_i^{(0)}=1/10$.
We report all runs for seeds $0,1,2$.

\paragraph{Three seeds for the fixed-separation experiment.}
We repeat the setting in Section~\ref{sec:experiments}: $d=10,500$,
ground-truth means $(\pm3,0,\ldots)$, batch size $B=8192$, and 2000 updates.
Figure~\ref{fig:simulation-app-original} shows all six trajectories.
At $d=10$, all three runs approach both true components. At $d=500$,
all three runs exhibit underfitting: one fitted component receives
essentially all the final weight and its mean approaches the mixture mean between the true centers. As with the main-text
experiment, this fixed-separation setting is an illustration rather than
an instance of the negative theorem.

\paragraph{Mean and weight trajectories at $d=10^6$.}
We compare separations $\Delta=d^{0.45}\approx501.19$ and
$\Delta=10\sqrt d=10000$ at $d=10^6$, placing the true means at
$(\Delta/2,0,\ldots)$ and $(-\Delta/2,0,\ldots)$.
We use $B=16$ fresh samples per update for 500 updates, keeping the
component counts, true weights, Gaussian initialization, and learning rate
specified above.
Figures~\ref{fig:simulation-app-means}
and~\ref{fig:simulation-app-weights} show all ten fitted means and
weights at every update, for each seed and separation.
Mean trajectories use the same projection as Figure~\ref{fig:traj}:
$x_i(t)=\mu_{i,1}^{(t)}$ and
$y_i(t)=\|\mu_{i,2:d}^{(t)}\|_2$.
At the smaller separation, one active mean moves between the true centers
and receives essentially all final weight in every seed. At the larger
separation, two active means approach the distinct true centers, and their
weights fluctuate during training. The full trajectories include temporary
weight losses. Component colors and indices are fixed across these two
figures and retain initialization order.
Here an active component has fitted weight at least $0.01$.

\begin{figure}[tbp]
\centering
\includegraphics[width=\linewidth]{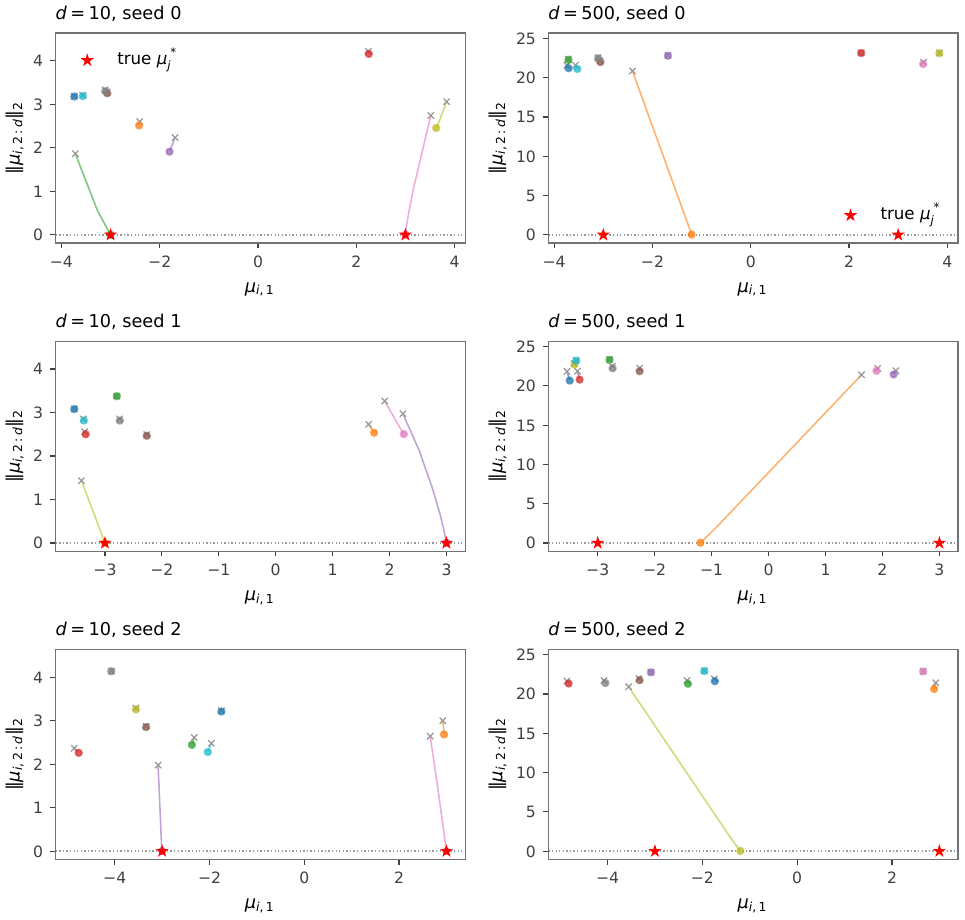}
\caption{Three-seed evaluation of the fixed-separation experiment in
Section~\ref{sec:experiments}. Columns show $d=10,500$; rows show seeds
$0,1,2$. As in Figure~\ref{fig:traj}, gray crosses mark initialization,
colored dots mark final means, and red stars mark true means. At $d=10$,
all three seeds approach both true components. At $d=500$, all three
seeds exhibit underfitting, with a single fitted component carrying
essentially all final weight.}
\label{fig:simulation-app-original}
\end{figure}

\begin{figure}[tbp]
\centering
\includegraphics[width=\linewidth]{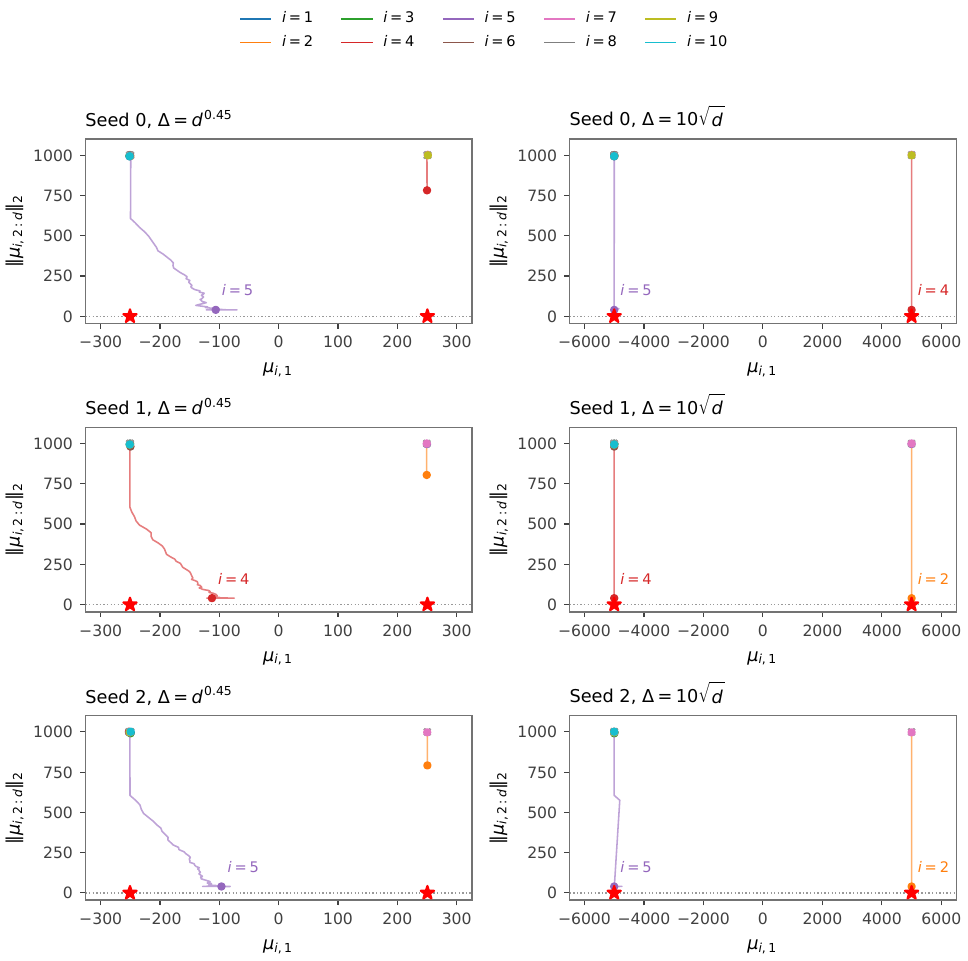}
\caption{Mean trajectories at $d=10^6$, using the same coordinates and
style as Figure~\ref{fig:traj}. Rows show seeds $0,1,2$; columns show
the two separations. 
At the smaller separation one active mean lies between the true centers; at the larger separation two active means approach the distinct true centers.}
\label{fig:simulation-app-means}
\end{figure}

\begin{figure}[tbp]
\centering
\includegraphics[width=\linewidth]{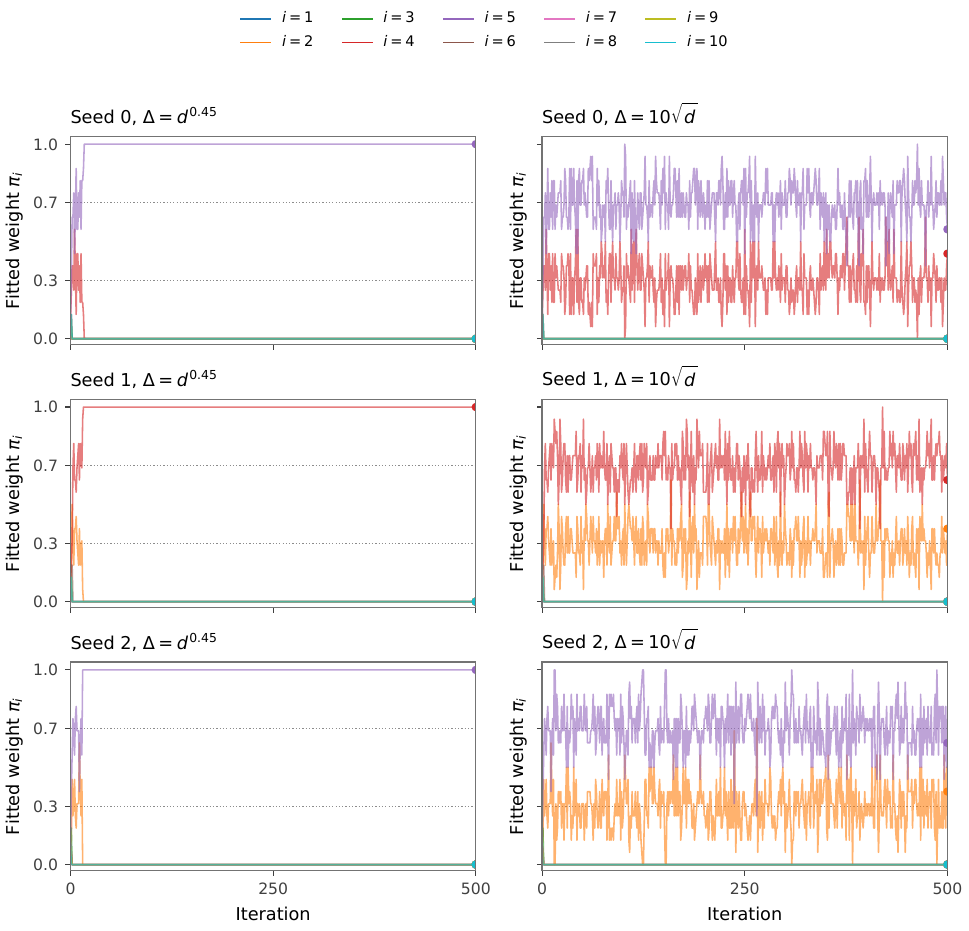}
\caption{Fitted-weight trajectories $\pi_i^{(t)}$ at $d=10^6$ for all ten
components and all 500 updates. At the smaller separation one component receives almost all final weight. At the larger separation two components have substantial final weights.}
\label{fig:simulation-app-weights}
\end{figure}

%%%%%%%%%%%%%%%%%%%%%%%%%%%%%%%%%%%%%%%%%%%%%%%%%%%%%%%%%%%%

\newpage

\end{document}